\documentclass[10pt,journal,compsoc]{IEEEtran}
\usepackage{booktabs}
\usepackage{multirow}
\usepackage{graphicx}
\usepackage{subcaption}
\usepackage{hyperref}
\usepackage{xcolor}

\ifCLASSOPTIONcompsoc
\usepackage[nocompress]{cite}
\else
\usepackage{cite}
\fi
\ifCLASSINFOpdf
\else
\fi
\usepackage{amsmath}
\usepackage{amssymb}
\usepackage{algorithm}
\usepackage{algorithmic}
\usepackage{amsthm}
\newtheorem{proposition}{Proposition}
\usepackage{amsmath}
\newtheorem{assumption}{Assumption}
\newtheorem{lemma}{Lemma}

\newtheorem{theorem}{Theorem}
\newtheorem{definition}{Definition}

\newif\ifupdate\updatefalse

\begin{document}
		%
		\title{Learning to Cut via Hierarchical Sequence/Set Model for Efficient Mixed-Integer Programming}
		%
		%
		%
		%
		
		\author{Jie Wang,~\IEEEmembership{Senior Member,~IEEE,} Zhihai Wang, Xijun Li, Yufei Kuang, \\ Zhihao Shi, Fangzhou Zhu, Mingxuan Yuan, Jia Zeng,~\IEEEmembership{Senior Member,~IEEE,} \\ Yongdong Zhang,~\IEEEmembership{Senior Member,~IEEE,} and Feng Wu,~\IEEEmembership{Fellow,~IEEE}
			\IEEEcompsocitemizethanks{
				\IEEEcompsocthanksitem J. Wang, Z. Wang, X. Li, Y. Kuang, Z. Shi, Y. Zhang, F. Wu are with: a) CAS Key Laboratory of Technology in GIPAS, University of Science and Technology of China, Hefei 230027, China; b) Institute of Artificial Intelligence, Hefei Comprehensive National Science Center, Hefei 230091, China. E-mail: jiewangx@ustc.edu.cn, \{zhwangx,yfkuang,zhihaoshi\}@mail.ustc.edu.cn, \{lihq,zhyd73,fengwu\}@ustc.edu.cn.
				\IEEEcompsocthanksitem X. Li, F. Zhu, M. Yuan, J. Zeng are with Huawei Noah's Ark Lab. E-mail: \{xijun.li,zhufangzhou,Yuan.Mingxuan,Zeng.Jia\}@huawei.com.
			}
			\thanks{Manuscript received April 19, 2005; revised August 26, 2015.}
		}
		
		%
		%

	\markboth{Journal of \LaTeX\ Class Files,~Vol.~14, No.~8, August~2015}%
	{Shell \MakeLowercase{\textit{et al.}}: Learning Cut Selection for Mixed Integer Programming via Hierarchical Sequence Model}
	%



	\IEEEtitleabstractindextext{%
		\begin{abstract}
			Cutting planes (cuts) play an important role in solving mixed-integer linear programs (MILPs), which formulate many important real-world applications. Cut selection heavily depends on \textbf{(P1)} which cuts to prefer and \textbf{(P2)} how many cuts to select. Although modern MILP solvers tackle \textbf{(P1)-(P2)} by human-designed heuristics, machine learning carries the potential to learn more effective heuristics. However, many existing learning-based methods learn which cuts to prefer, neglecting the importance of learning how many cuts to select. Moreover, we observe that \textbf{(P3)} what order of selected cuts to prefer significantly impacts the efficiency of MILP solvers as well. To address these challenges, we propose a novel \textbf{h}ierarchical s\textbf{e}quence/s\textbf{e}t \textbf{m}odel (HEM) to learn cut selection policies. Specifically, HEM is a bi-level model: (1) a higher-level module that learns how many cuts to select, (2) and a lower-level module---that formulates the cut selection as a sequence/set to sequence learning problem---to learn policies selecting an \textit{ordered subset} with the cardinality determined by the higher-level module. To the best of our knowledge, HEM is \textit{the first} data-driven methodology that well tackles \textbf{(P1)-(P3)} simultaneously. Experiments demonstrate that HEM significantly improves the efficiency of solving MILPs on eleven challenging MILP benchmarks, including two Huawei's real problems. 
		\end{abstract}
		
		}

	\maketitle

	\IEEEdisplaynontitleabstractindextext

	%
	\IEEEpeerreviewmaketitle

	\IEEEraisesectionheading{\section{Introduction}\label{sec:introduction}}
	
	\IEEEPARstart{M}ixed-integer linear programming (MILP)---which aims to optimize a linear objective subject to linear constraints and integer constraints (i.e., some or all of the variables are integer-valued)---is one of the most fundamental mathematical models \cite{mip_np_hard, scip_thesis}.  
	MILP is applicable to a wide range of important real-world applications, such as supply chain management \cite{supply_chain_management}, production planning \cite{production_planning}, scheduling \cite{production_scheduling}, 
	vehicle routing \cite{vehicle_routing}, 
	facility location \cite{facility_location}, 
	bin packing \cite{milp_google},
	etc. A standard MILP takes the form of
	\begin{align}
		z^* \triangleq \min_{\textbf{x}} \{ \textbf{c}^{\top} \textbf{x} | \textbf{A}\textbf{x} \leq \textbf{b},\textbf{x}\in \mathbb{R}^n, x_j\in \mathbb{Z} \,\, \text{for all } j \in \mathcal{I}  \}. \label{milp1}
	\end{align}
	Here $\textbf{c} \in \mathbb{R}^n$, $\textbf{A}\in \mathbb{R}^{m\times n}$, $\textbf{b} \in \mathbb{R}^{m}$, $x_j$ denotes the $j$-th entry of vector $\textbf{x}$, $\mathcal{I} \subseteq \{1,\dots,n\}$ denotes 
	the set of indices of integer variables, and $z^*$ denotes the optimal objective value of Problem \eqref{milp1}. However, it can be extremely hard to solve MILPs, as MILPs are $\mathcal{NP}$-hard problems \cite{mip_np_hard}.
	
	Many modern MILP solvers \cite{gurobi, scip8, xpress} solve MILPs by a branch-and-bound tree search algorithm \cite{branch_and_bound}. It builds a tree where it solves a linear programming (LP) relaxation of a MILP (Problem \eqref{milp1} or its subproblems) at each node. To improve the efficiency of the tree search algorithm, cutting planes (cuts) \cite{gomory_cuts} are introduced in the attempt to strengthen the LP relaxations \cite{scip_thesis, bengio_ml4co}.
	Existing work on cuts falls into two categories: cut generation and cut selection \cite{adaptive_cut_selection}. Cut generation aims to generate candidate cuts, i.e., valid linear inequalities that strengthen the LP relaxations \cite{scip_thesis}. Researchers have theoretically studied many families of both general-purpose and class specific cuts \cite{l2c_lookahead}. Nevertheless, it poses a computational challenge if adding all generated cuts to the LP relaxations \cite{implementing_cutting, theoretical_cuts}. To address this challenge, cut selection is a common idea to further improve the efficiency of solving MILPs by selecting a proper subset of the generated cuts \cite{implementing_cutting, theoretical_cuts, adaptive_cut_selection}. 
	In this paper, we focus on \textbf{the cut selection problem}, which plays an important role in terms of the efficiency of MILP solvers \cite{scip_thesis, tang_icml20, l2c_lookahead}.  
	
	Cut selection heavily depends on \textbf{(P1)} which cuts to prefer and \textbf{(P2)} how many cuts to select \cite{scip_thesis, theoretical_cuts}. To tackle \textbf{(P1)-(P2)}, many modern MILP solvers \cite{gurobi, scip8,xpress} employ manually-designed heuristics. However, it is difficult for hard-coded heuristics to take into account underlying patterns among MILPs collected from certain types of real-world applications, e.g., day-to-day production planning, bin packing, and vehicle routing problems \cite{pochet2006production, vehicle_routing, milp_google}.
	To enhance modern MILP solvers, recent methods \cite{tang_icml20, l2c_lookahead, adaptive_cut_selection, cut_ranking} propose to learn cut selection policies via machine learning, especially reinforcement learning. Specifically, many existing learning-based methods \cite{tang_icml20, l2c_lookahead, cut_ranking} learn a scoring function to measure cut quality and select a fixed ratio/number of cuts with high scores. They carry the potential to learn more effective heuristics, as they can effectively capture underlying patterns among MILPs collected from specific applications \cite{bengio_ml4co, tang_icml20, l2c_lookahead, adaptive_cut_selection}. 
	
	However, 
	the aforementioned learning-based methods
	suffer from two limitations. First, they learn which cuts to prefer by learning a scoring function, without learning how many cuts to select \cite{theoretical_cuts, implementing_cutting}.
	Moreover, we observe from extensive empirical results that \textbf{(P3)} what \textit{order of selected cuts} to prefer has a significant impact on the efficiency of solving MILPs as well (see Section \ref{sec:order}). Second, they do not take into account the interaction among cuts when learning which cuts to prefer, as they score each cut \textit{independently}. Consequently, they struggle to select cuts complementary to each other, which could severely degrade the efficiency of solving MILPs \cite{theoretical_cuts}. Indeed, we empirically show that they tend to select many similar cuts with high scores, while many selected cuts are possibly redundant and increase the computational burden (see Section \ref{sec:visu_diversity}). 
	
	To tackle the aforementioned problems, we propose a novel \textbf{h}ierarchical s\textbf{e}quence/s\textbf{e}t \textbf{m}odel (HEM) to learn cut selection policies via reinforcement learning (RL). To the best of our knowledge, HEM is \textit{the first} data-driven methodology
	that can well tackle \textbf{(P1)-(P3)} in cut selection simultaneously. Specifically, HEM is a bi-level model: (1) a higher-level module to learn the number of cuts that
	should be selected, (2) and a lower-level module to learn policies selecting an \textit{ordered subset} with the cardinality determined by the higher-level module. The lower-level module formulates the cut selection task as a novel sequence/set to sequence (Seq2Seq/Set2Seq) learning problem, which leads to two major advantages. 
	First, the Seq2Seq/Set2Seq model is popular in capturing the underlying order information \cite{order_matters}, which is critical for tackling \textbf{(P3)}.
	Second, the Seq2Seq/Set2Seq model can well capture the \textit{interaction} among cuts, as it models the \textit{joint} conditional probability of the selected cuts given an input sequence/set of the candidate cuts.
	As a result, experiments demonstrate that HEM significantly and consistently outperforms competitive baselines in terms of solving efficiency on three synthetic and eight challenging MILP benchmarks. The MILP problems include some benchmarks from MIPLIB 2017 \cite{miplibs_2017}, and large-scale real-world problems at Google and Huawei. 
	Therefore, the powerful performance of HEM demonstrates its promising potential for enhancing modern MILP solvers in real-world applications.
	Moreover, experiments demonstrate that HEM well generalizes to MILPs that are significantly larger than those in the training data. 
	
	We summarize our major contributions as follows. (1) We observe an important problem in cut selection, i.e., what \textit{order of selected cuts} to prefer 
	has a significant impact on the efficiency of solving MILPs (see Section \ref{sec:order}).
	(2) To the best of our knowledge, our proposed HEM is \textit{the first} data-driven methodology that is able to tackle \textbf{(P1)-(P3)} in cut selection simultaneously. (3) We formulate the cut selection task as a novel sequence/set to sequence learning problem, which captures not only the underlying order information, but also the interaction among cuts to select cuts complementary to each other. (4) Experiments demonstrate that HEM achieves significant and consistent improvements over competitive baselines on challenging MILP problems, including some benchmarks from MIPLIB 2017 and large-scale real-world problems at Google and Huawei.
	
	An earlier version of this paper has been published at
	ICLR 2023 \cite{hem_iclr}. This journal manuscript significantly
	extends the conference version by proposing an enhanced version of HEM, namely HEM++, which further improves HEM in terms of the formulation, policy, and training method, respectively. (1) Regarding the formulation, HEM++ extends the formulation to the multiple rounds setting---which is widely applicable in real-world scenarios---while HEM mainly focuses on training policies under the one round setting (see Section \ref{sec:formulation_multiple_rounds}). (2) Regarding the policy, HEM++ formulates the model as a \textit{set} to sequence model instead of a sequence to sequence (i.e., the model in HEM), which improves the representation of input cuts by removing redundant input-order information (see Section \ref{method:extension_set2seq}).
	(3) Regarding the training method, HEM++ proposes a hierarchical proximal policy optimization method, which further improves sample efficiency (see Section \ref{method:hppo}). Sample-efficient training is important for learning cut selection, especially under the multiple rounds setting, as the action space is large and increases exponentially with the number of cut separation rounds.
	Experiments demonstrate the superiority of HEM++ in Section \ref{exp:hem++}. In addition to HEM++, this paper 
	further deploys HEM to Huawei’s real production scenario to evaluate its effectiveness (see Sections \ref{method:order_heuristics} and \ref{exp:deployment}), and conducts visualization experiments to provide further insight into our method (see Section \ref{exp:visu_and_explain}).

	\section{Related work}
	
	The use of machine learning to improve the MILP solver performance has been an active topic of significant interest in recent years \cite{bengio_ml4co, lodi2017learning, nips21_ml4co_competition, nips19_gcnn, tang_icml20}. During the solving process of the solvers, many crucial decisions that significantly impact the solver performance are based on heuristics \cite{scip_thesis}. Recent methods propose to learn more effective heuristics via machine learning from MILPs collected from specific applications \cite{bengio_ml4co, ecole}. This line of research has shown significant improvement on the solver performance, including cut selection \cite{tang_icml20,l2c_lookahead, adaptive_cut_selection,baltean2019scoring}, variable selection \cite{Khalil_learn_to_branch,nips19_gcnn,pmlr-v80-balcan18a, Parameterizing_branch}, node selection \cite{learning_to_search, node_uct}, column generation \cite{morabit2021machine}, and primal heuristics selection \cite{khalil2017learning,hendel2019adaptive}. 
	
	In this paper, we focus on cut selection, which plays an important role in modern MILP solvers \cite{theoretical_cuts, implementing_cutting}. For cut selection, many existing learning-based methods \cite{tang_icml20, l2c_lookahead, cut_ranking} focus on learning which cuts should be preferred by learning a scoring function to measure cut quality. Specifically, \cite{tang_icml20} proposes a reinforcement learning approach to learn to select the best Gomory cut \cite{gomory_cuts}. Furthermore, \cite{l2c_lookahead} proposes to learn to select 
	a cut that yield the best dual bound improvement via imitation learning. Instead of selecting the best cut, \cite{cut_ranking} frames cut selection as multiple instance learning to learn a scoring function, and selects a fixed ratio of cuts with high scores. However, they suffer from two limitations as mentioned in Section \ref{sec:introduction}.
	In addition, \cite{adaptive_cut_selection} proposes to learn weightings of four expert-designed scoring rules. On the theoretical side, \cite{balcan2021sample, balcan2022structural} has provided some provable guarantees for learning cut selection policies. 
	

	\section{Background}
	
	\subsection{Cutting planes} 
	Given the MILP problem in (\ref{milp1}), we drop all its integer constraints to obtain its \textit{linear programming (LP) relaxation}, which takes the form of  
	\begin{align}\label{lp_relaxation}
		z_{\text{LP}}^* \triangleq \min_{\textbf{x}} \{\textbf{c}^{\top} \textbf{x} | \textbf{A}\textbf{x} \leq \textbf{b},\textbf{x}\in \mathbb{R}^n\}.
	\end{align}
	As the problem in (\ref{lp_relaxation}) expands the feasible set of the problem in (\ref{milp1}), we have $z_{\text{LP}}^* \leq z^*$. We denote any lower bound found via an LP relaxation by a \textit{dual bound}.
	Given the LP relaxation in (\ref{lp_relaxation}), cutting planes (cuts) are linear inequalities that are added to the LP relaxation 
	to tighten it without removing any integer feasible solutions of (\ref{milp1}). Cuts generated by MILP solvers are added in successive rounds. Specifically, each round $k$ involves (i) solving the current LP relaxation, (ii) generating a pool of candidate cuts $\mathcal{C}^k$, (iii) selecting a subset $\mathcal{S}^k\subseteq \mathcal{C}^k$, (iv) adding $\mathcal{S}^k$ to the current LP relaxation to obtain the next LP relaxation, (v) and proceeding to the next round.
	Adding all the generated cuts to the LP relaxation would maximally strengthen the LP relaxation and improve the lower bound at each round. However, adding too many cuts could lead to large models, which can increase the computational burden and present numerical instabilities \cite{implementing_cutting, l2c_lookahead}. Therefore, 
	cut selection is proposed to select a proper subset of the candidate cuts, which is significant for improving the efficiency of solving MILPs \cite{theoretical_cuts, tang_icml20}.
	
	\subsection{Branch-and-cut} 
	In modern MILP solvers, cutting planes are often combined with the branch-and-bound algorithm \cite{branch_and_bound}, which is known as the branch-and-cut algorithm 
	\cite{branch_and_cut}.
	Branch-and-bound techniques perform implicit enumeration by building a search tree, in which every node represents a subproblem of the original problem in (\ref{milp1}).
	The solving process begins by selecting a leaf node of the tree and solving its LP relaxation. Let $\textbf{x}^*$ be the optimal solution of the LP relaxation. If $\textbf{x}^*$ violates the original integrality constraints, two subproblems (child nodes) of the leaf node are created by \textit{branching}. Specifically, the leaf node is added with constraints $ x_i \leq \lfloor x_i^* \rfloor\,\,\text{and}\,\, x_i \geq \lceil x_i^* \rceil $, respectively, where $x_i$ denotes the $i$-th variable, $x_i^*$ denotes the $i$-th entry of vector $\textbf{x}^*$, and $\lfloor \rfloor$ and $\lceil \rceil$ denote the floor and ceil functions.
	In contrast, if $\textbf{x}^*$ is a (mixed-)integer solution of (\ref{milp1}), then we obtain an upper bound on the optimal objective value of (\ref{milp1}), which we denote by \textit{primal bound}.
	In modern MILP solvers, the addition of cutting planes is alternated with the \textit{branching} phase. That is, cuts are added at search tree nodes before branching to tighten their LP relaxations. 
	As strengthening the relaxation before starting to branch is decisive to ensure an efficient tree search \cite{implementing_cutting, bengio_ml4co}, we focus on adding cuts at the root node, which follows \cite{nips19_gcnn, l2c_lookahead}.
	
	\subsection{Primal-dual gap integral}
	We keep track of two important bounds when running branch-and-cut, i.e., the global primal and dual bounds, which are the best upper and lower bounds on the optimal objective value of Problem \eqref{milp1}, respectively.
	We define the \textit{primal-dual gap integral} (PD integral) by the \textit{area} between the curve of the solver's global primal bound and the curve of the solver's global dual bound. 
	The PD integral is a 
	widely used metric for evaluating solver performance \cite{nips21_ml4co_competition, cao2022ml4co}. For example, the primal-dual gal integral is used as a primary evaluation metric in the NeurIPS 2021 ML4CO competition \cite{nips21_ml4co_competition}. Moreover, We define the \textit{primal-dual gap} by the difference between the global primal and dual bounds.
	
	
	\begin{figure*}[t]
		\centering
		\begin{subfigure}{0.24\textwidth}
			\includegraphics[width=\textwidth,height=0.6\textwidth]{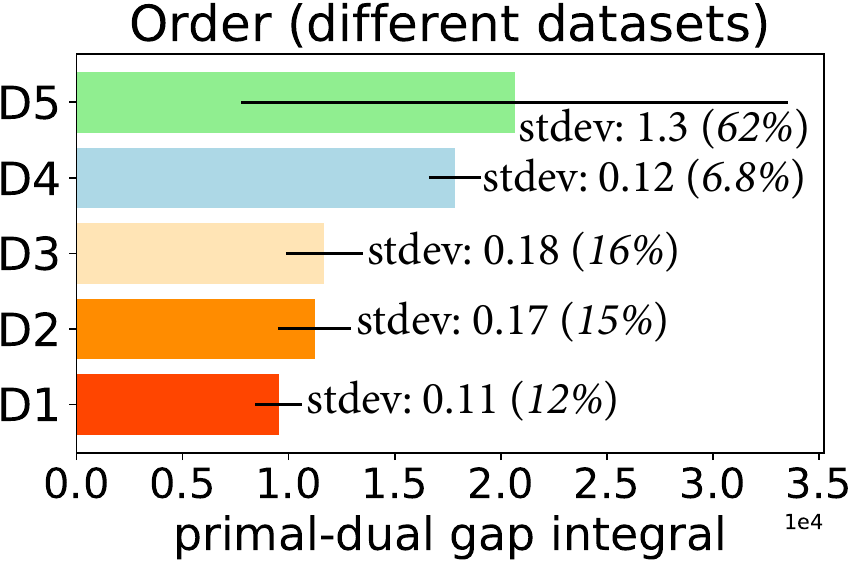}
			\caption{Evaluate RandomAll on five different datasets.}
			\label{fig:order_randomall}
		\end{subfigure}
		\begin{subfigure}{0.24\textwidth}
			\includegraphics[width=\textwidth,height=0.6\textwidth]{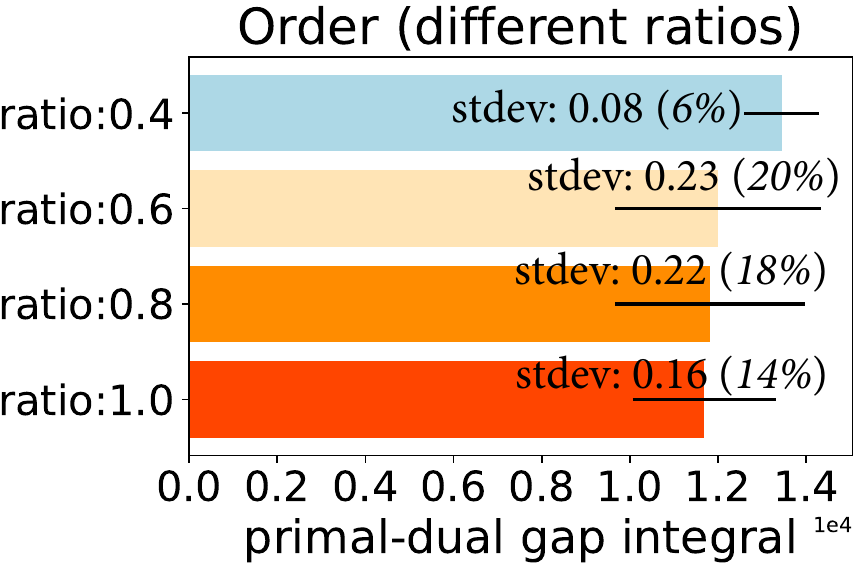}
			\caption{Evaluate RandomNv on MIPLIB mixed neos.}
			\label{fig:order_randomnv}
		\end{subfigure}
		\begin{subfigure}{0.24\textwidth}
			\includegraphics[width=\textwidth,height=0.6\textwidth]{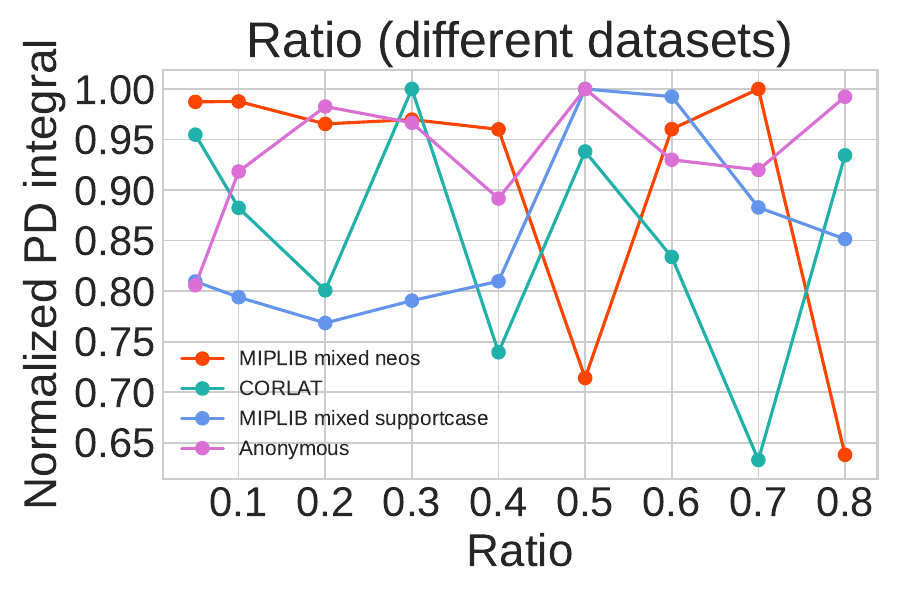}
			\caption{Evaluate NV with different ratios on four datasets.}
			\label{fig:size_diff_datasets}
		\end{subfigure}
		\begin{subfigure}{0.24\textwidth}
			\includegraphics[width=\textwidth,height=0.6\textwidth]{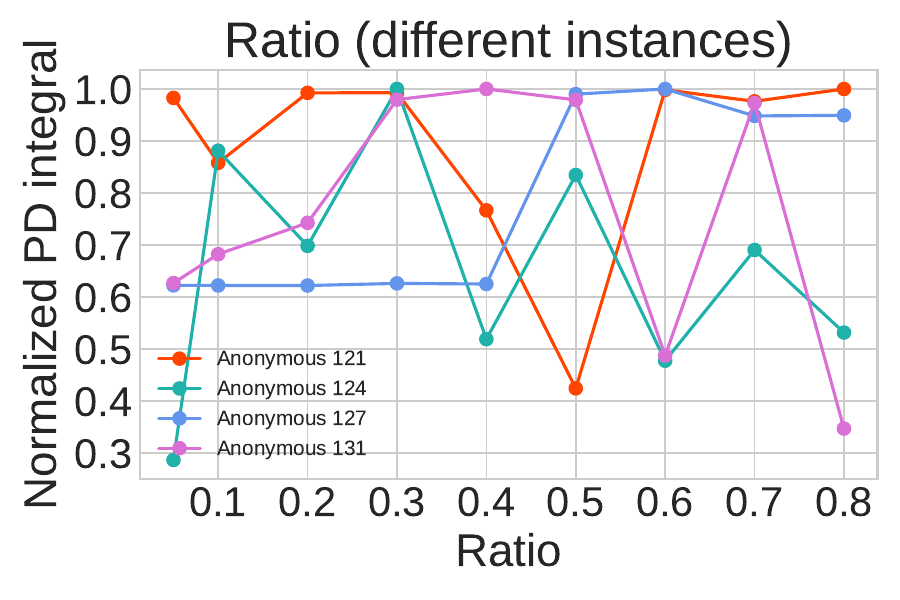}
			\caption{Evaluate NV with different ratios on Anonymous.}
			\label{fig:size_diff_instances}
		\end{subfigure}
		\caption{(a)-(b) We design two cut selection heuristics, namely RandomAll and RandomNV (see Section \ref{sec:order} for details), which both add the same subset of cuts in random order for a given MILP. The results in (a) and (b) show that adding the same selected cuts in different order leads to variable overall solver performance.
			(c)-(d) We use the Normalized Violation (NV) heuristics \cite{cut_ranking} in the following experiments. The results in (c) show that the performance of NV varies widely with the given ratios across different datasets. The results in (d) show that the performance of NV varies widely with the given ratios across different instances from the Anonymous dataset.
		}
		\label{fig:motivating_results_order_size}
	\end{figure*}
	
	\section{Motivating results}
	
	In this section, we first empirically show that
	the \textit{order of selected cuts}, i.e., the selected cuts are added to the LP relaxations in this \textit{order}, significantly impacts the efficiency of solving MILPs. Moreover, we then empirically show that the \textit{ratio of selected cuts} significantly impacts the efficiency of solving MILPs as well. Please see Appendix \ref{datasets_motivating} for details of the datasets used in this section.
	
	\subsection{Order matters}\label{sec:order}
	
	Previous work \cite{implementing_simplex, lp_textbook, lodi2013performance, Reformulate} has shown that the order of constraints for a given linear program (LP) significantly impacts its constructed initial basis, which is important for solving the LP. As a cut is a linear constraint, adding cuts to the LP relaxations is equivalent to adding constraints to the LP relaxations. Therefore, the order of added cuts could have a significant impact on solving the LP relaxations as well, thus being important for solving MILPs. Indeed, our empirical results show that this is the case.
	\textbf{(1)} We design a \textbf{RandomAll} cut selection rule,
	which randomly permutes all the candidate cuts, and adds all the cuts to the LP relaxations in the random order. 
	We evaluate RandomAll on five challenging datasets, namely D1, D2, D3, D4, and D5. 
	We use the SCIP 8.0.0 \cite{scip8} as the backend solver, and evaluate the solver performance by the average PD integral within a time limit of 300 seconds. We evaluate RandomAll on each dataset over ten random seeds, and each bar in Figure \ref{fig:order_randomall} shows the mean and standard deviation (stdev) of its performance on each dataset.
	As shown in Figure \ref{fig:order_randomall},
	the performance of RandomAll on each dataset varies widely with the order of selected cuts.
	\textbf{(2)} We further design a \textbf{RandomNV} cut selection rule. RandomNV is different from RandomAll in that it selects a given ratio of the candidate cuts rather than all the cuts. RandomNV first scores each cut using the Normalized Violation \cite{cut_ranking} and selects a given ratio of cuts with high scores. It then randomly permutes the selected cuts. Each bar in Figure \ref{fig:order_randomnv} shows the mean and stdev of the performance of RandomNV with a given ratio on the same dataset.
	Figures \ref{fig:order_randomall} and \ref{fig:order_randomnv} show that adding the same selected cuts in different order leads to variable solver performance, which demonstrates that the order of selected cuts is important for solving MILPs.
	
	\subsection{Ratio matters}\label{sec:ratio}
	We use the Normalized Violation (NV) \cite{cut_ranking} method in the following experiments. \textbf{(1)} We first evaluate the NV methods that select different ratios of candidate cuts on four datasets. Each line in Fig. \ref{fig:size_diff_datasets} shows the performance (normalized average PD integral) of NV with different given ratios of selected cuts on each dataset. The results in Fig. \ref{fig:size_diff_datasets} show that the performance of NV varies widely with the ratio of selected cuts. Moreover, the results demonstrate that the ratio that leads to better solver performance is variable across different datasets, suggesting that learning dataset-dependent ratios is important. 
	\textbf{(2)} We then evaluate the NV methods that select different ratios of candidate cuts on four instances from the Anonymous dataset. Each line in Fig. \ref{fig:size_diff_instances} shows the performance (normalized average PD integral) of NV with different given ratios of selected cuts on each instance. The results in Fig. \ref{fig:size_diff_instances} show that the performance of NV varies widely with the ratio of selected cuts. Moreover, the results demonstrate that the ratio that leads to better solver performance is variable across different instances, which suggests that learning instance-dependent ratios is important for cut selection as well. 
	
	\begin{figure}
		\centering
		\includegraphics[width=0.98\columnwidth]{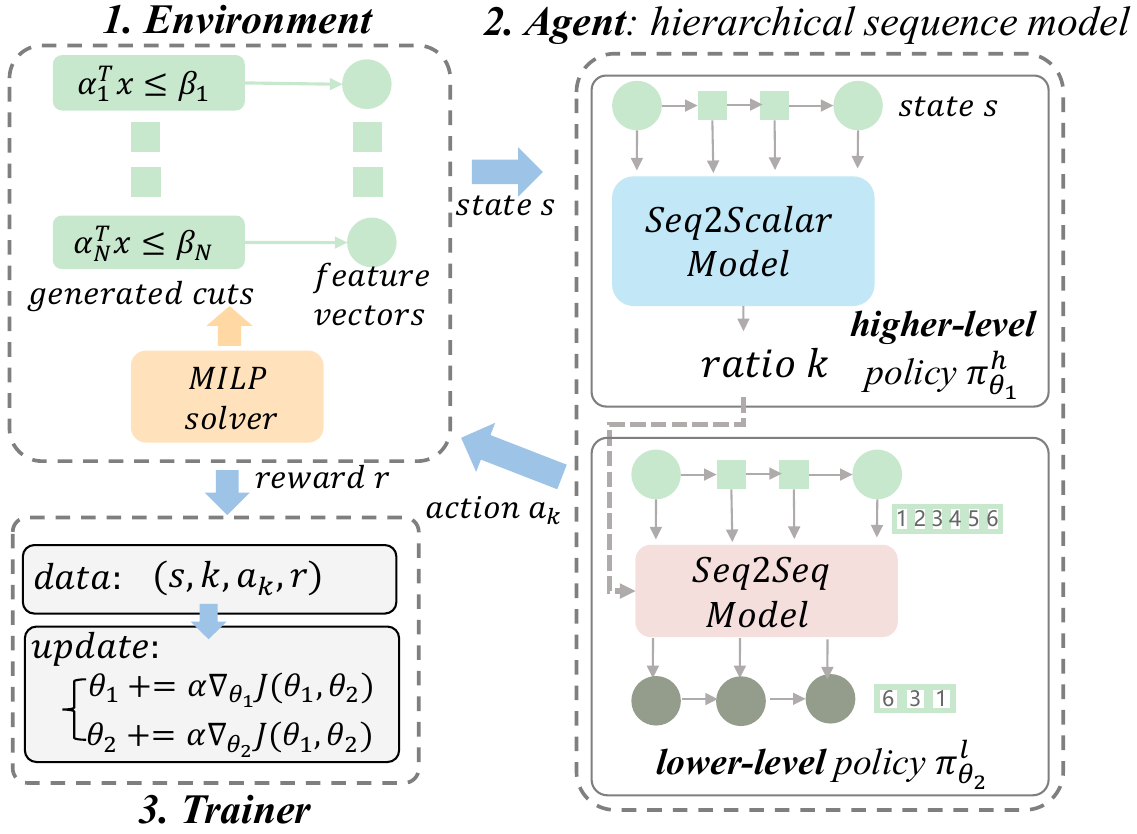}
		\caption{Illustration of our proposed RL framework for learning cut selection policies. We formulate a MILP solver as the environment and the
			HEM as the agent. Moreover, we train HEM via
			hierarchical policy gradients.}
		\label{fig:hem_illustration}
	\end{figure}
	
	\section{Learning Cut Selection via Hierarchical Sequence Model (HEM)}\label{method:main}
	
	For cut selection, the optimal subsets that should be selected are inaccessible, while one can assess the quality of selected subsets using a solver and provide the feedbacks to learning algorithms. Therefore, we leverage reinforcement learning (RL) to learn cut selection policies. 
	In this section, we provide a detailed description of our proposed RL framework for learning cut selection as shown in Fig \ref{fig:hem_illustration}. First, we present our formulation of the cut selection as a Markov decision process (MDP) \cite{rl_sutton} in Section \ref{sec:formulation}. Second, we present a detailed description of our proposed HEM in Section \ref{method:hem}. Finally, we derive a hierarchical policy gradient for training HEM under the one round setting in Section \ref{train:hpg}. 
	
	
	\subsection{Reinforcement Learning Formulation}\label{sec:formulation}
	
	As shown in Fig \ref{fig:hem_illustration}, we formulate a MILP solver as the environment and our proposed HEM as the agent. 
	We consider an MDP defined by the tuple
	$(\mathcal{S}, \mathcal{A}, r, f)$.
	Specifically, we specify the state space $\mathcal{S}$, the action space $\mathcal{A}$, the reward function $r:\mathcal{S} \times \mathcal{A} \to \mathbb{R}$, the transition function $f$, and the terminal state in the following. 
	\textbf{(1) The state space $\mathcal{S}$.}  
	Since the current LP relaxation and the generated cuts contain the core information for cut selection, we define a state $s$ by $(M_{\text{LP}}, \mathcal{C}, \textbf{x}_{\text{LP}}^*)$. Here $M_{\text{LP}}$ denotes the mathematical model of the current LP relaxation, $\mathcal{C}$ denotes the set of the candidate cuts, and $\textbf{x}_{\text{LP}}^*$ denotes the optimal solution of the LP relaxation. To encode the state information, we 
	follow \cite{scip_thesis, cut_ranking} to design thirteen features for each candidate cut based on the information of $(M_{\text{LP}}, \mathcal{C}, \textbf{x}_{\text{LP}}^*)$. That is, we actually represent a state $s$ by \textit{a sequence of thirteen-dimensional feature vectors}. We present details of the designed features in Appendix \ref{appendix_cut_features}. 
	\textbf{(2) The action space $\mathcal{A}$.}  
	To take into account the ratio and order of selected cuts, we define the action space by \textit{all the ordered subsets} of the candidate cuts $\mathcal{C}$. 
	It can be challenging to explore the action space efficiently, as the cardinality of the action space can be extremely large due to its combinatorial structure.
	\textbf{(3) The reward function $r$.} 
	To evaluate the impact of the added cuts on solving MILPs, we design the reward function by (i) measures collected at the end of solving LP relaxations such as the dual bound improvement, (ii) or end-of-run statistics, such as the solving time and the primal-dual gap integral. For the first, the reward $r(s,a)$ can be defined as the negative dual bound improvement at each step. For the second, the reward $r(s,a)$ can be defined as the negative solving time or primal-dual gap integral at each step. 
	\textbf{(4) The transition function $f$.} The transition function maps the current state $s$ and the action $a$ to the next state $s^{\prime}$, where $s^{\prime}$ represents the next LP relaxation generated by adding the selected cuts at the current LP relaxation. 
	\textbf{(5) The terminal state.}
	There is no standard and unified criterion to determine when to terminate the cut separation procedure \cite{l2c_lookahead}. Suppose we set the cut separation rounds as $T$, then the solver environment terminates the cut separation after $T$ rounds. Under the multiple rounds setting (i.e., $T>1$), we formulate the cut selection as a Markov decision process. Under the one round setting (i.e., $T=1$), the formulation can be simplified as a contextual bandit \cite{rl_sutton}.
	
	\subsection{Hierarchical Sequence Model}\label{method:hem}
	\subsubsection{Motivation}
	Let $\pi$ denote the cut selection policy $\pi: \mathcal{S} \to \mathcal{P}(\mathcal{A})$, where $\mathcal{P}(\mathcal{A})$ denotes the probability distribution over the action space, and $\pi(\cdot|s)$ denotes the probability distribution over the action space given the state $s$. We emphasize that learning such policies can tackle \textbf{(P1)-(P3)} in cut selection simultaneously. 
	However, directly learning such policies is challenging for the following reasons. First, it is challenging to explore the action space efficiently, as the cardinality of the action space can be extremely large due to its combinatorial structure. Second, the length and max length of actions (i.e., ordered subsets) are variable across different MILPs. 
	However, traditional RL usually deals with problems whose actions have a fixed length. 
	Instead of directly learning the aforementioned policy, many existing learning-based methods \cite{tang_icml20, cut_ranking, l2c_lookahead} learn a scoring function that outputs a score given a cut, and select a fixed ratio/number of cuts with high scores. However, they suffer from two limitations as mentioned in Section \ref{sec:introduction}.
	
	\subsubsection{Policy network architecture}
	To tackle the aforementioned problems, we propose a novel hierarchical sequence model (HEM) to learn cut selection policies. 
	To promote efficient exploration, HEM leverages the hierarchical structure of the cut selection task to decompose the policy into two sub-policies, i.e., a higher-level policy $\pi^h$ and a lower-level policy $\pi^l$. The policy network architecture of HEM is also illustrated in Figure \ref{fig:hem_illustration}. 
	
	First, we formulate the higher-level policy as a Sequence to Scalar (Seq2Scalar) model, which learns the number of cuts that should be selected by predicting a proper ratio. Suppose the length of the state is $N$ and the predicted ratio is $k$, then the predicted number of cuts that should be selected is $\lfloor N*k \rfloor$, where $\lfloor \cdot \rfloor$ denotes the floor function.
	For exploration and differentiability of the policy, we define the higher-level policy by a stochastic policy, which is commonly used in RL \cite{rl_sutton}. Specifically, we define the higher-level policy by $\pi^h: \mathcal{S} \to \mathcal{P}(\left[0,1\right])$, where  $\pi^h(\cdot|s)$ denotes the probability distribution over $\left[0,1\right]$ given the state $s$.  
	
	Second, we formulate the lower-level policy as a Sequence to Sequence (Seq2Seq) model, which learns to select an ordered subset with the cardinality determined by the higher-level policy. 
	Specifically, we define the lower-level policy by $\pi^l: \mathcal{S} \times \left[0,1\right] \to \mathcal{P}(\mathcal{A})$, where $\pi^l(\cdot|s,k)$ denotes the probability distribution over the action space given the state $s$ and the ratio $k$. To the best of our knowledge, we are the first to formulate the cut selection task as a sequence to sequence learning problem. This leads to two major advantages: (1) capturing the underlying order information, (2) and the interaction among cuts. 
	
	
	
	Finally, we derive the cut selection policy via the law of total probability, i.e.,
	$\pi(a_k|s) = \mathbb{E}_{k\sim \pi^h(\cdot|s)}[\pi^l(a_k|s,k)],$
	where $k$ denotes the given ratio and $a_k$ denotes the action. The policy is computed by an expectation, as $a_k$ cannot determine the ratio $k$. For example, suppose that $N=100$ and the length of $a_k$ is $10$, then the ratio $k$ can be any number in the interval $[0.1,0.11)$. Actually, we sample an action from the policy $\pi$ by first sampling a ratio $k$ from $\pi^h$ and then sampling an action from $\pi^l$ given the ratio.
	
	\begin{figure}
		\centering
		\includegraphics[width=0.99\columnwidth]{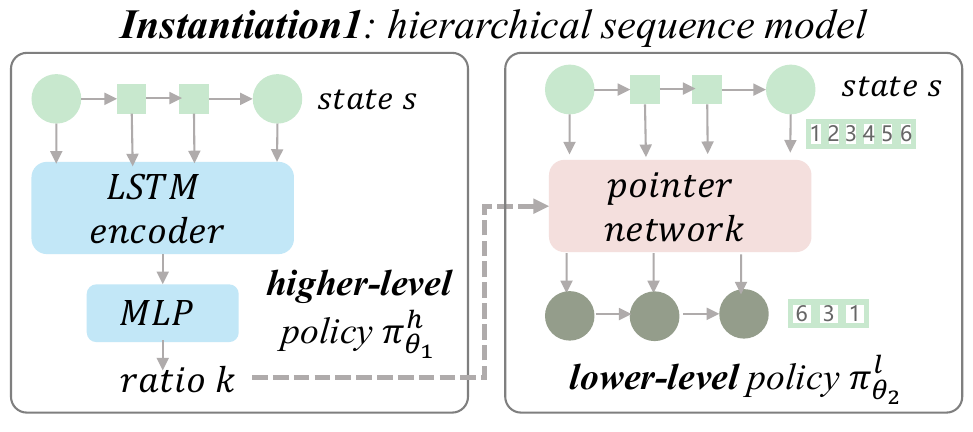}
		\caption{The instantiation of HEM.}
		\label{fig:hem_model_illustration1}
	\end{figure}
	
	\subsubsection{Instantiation of the policy network}
	For the higher-level policy, we first model the higher-level policy as a tanh-Gaussian, i.e., a Gaussian distribution with an invertible squashing function ($\tanh$), which is commonly used in deep reinforcement learning \cite{ppo,sac}. The mean and variance of the Gaussian are given by neural networks. 
	The support of the tanh-Gaussian is $[-1,1]$, but a ratio of selected cuts should belong to $[0,1]$. Thus, we further perform a linear transformation on the tanh-Gaussian. Specifically, we define the parameterized higher-level policy by 
	$\pi^{h}_{\theta_1}(\cdot|s) = 0.5 * \tanh{(K)} + 0.5$, where $K\sim \mathcal{N}(\mu_{\theta_1}(s), \sigma_{\theta_1}(s))$.
	Since the sequence lengths of states are variable across different instances (MILPs), we use a long-short term memory (LSTM) 
	\cite{lstm} network to embed the sequence of candidate cuts. We then use a multi-layer perceptron (MLP) \cite{deeplearning} to predict the mean and variance from the last hidden state of the LSTM. 
	
	For the lower-level policy, we formulate it as a Seq2Seq model. That is, its input is a sequence of candidate cuts, and its output is the probability distribution over ordered subsets of candidate cuts with the cardinality determined by the higher-level policy. Specifically, given a state action pair $(s,k,a_k)$, the Seq2Seq model computes the conditional probability $\pi^l_{\theta_2}(a_k|s,k)$ using a parametric model to estimate the terms of the probability chain rule, i.e.,
	$\pi^{l}_{\theta_2}(a_k|s,k) = \prod_{i=1}^{m} p_{\theta_2}(a_k^i|a_k^1,\dots,a_k^{i-1},s,k)$.
	Here $s=\{s^1,\dots,s^N\}$ is the input sequence, $m=\lfloor N*k \rfloor$ is the length of the output sequence, and $a_k=\{a_k^1,\dots,a_k^m\}$ is a sequence of $m$ indices, each corresponding a position in the input sequence $s$. Such policy can be parametrized by the vanilla Seq2Seq model commonly used in machine translation \cite{seq2seq, transformer}. However, the vanilla Seq2Seq model is only applicable to learning on a single instance, as the number of candidate cuts varies on different instances. To generalize across different instances, we use a pointer network \cite{pn, neural_combinatorial}---which uses attention as a pointer to select a member of the input sequence as the output at each decoder step---to parametrize $\pi^l_{\theta_2}$ (see Appendix \ref{appendix_pn_details} for details). 
	
	
	\subsection{Training: hierarchical policy gradient}\label{train:hpg}
	
	For the cut selection task, we aim to find $\theta$ that maximizes the expected reward over all trajectories
	\begin{align}\label{eq:obj}
		J(\theta) = \mathbb{E}_{s\sim \mu, a_k \sim \pi_{\theta}(\cdot|s)}[r(s,a_k)],
	\end{align}
	where $\theta=
	\left[\theta_1,\theta_2\right]$ with $\left[\theta_1,\theta_2\right]$ denoting the concatenation of the two vectors, $\pi_{\theta}(a_k|s) = \mathbb{E}_{k\sim \pi^h_{\theta_1}(\cdot|s)}[\pi^l_{\theta_2}(a_k|s,k)]$, and $\mu$ denotes the initial state distribution. \textit{Here we focus on training cut selection policies under the one round setting.} We defer discussion on training policies under the multiple rounds setting to Section \ref{Method:extension}. 
	To train the policy with a hierarchical structure, we derive a hierarchical policy gradient following the well-known policy gradient theorem \cite{pg_theorem,rl_sutton}.
	\begin{proposition}\label{proof_hpg}
		Given the cut selection policy $\pi_{\theta}(a_k|s) = \mathbb{E}_{k\sim \pi^h_{\theta_1}(\cdot|s)}[\pi^l_{\theta_2}(a_k|s,k)]$ and the training objective (\ref{eq:obj}), the hierarchical policy gradient takes the form of 
		\begin{align*}
			& \nabla_{\theta_1}J(\left[\theta_1,\theta_2\right]) \\
			& = \mathbb{E}_{s\sim\mu, k\sim \pi^h_{\theta_1}(\cdot|s)} [\nabla_{\theta_1} \log(\pi^h_{\theta_1}(k|s)) \mathbb{E}_{a_k\sim \pi^l_{\theta2}(\cdot|s,k)}[r(s,a_k)] ], \\
			& \nabla_{\theta_2}J(\left[\theta_1,\theta_2\right]) \\
			& = \mathbb{E}_{s\sim\mu, k\sim \pi^h_{\theta_1}(\cdot|s), a_k\sim \pi^l_{\theta_2}(\cdot|s,k)} [\nabla_{\theta_2}\log \pi^l_{\theta_2}(a_k|s,k) r(s,a_k)].
		\end{align*}
	\end{proposition}
	We provide detailed proof in Appendix \ref{appendix_proof_hpg_one_round}. We use the derived hierarchical policy gradient to update the parameters of the higher-level and lower-level policies. We implement the training algorithm in a parallel manner that is closely related to the asynchronous advantage actor-critic (A3C) \cite{a3c}. Furthermore, we summarize the procedure of the training algorithm in Algorithm \ref{alg:hem1}.
	
	
	
	\begin{algorithm}[t]
		\caption{Pseudo code for training HEM}.
		\label{alg:hem1}
		\begin{algorithmic}[1]
			\STATE \textbf{Initialize} Hierarchical sequence model $\pi_{[\theta_1,\theta_2]}$, MILP instances $\mathcal{D}$, training dataset $\mathcal{D}_{\text{train}}$, batch size $N_b$, training epochs $N_e$, policy learning rate $\alpha$
			\FOR{$N_e$ epochs}
			\STATE Empty the training dataset $\mathcal{D}_{\text{train}}$
			\FOR{$N_b$ steps}
			\STATE Randomly sample a MILP $s_0$ from $\mathcal{D}$
			\STATE Take action $k$ and $a_k$ at state $s_0$ with the policy $\pi$
			\STATE Receive reward $r$ and add $(s_0,k,a_k,r)$ to  $\mathcal{D}_{\text{train}}$
			\ENDFOR
			\STATE Compute hierarchical policy gradient using $\mathcal{D}_{\text{train}}$ as in proposition \ref{proof_hpg}
			\STATE Update the parameters, $\theta_1 = \theta_1 + \alpha \nabla_{\theta_1}J(\left[\theta_1,\theta_2\right])$, $\theta_2 = \theta_2 + \alpha \nabla_{\theta_2}J(\left[\theta_1,\theta_2\right])$
			\ENDFOR
		\end{algorithmic}
	\end{algorithm}
	
	\section{HEM++: Further Improvements of HEM}\label{Method:extension}
	To further enhance HEM, we propose its improvements in terms of the formulation, policy, and training method, respectively. We denote the enhanced method of HEM by HEM++. We present the details as follows. 
	
	\subsection{RL formulation}\label{sec:formulation_multiple_rounds}
	
	In this part, we extend the RL formulation mentioned in Section \ref{sec:formulation} to the multiple rounds setting by redefining the action space and the hierarchical cut selection policy. The details are as follows. Under the multiple rounds setting, we formulate the cut selection as a Markov decision process. However, this entails the challenge of computing the hierarchical policy gradient, as the hierarchical cut selection policy is defined by an expectation. To address this challenge, we redefine the action space $\mathcal{A}^{\text{ext}}$ by $[0,1] \times \mathcal{A}$, and the hierarchical cut selection policy by 
	$\pi(k,a_k|s) = \pi^h(k|s)\cdot\pi^l(a_k|s,k)$. We present the derivation of the hierarchical policy gradient under the multiple rounds setting as follows. 
	
	Under the multiple rounds setting, we aim to find $\theta$ that maximizes the expected discounted cumulative rewards
	\begin{align}\label{eq:obj_multi_rounds}
		J(\theta) = \mathbb{E}_{s_0,k_0,a_{0}^k\dots}[\sum_{t=0}^{\infty} \gamma^t r(s_t,k_t,a_{t}^k)],
	\end{align}
	where $s_0\sim \mu$, $(k_t,a_t^k)\sim \pi_{\theta}(\cdot|s_t)$, $s_{t+1}=f(s_t,a_t)$, $\pi_{\theta}(k_t,a_t^k|s_t) = \pi^h_{\theta_1}(k_t|s_t) \cdot \pi^l_{\theta_2}(a_t^k|s_t,k_t)$, and $\gamma\in (0,1]$ is the discount factor. Then we derive a hierarchical policy gradient under the multiple rounds setting as follows. 
	\begin{proposition}\label{proof_hpg_multi_rounds}
		Given the cut selection policy $\pi_{\theta}(k,a^k|s) = \pi^h_{\theta_1}(k|s) \cdot \pi^l_{\theta_2}(a^k|s,k)$ and the training objective (\ref{eq:obj_multi_rounds}), the hierarchical policy gradient takes the form of 
		\begin{align*}
			& \nabla_{\theta_1}J(\left[\theta_1,\theta_2\right]) \\
			& = \mathbb{E}_{s\sim\rho^{\pi_{\theta}}, (k,a^k)\sim \pi_{\theta}(\cdot|s) } [Q^{\pi_{\theta}} (s,(k,a^k)) \nabla_{\theta_1} \log(\pi^h_{\theta_1}(k|s))], \\
			& \nabla_{\theta_2}J(\left[\theta_1,\theta_2\right]) \\
			& = \mathbb{E}_{s\sim\rho^{\pi_{\theta}}, (k,a^k)\sim \pi_{\theta}(\cdot|s) } [Q^{\pi_{\theta}} (s,(k,a^k)) \nabla_{\theta_2} \log(\pi^l_{\theta_2}(a^k|s,k))],
		\end{align*}
	\end{proposition}
	\noindent where $\rho^{\pi_{\theta}}(s) = \sum_{t=0}^{\infty} \gamma^t P(s_t=s|\pi_{\theta})$ denotes the expected state visitation frequencies under the policy $\pi_{\theta}$, and $Q^{\pi_{\theta}}(s,(k,a^k)) = \mathbb{E}_{s_0,k_0,a_0^k,\dots}[\sum_{t=0}^{\infty} r(s_t,k_t,a_t^k) | s_0=s, k_0=k, a_0^k=a^k]$ denotes the Q-value function. Due to limited space, we defer detailed derivations to Appendix \ref{appendix_proof_hpg_multi_rounds}.  
	
	\begin{figure}
		\centering
		\includegraphics[width=0.99\columnwidth]{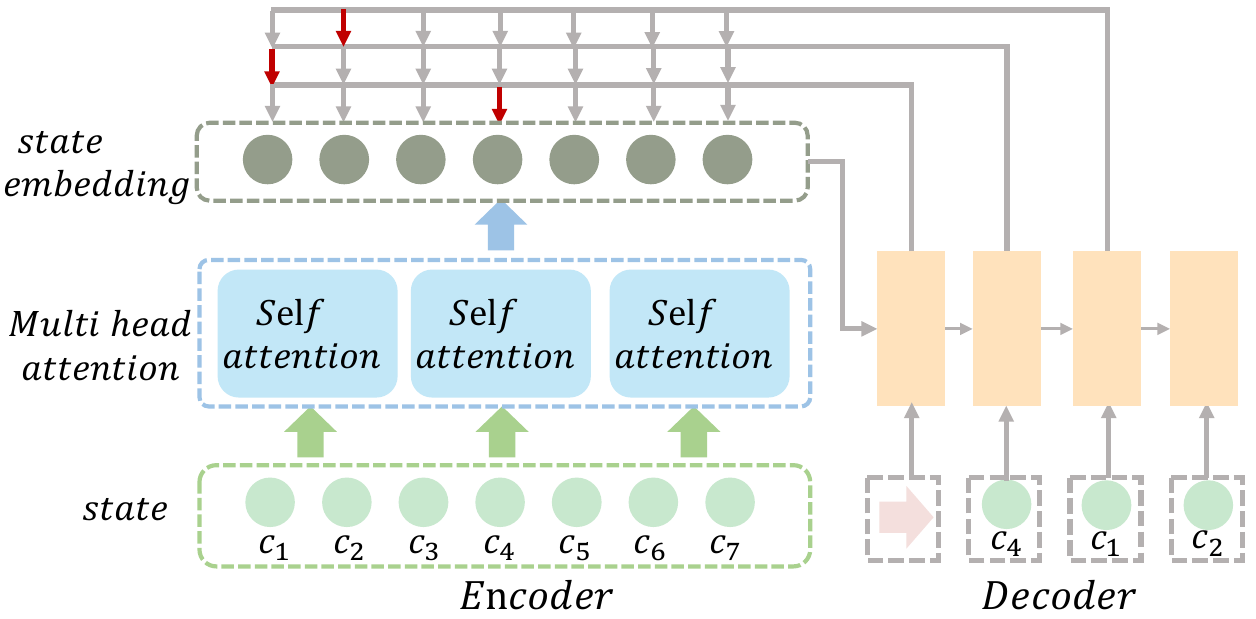}
		\caption{We instantiate the Set2Seq model via a multi-head attention encoder and a pointer decoder.}
		\label{fig:hem_model_illustration2}
	\end{figure}
	
	\subsection{Policy network: Set to Sequence model}\label{method:extension_set2seq}
	
	In this part, we formulate the lower-level model of HEM++ as a Set to Sequence (Set2Seq) model to improve the representation of input cuts by learning their \textit{order-independent} embeddings. 
	The details are as follows. 
	As shown in Fig \ref{fig:hem_model_illustration1}, we instantiate the lower-level model by a pointer network \cite{pn} (Seq2Seq model). The pointer network uses an LSTM encoder, which captures redundant order information when encoding input cuts. However, it is preferable for cut selection policies to be agnostic to the order of input cuts. The reason is that given the same instance and candidate cuts in different order, the corresponding optimal ordered subsets should be unchanged. Moreover, \cite{order_matters} has shown that vanilla Seq2Seq model can perform poorly when handling input sets. To learn order-independent embeddings of input cuts, we propose to formulate the lower-level model as a Set to Sequence (Set2Seq) model. Specifically, we first formulate the set of input cuts as a fully connected graph, and then propose an attention-pointer network to instantiate the Set2Seq model as shown in Fig \ref{fig:hem_model_illustration2}.
	The attention-pointer network comprises a multi-head attention encoder and a pointer decoder. The multi-head attention encoder is 
	a core component of the Transformer \cite{transformer}, which has been widely used in natural language processing \cite{transformer}\cite{bert} and computer vision \cite{vit}\cite{swin_transformer}.Compared to recurrent models, multi-head attention encoders can well capture representations invariant to the input order by leveraging the self-attention mechanism without the positional encoding. Please refer to Appendix \ref{appendix_mha} for details of the self-attention mechanism. The pointer decoder is the same as that of the pointer network (see Appendix \ref{appendix_pn_details} for details). 
	
	\subsection{Training: hierarchical proximal policy optimization}\label{method:hppo}
	Under the one round setting, the formulation of the cut selection is simplified as a context bandit. Instead, the formulation of the cut selection is a MDP under the multiple rounds setting. The action space of the MDP increases exponentially with the number of the cut separation rounds, which entails the challenge of efficient exploration \cite{vime}\cite{rnd}\cite{r3l}.
	To address this challenge, we propose to leverage the proximal policy optimization (PPO) method \cite{ppo}, an on-policy state-of-the-art training method. It is well-known that PPO is much more sample-efficient than reinforce. To train the policy with a hierarchical structure, we derive a \textbf{h}ierarchical \textbf{p}roximal \textbf{p}olicy \textbf{o}ptimization (HPPO) based on the extended formulation in Section \ref{sec:formulation_multiple_rounds}. Specifically, we denote the action space $\mathcal{A}^{\text{ppo}}$ by $[0,1] \times \mathcal{A}$, and the cut selection policy by $\pi^{\text{ppo}}(k,a_k|s) = \pi^h(k|s)\cdot\pi^l(a_k|s,k)$. We further denote the parameterized policy by $\pi^{\text{ppo}}_{\theta}(k,a_k|s) = \pi^h_{\theta_1}(k|s)\cdot\pi^l_{\theta_2}(a_k|s,k)$, where $\theta=\left[ \theta_1, \theta_2 \right]$. 
	We then derive the HPPO as follows. We denote the probability ratio of the current policies and old policies by 
	\begin{align}\label{eq: ratio}
		r(\theta) = \frac{\pi^{\text{ppo}}_{\theta}(k,a_k|s)}{\pi^{\text{ppo} }_{\theta_{\text{old}}}(k,a_k|s)} = \frac{ \pi^{h}_{\theta_1}(k|s)\cdot \pi^l_{\theta_2}(a_k|s,k) }{ \pi^{h}_{\theta_{\text{old}1}}(k|s)\cdot \pi^l_{\theta_{\text{old}2}}(a_k|s,k) }.
	\end{align}
	HPPO clips the ratio $r(\theta)$ to avoid destructively large policy updates following \cite{ppo}, i.e., 
	\begin{equation}\label{eq:clip_ratio}
		r_\epsilon^{\text{clip}}(\theta)=\begin{cases}
			1+\epsilon, &\hat{A}>0 \text{ and } r(\theta)\geq1+\epsilon; \\
			1-\epsilon, &\hat{A}<0 \text{ and } r(\theta)\leq1-\epsilon;\\
			r(\theta), &\text{otherwise}.
		\end{cases}
	\end{equation}
	Here $\epsilon$ is a hyperparameter, and $\hat{A}$ is an estimator of the advantage function \cite{rl_sutton}. 
	Then, we aim to find $\theta$ that maximizes the following objective:
	\begin{equation}\label{eq: ppo_obj}
		J(\theta)  = \mathbb{E}_{t}\left[ r_\epsilon^{\text{clip}}(\theta)\hat{A} \right].
	\end{equation}
	Here $\mathbb{E}_t[\cdot] = \mathbb{E}_{s\sim \rho^{\pi_{\theta_{\text{old}}}},(k,a_k)\sim \pi_{\theta_{\text{old}}}}[\cdot]$, and $\rho^{\pi_{\theta_{\text{old}}}}$ denotes the expected state visitation frequencies under the policy $\pi_{\theta_{\text{old}}}$. Then, the hierarchical policy gradient takes the form of 
	\begin{align}\label{eq: hierarchical_pg_of_ppo}
		\nabla_{\theta_1} J\left[\theta_1,\theta_2\right] 
		& = \mathbb{E}_{t} [\frac{ \pi^l_{\theta_2}(a_k|s,k) \cdot \nabla_{\theta_1} \pi^h_{\theta_1}(k|s) }{\pi^h_{\theta_{\text{old}1}}(k|s) \cdot \pi^l_{\theta_{\text{old}2}}(a_k|s,k)} \hat{A} ], \\
		\nabla_{\theta_2} J\left[\theta_1,\theta_2\right] 
		& = \mathbb{E}_{t}[\frac{ \pi^h_{\theta_1}(k|s) \cdot \nabla_{\theta_2} \pi^l_{\theta_2}(a_k|s,k) }{\pi^h_{\theta_{\text{old}1}}(k|s) \cdot \pi^l_{\theta_{\text{old}2}}(a_k|s,k)} \hat{A}],
	\end{align}
	where $\hat{A}\leq 0$ and $r(\theta)\geq 1-\epsilon$ or $\hat{A}\geq 0 $ and $r(\theta)\leq 1+\epsilon$. Please refer to Appendix \ref{appendix_hppo} for details of estimating the advantage function $\hat{A}$ and the implementation of HPPO. 
	
	\subsection{Discussion on advantages of HEM/HEM++}
	We discuss some additional advantages of HEM/HEM++ as follows. (1) Inspired by hierarchical reinforcement learning \cite{smdp_sutton, hiro}, HEM/HEM++ leverages the hierarchical structure of the cut selection task, which is important for efficient exploration in complex decision-making tasks. (2) Previous methods  \cite{tang_icml20, cut_ranking} usually train cut selection policies via black-box optimization methods such as evolution strategies \cite{es}. In contrast, HEM/HEM++ is differentiable and we train the HEM/HEM++ via gradient-based algorithms, which is more sample efficient than black-box optimization methods \cite{rl_sutton,trpo}. Although we can offline generate training samples as much as possible using a MILP solver, high sample efficiency is significant as generating samples can be extremely time-consuming in practice. 
	
	\subsection{Theoretical Analysis}
	We provide analysis to establish theoretical performance guarantees for HEM/HEM++ in this part. 
	
	We show that \textbf{the cutting plane algorithm with our proposed HEM/HEM++ selecting cuts finds optimal solutions of Integer Linear Programs in a finite number of iterations} under some mild assumptions. We provide a detailed analysis as follows. 
	Specifically, we focus on a pure cutting plane algorithm for solving Integer Linear Programs (ILP), which follows previous work \cite{neto2012simple, orlin1985finitely}. 
	We focus on the ILPs under the assumption that the objective function is defined by a positive integer vector and the feasible region is bounded. Specifically, the ILP takes the form of 
	\begin{align}\label{eq:ilp}
		\max_{\mathbf{x}} \{ \mathbf{c}^{\top} \mathbf{x}: \mathbf{x} \in \mathbb{P} \cap \mathbb{Z}^n \}     
	\end{align}
	, where $\mathbf{c} \in \mathbb{Z}^n_{+}$ and $\mathbb{P}$ denotes a bounded polyhedron in $\left[ 0,d \right]^n$ with $d$ being some positive integer. We assume that $\mathbb{P} \cap \mathbb{Z}^n$ is nonempty, i.e., $\mathbb{P}$ contains at least one integral point. 
	
	\begin{definition}
		(\textbf{Lexicographically Lower}) Given two vectors $\mathbf{x},\mathbf{y} \in \mathbb{R}^{n+1}$, the vector $\mathbf{x}=(x_0,\dots,x_n)$ is defined to be lexicographically lower than the vector $\mathbf{y}=(y_0,\dots,y_n)$ if there exists an integer $k\in \{0,\dots,n\}$ with $x_k < y_k$ and $x_i=y_i$ for all $i<k$. We write $\mathbf{x} <_{L} \mathbf{y}$. And similarly, we write $\mathbf{x} \leq_{L} \mathbf{y} $ if $\mathbf{x} <_{L} \mathbf{y}$ or $\mathbf{x} = \mathbf{y}$.
	\end{definition}
	
	For simplicity, we write the objective in (\ref{eq:ilp}) as a lexicographic optimization problem \cite{isermann1982linear}, taking the form of
	\begin{align}\label{eq:obj_lex}
		\text{lexmax} \{ z_0,z_1,\dots,z_n: \mathbf{z} \in \mathbb{S} \},  
	\end{align}
	where $z_0=\mathbf{c}^{\top}\mathbf{x}$, $z_i=x_i, \forall i\in\{1,\dots,n\}$ and $\mathbb{S} = \mathbb{Z} \times (\mathbb{P} \cap \mathbb{Z}^n)$. The lexicographic optimization objective in (\ref{eq:obj_lex}) aims to find a \textit{lexicographically maximum vector} in $\mathbb{S}$. We define the linear programming (LP) relaxation of the problem (\ref{eq:obj_lex}) by
	\begin{align}\label{eq:obj_lex_lp}
		\text{lexmax} \{ z_0,z_1,\dots,z_n: \mathbf{z} \in \mathbb{S}^{\prime} \},
	\end{align}
	where $\mathbb{S}^{\prime} = \mathbb{R} \times \mathbb{P}$. 
	
	Following \cite{neto2012simple}, we focus on the following cutting planes.  
	
	\begin{proposition}\label{pro:cuts}
		(\textbf{Cutting Planes}) Let $\mathbf{z} = (z_0,z_1,\dots,z_n) \in \mathbb{R} \times \left[ 0,d \right]^n $. Then the following \textbf{linear inequalities} are satisfied  by any integer vector $\mathbf{x} \in \mathbb{Z} \times (\left[ 0,d \right] \cap \mathbb{Z})^n$ which is lexicographically lower than $\mathbf{z}$ (i.e., $\mathbf{x} \leq_{L} \mathbf{z}$)
		\begin{align}\label{eq:cuts}
			x_i + \sum_{j=0}^{i-1} a_{i-j} (x_j - \lceil z_j \rceil) \leq \lfloor z_i \rfloor
		\end{align}
		for all $i\in \{0,\dots, n\}$ with $a_1=d$ and $a_k=d(1+\sum_{j=1}^{k-1}a_j)$ for any integer $k \in \{2,\dots,n\}$, and $\sum_{j=0}^{-1}a_{0}(x_0-\lceil z_0 \rceil) = 0$.
	\end{proposition}
	
	We defer detailed proof of Proposition \ref{pro:cuts} to Appendix \ref{proof:cuts}.
	We denote the inequality (\ref{eq:cuts}) with some particular index value $i \in \{ 0,\dots,n \}$ by (\ref{eq:cuts})$_{i}$. 
	Thus, we can generate $n+1$ cutting planes according to the inequalities (\ref{eq:cuts}).
	
	The cutting plane algorithm with HEM/HEM++ starts with solving the LP relaxation of (\ref{eq:obj_lex}) (i.e., the problem (\ref{eq:obj_lex_lp})). Then we generate $n+1$ cutting planes according to (\ref{eq:cuts}), and HEM/HEM++ select a cut to be added to the LP relaxation. Then the procedure proceeds to the next iteration by solving the new LP relaxation. Finally, the algorithm is iterated until the optimal integer solution is found. We show that the cutting plane algorithm with HEM/HEM++ selecting cuts finds optimal integer solutions of the ILPs in a finite number of iterations under the following mild assumption.
	
	\begin{assumption}\label{assum:hem}
		HEM/HEM++ learns a cut selection policy to select the $k$-th cut at each iteration. Note that $k = \text{argmin}_{j}\{z_j^* \not\in \mathbb{Z}, j=0,\dots,n\}$, where $\mathbf{z}_t^{\star} = (z_0^*, z_1^*,\dots,z_n^*)$ denotes the optimal solution of the LP relaxation at iteration $t$.
	\end{assumption}
	
	\begin{theorem}
		With the cutting planes in Proposition \ref{pro:cuts} and under the Assumption \ref{assum:hem}, the cutting plane algorithm with HEM/HEM++ selecting cuts finds an optimal solution of the ILP (\ref{eq:obj_lex}) in a finite number of iterations.
		\label{theorem:convergence}
	\end{theorem}
	We defer detailed proof to Appendix \ref{proof:cuts_convergence}.        
	
	\begin{algorithm}[t]
		\caption{Pseudo code for extracting order rules from hierarchical sequence model}.
		\label{alg:hem_order_rules}
		\begin{algorithmic}[1]
			\STATE \textbf{Input} Hierarchical sequence model $\pi_{[\theta_1,\theta_2]}$, MILP training dataset $\mathcal{D}_{\text{train}}$
			\STATE \textbf{Input} Top-3 cutting plane category counters $\mathcal{C}_{1},\mathcal{C}_{2},\mathcal{C}_{3}=\{0\}$, cutting plane categories $\mathcal{K}$
			\FOR{each instance $\mathcal{D}_i$ in $\mathcal{D}_{\text{train}}$}
			\STATE Solve $\mathcal{D}_i$ with HEM selecting cuts
			\STATE Record the order of categories of selected cuts
			\STATE Record the top-3 categories in the order, and find their corresponding indexes in $\mathcal{K}$, i.e., $i,j,k$.
			\STATE Increase the counters: $\mathcal{C}_{1i} += 1$, $\mathcal{C}_{2j} += 1$, $\mathcal{C}_{3k} += 1$
			\ENDFOR
			\STATE Find the maximum counter indexes: $i_{\text{top1}} = \text{argmax} (\mathcal{C}_1)$, $i_{\text{top2}} = \text{argmax} (\mathcal{C}_2)$, $i_{\text{top3}} = \text{argmax} (\mathcal{C}_3)$
			\STATE \textbf{Return} the top-3 categories, i.e., $\mathcal{K}_{i_{\text{top1}}}$, $\mathcal{K}_{i_{\text{top2}}}$, $\mathcal{K}_{i_{\text{top3}}}$
		\end{algorithmic}
	\end{algorithm}
	
	\begin{table*}[t]
		\caption{Statistical description of used datasets. In all datasets, $m$ denotes the average number of constraints and $n$ denotes the average number of variables. NCuts denotes the number of generated candidate cuts. Inference Time denotes the inference time of our proposed HEM given the average number of candidate cuts.}
		\label{exp:datasets}
		\centering
		\resizebox{\textwidth}{!}{
			\begin{tabular}{@{}cccccccccc@{}}
				\toprule
				\toprule
				Datasets & Set Covering & Maximum Independent Set & Multiple Knapsack & MIK & CORLAT & Load Balancing & Anonymous & MIPLIB mixed neos & MIPLIB mixed supportcase \\ \midrule
				$m$ & 500 & 1953 & 72 & 346 & 486 & 64304 & 49603 & 5660 & 19910 \\
				$n$ & 1000 & 500 & 720 & 413 & 466 & 61000 & 37881 & 6958 & 19766 \\ \midrule
				NCuts Avg (stdev)  & 780.51 (290) & 57.04 (16) & 45.00 (13) & 62.00 (13) & 60.00 (33) & 392.53 (33) & 79.40 (73) & 239.00 (154) & 173.25 (267) \\ 
				Inference Time (s) & 1.58 & 0.11 & 0.09 & 0.12 & 0.12 & 0.77 & 0.15 & 0.47 & 0.34 \\ \bottomrule
			\end{tabular}
		}
	\end{table*}
	
	\section{Extracting Order Rules from HEM}\label{method:order_heuristics}
	To enhance modern MILP solvers, we can directly deploy our HEM to MILP solvers to improve their efficiency. In addition, we can also extract more effective heuristic rules than manually-designed heuristics from learned policies to facilitate the deployment of HEM. Extracting rules from HEM leads to two major advantages. First, the extracted rules are well readable for human experts, and thus easy to debug. Second, we can deploy the extracted rules to purely CPU-based environments, which are applicable to the common solver deployment environments, i.e., the CPU computing-intensive environments.
	However, it is challenging to directly extract the cut selection rules from HEM. As it involves three problems (i.e., \textbf{(P1)-(P3)}) that are very complex, it is difficult to extract general cut selection rules from HEM. Furthermore, existing heuristics aim to tackle \textbf{(P1)-(P2)} but neglect the importance of tackling \textbf{(P3)}. Therefore, we propose to extract \textit{order rules} for tackling \textbf{(P3)} to enhance existing manually-designed heuristics. 
	
	Specifically, we record the order of categories of selected cuts for each instance. For generality, we aim to extract instance-independent order rules rather than instance-dependent.
	To this end, we count the number of times that each cutting plane category 
	is located at the k-th ($k=1,2,3,\dots$) position in the order across all instances. Then we record the cutting plane category with the highest number of times that is located at the k-th ($k=1,2,3,\dots$) position, respectively. Finally, we increase the priority of the cut generators according to the order of cutting plane categories. We denote the default heuristics in MILP solvers with extracted order rules by Default+. We summarize the procedure of extracting order rules from HEM in Algorithm \ref{alg:hem_order_rules}. Please refer to Section \ref{exp:order_rules} for the evaluation of Default+.

	\section{Experiments}
	
	Our experiments have six main parts. (1) We evaluate HEM on three classical MILP problems and six challenging MILP problem benchmarks from diverse application areas (see Section \ref{exp:main_evaluation}).
	(2) We evaluate whether HEM++ improves HEM (see Section \ref{exp:hem++}). (3) We perform carefully designed ablation studies to provide further insight into HEM (see Section \ref{exp:ablation_study}). (4) We test whether HEM can generalize to instances significantly larger than those seen during training (see Section \ref{exp:generalization}). (5) We perform carefully designed 
	visualization experiments and explainability analysis (see Section \ref{exp:visu_and_explain}). (6) We deploy our approach to real-world challenging MILP problems (see Section \ref{exp:deployment}). 
	
	\subsection{Experiments setup}

	\subsubsection{Benchmarks}
	We evaluate our approach on nine $\mathcal{NP}$-hard MILP problem benchmarks, which consist of three classical synthetic MILP problems and six challenging MILP problems from diverse application areas. We divide the nine problem benchmarks into three categories according to the difficulty of solving them using the SCIP 8.0.0 solver \cite{scip8}. We call the three categories easy, medium, and hard datasets, respectively. (1) \textbf{Easy datasets} comprise three widely used synthetic MILP problem benchmarks: Set Covering \cite{setcover}, Maximum Independent Set \cite{mis}, and Multiple Knapsack \cite{tree_mdp}. We artificially generate instances following \cite{nips19_gcnn, rl_branch}. (2) \textbf{Medium datasets} comprise MIK \cite{mik} and CORLAT \cite{corlat}, which are widely used benchmarks for evaluating MILP solvers \cite{learning_to_search, milp_google}. (3) \textbf{Hard datasets} include the Load Balancing problem, inspired by large-scale systems at Google, and the Anonymous problem, inspired by a large-scale industrial application \cite{nips21_ml4co_competition}. Moreover, hard datasets contain benchmarks from MIPLIB 2017 (MIPLIB) \cite{miplibs_2017}. Although \cite{adaptive_cut_selection} has shown that directly learning over the full MIPLIB can be extremely challenging, we propose to learn over subsets of MIPLIB. We construct two subsets, called MIPLIB mixed neos and MIPLIB mixed supportcase. Due to limited space, please see Appendix \ref{appendix_main_datasets} for details of these datasets.
	
	We summarize the statistical description of these datasets in Table \ref{exp:datasets}. Let $n,m$ denote the average number of variables and constraints in the MILPs. Let $m\times n$ denote the size of the MILPs. It is worth emphasizing that the maximum size of the dataset we use is two orders of magnitude larger than that used in previous work \cite{tang_icml20, l2c_lookahead}. Furthermore, we evaluate the average inference time of HEM given the generated cuts. The results in Table \ref{exp:datasets} show that the computational overhead of HEM is quite low.

	\begin{table*}[t]
		\caption{Policy evaluation on easy, medium, and hard datasets. The best performance is marked in bold. For all datasets, the results show that HEM significantly outperforms the baselines in terms of solving time and primal-dual gap integral.}
		\label{evaluation_all}
		\centering
		\resizebox{0.96\textwidth}{!}{
			\begin{tabular}{@{}cccccccccc@{}}
				\toprule
				\toprule
				& \multicolumn{3}{c}{Easy: Set Covering ($n=1000, \,\,m=500$)} & \multicolumn{3}{c}{Easy: Max Independent Set ($n=500,\,\,m=1953$)} & \multicolumn{3}{c}{Easy: Multiple Knapsack ($n=720,\,\,m=72$)} \\ \midrule
				Method & Time(s) $\downarrow$ & Improvement (time, \%) $\uparrow$ & PD integral $\downarrow$ & Time(s) $\downarrow$ & Improvement (time, \%) $\uparrow$ & PD integral $\downarrow$ & Time(s) $\downarrow$ & Improvement (time, \%) $\uparrow$ & PD integral $\downarrow$\\\cmidrule(r){1-4} \cmidrule(lr){5-7} \cmidrule(l){8-10}
				NoCuts & 6.31 (4.61) & NA & 56.99 (38.89) & 8.78 (6.66) & NA & 71.31 (51.74) & 9.88 (22.24) & NA & 16.41 (14.16) \\
				Default & 4.41 (5.12) & 29.90 & 55.63 (42.21) & 3.88 (5.04) & 55.80 & 29.44 (35.27) & 9.90 (22.24) & -0.20 & 16.46 (14.25) \\
				Random & 5.74 (5.19) & 8.90 & 67.08 (46.58) & 6.50 (7.09) & 26.00 & 52.46 (53.10) & 13.10 (35.51) & -32.60 & 20.00 (25.14) \\
				NV & 9.86 (5.43) & -56.50 & 99.77 (53.12) & 7.84 (5.54) & 10.70 & 61.60 (43.95) & 13.04 (36.91) & -32.00 & 21.75 (24.71) \\
				Eff & 9.65 (5.45) & -53.20 & 95.66 (51.71) & 7.80 (5.11) & 11.10 & 61.04 (41.88) & 9.99 (19.02) & -1.10 & 20.49 (22.11) \\\midrule
				SBP & 1.91 (0.36) & 69.60 & 38.96 (8.66) & 2.43 (5.55) & 72.30 & 21.99 (40.86) & 7.74 (12.36) & 21.60 & 16.45 (16.62) \\
				HEM (Ours) & \textbf{1.85 (0.31)} & \textbf{70.60} & \textbf{37.92 (8.46)} & \textbf{1.76 (3.69)} & \textbf{80.00} & \textbf{16.01 (26.21)} & \textbf{6.13 (9.61)} & \textbf{38.00} & \textbf{13.63 (9.63)} \\ \bottomrule
			\end{tabular}
		}
		\newline
		\vspace{2mm}
		\newline
		\resizebox{0.96\textwidth}{!}{
			\begin{tabular}{@{}cccccccccc@{}}
				\toprule
				\toprule
				& \multicolumn{3}{c}{Medium: MIK ($n=413,\,\,m=346$)} & \multicolumn{3}{c}{Medium: Corlat ($n=466,\,\,m=486$)} & \multicolumn{3}{c}{Hard: Load Balancing ($n=61000,\,\,m=64304$)} \\ \midrule
				\multirow{2}{*}{Method} & \multirow{2}{*}{Time(s) $\downarrow$} & \multirow{2}{*}{PD integral $\downarrow$} & \multirow{2}{*}{\begin{tabular}[c]{@{}c@{}}Improvement $\uparrow$\\      (PD integral, \%)\end{tabular} } & \multirow{2}{*}{Time(s) $\downarrow$} & \multirow{2}{*}{PD integral $\downarrow$} & \multirow{2}{*}{\begin{tabular}[c]{@{}c@{}}Improvement $\uparrow$\\      (PD integral, \%)\end{tabular} } & \multirow{2}{*}{Time(s) $\downarrow$} & \multirow{2}{*}{PD integral $\downarrow$} & \multirow{2}{*}{\begin{tabular}[c]{@{}c@{}}Improvement $\uparrow$\\      (PD integral, \%)\end{tabular} } \\
				&  &  &  &  &  &  &  &  &  \\\cmidrule(r){1-4} \cmidrule(lr){5-7} \cmidrule(l){8-10}
				NoCuts & 300.01 (0.009) & 2355.87 (996.08) & NA & 103.30 (128.14) & 2818.40 (5908.31) & NA & 300.00 (0.12) & 14853.77 (951.42) & NA \\
				Default & 179.62 (122.36) & 844.40 (924.30) & 64.10 & 75.20 (120.30) & 2412.09 (5892.88) & 14.40 & 300.00 (0.06) & 9589.19 (1012.95) & 35.40 \\
				Random & 289.86 (28.90) & 2036.80 (933.17) & 13.50 & 84.18 (124.34) & 2501.98 (6031.43) & 11.20 & 300.00 (0.09) & 13621.20 (1162.02) & 8.30 \\
				NV & 299.76 (1.32) & 2542.67 ( 529.49) & -7.90 & 90.26 (128.33) & 3075.70 (7029.55) & -9.10 & 300.00 (0.05) & 13933.88 (971.10) & 6.20 \\
				Eff & 298.48 (5.84) & 2416.57 (642.41) & -2.60 & 104.38 (131.61) & 3155.03 (7039.99) & -11.90 & 300.00 (0.07) & 13913.07 (969.95) & 6.30 \\\midrule
				SBP & 286.07 (41.81) & 2053.30 (740.11) & 12.80 & 70.41 (122.17) & 2023.87 (5085.96) & 28.20 & 300.00 (0.10) & 12535.30 (741.43) & 15.60 \\ 
				HEM(Ours) & \textbf{176.12 (125.18)} & \textbf{785.04 (790.38)} & \textbf{66.70} & \textbf{58.31 (110.51)} & \textbf{1079.99 (2653.14)} & \textbf{61.68} & \textbf{300.00 (0.04)} & \textbf{9496.42 (1018.35)} & \textbf{36.10} \\ \bottomrule
			\end{tabular}
		}
		\newline
		\vspace{2mm}
		\newline
		\resizebox{0.96\textwidth}{!}{
			\begin{tabular}{@{}cccccccccc@{}}
				\toprule
				\toprule
				& \multicolumn{3}{c}{Hard: Anonymous ($n=37881,\,\,m=49603$)} & \multicolumn{3}{c}{Hard: MIPLIB mixed neos ($n=6958,\,\,m=5660$)} & \multicolumn{3}{c}{Hard: MIPLIB mixed supportcase ($n=19766,\,\,m=19910$)} \\ \midrule
				\multirow{2}{*}{Method} & \multirow{2}{*}{Time(s) $\downarrow$} & \multirow{2}{*}{PD integral $\downarrow$} & \multirow{2}{*}{\begin{tabular}[c]{@{}c@{}}Improvement $\uparrow$\\      (PD integral, \%)\end{tabular}} & \multirow{2}{*}{Time(s) $\downarrow$} & \multirow{2}{*}{PD integral $\downarrow$} & \multirow{2}{*}{\begin{tabular}[c]{@{}c@{}}Improvement $\uparrow$\\      (PD integral, \%)\end{tabular}} & \multirow{2}{*}{Time(s) $\downarrow$} & \multirow{2}{*}{PD integral $\downarrow$} & \multirow{2}{*}{\begin{tabular}[c]{@{}c@{}}Improvement $\uparrow$\\      (PD integral, \%)\end{tabular}} \\
				&  &  &  &  &  &  &  &  &  \\\cmidrule(r){1-4} \cmidrule(lr){5-7} \cmidrule(l){8-10}
				NoCuts & 246.22 (94.90) & 18297.30 (9769.42) & NA & 253.65 (80.29) & 14652.29 (12523.37) & NA & 170.00 (131.60) & 9927.96 (11334.07) & NA \\
				Default & 244.02 (97.72) & 17407.01 (9736.19) & 4.90 & 256.58 (76.05) & 14444.05 (12347.09) & 1.42 & 164.61 (135.82) & 9672.34 (10668.24) & 2.57 \\
				Random & 243.49 (98.21) & 16850.89 (10227.87) & 7.80 & 255.88 (76.65) & 14006.48 (12698.76) & 4.41 & 165.88 (134.40) & 10034.70 (11052.73) & -1.07 \\
				NV & 242.01 (98.68) & 16873.66 (9711.16) & 7.80 & 263.81 (64.10) & 14379.05 (12306.35) & 1.86 & \textbf{161.67 (131.43)} & 8967.00 (9690.30) & 9.68 \\
				Eff & 244.94 (93.47) & 17137.87 (9456.34) & 6.30 & 260.53 (68.54) & 14021.74 (12859.41) & 4.30 & 167.35 (134.99) & 9941.55 (10943.48) & -0.14 \\\midrule
				SBP & 245.71 (92.46) & 18188.63 (9651.85) & 0.59 & 256.48 (78.59) & 13531.00 (12898.22) & 7.65 & 165.61 (135.25) & 7408.65 (7903.47) & 25.37 \\ 
				HEM(Ours) & \textbf{241.68 (97.23)} & \textbf{16077.15 (9108.21)} & \textbf{12.10} & \textbf{248.66 (89.46)} & \textbf{8678.76 (12337.00)} & \textbf{40.77} & 162.96 (138.21) & \textbf{6874.80 (6729.97)} & \textbf{30.75} \\ \bottomrule
			\end{tabular}
		}
	\end{table*}
	
	\subsubsection{Implementation details}
	Throughout all experiments, we use SCIP 8.0.0 \cite{scip8} as the backend solver, which is the state-of-the-art open source solver, and is widely used in research of machine learning for combinatorial optimization \cite{nips19_gcnn,cut_ranking,adaptive_cut_selection,milp_google}. Following \cite{nips19_gcnn, cut_ranking, l2c_lookahead}, we only allow cutting plane generation and selection at the root node. Throughout all experiments, we set the cut separation rounds as one unless otherwise specified. We keep all the other SCIP parameters to default so as to make comparisons as fair and reproducible as possible. We emphasize that all of the SCIP solver's advanced features, such as presolve and heuristics, are open, which ensures that our setup is consistent with the practice setting. Throughout all experiments, we set the solving time limit as 300 seconds. For completeness, we also evaluate HEM with a much longer time limit of three hours. The results are given in Appendix \ref{appendix_results_long_test_time}. We train HEM with ADAM \cite{adam} using the PyTorch \cite{torch}. Additionally, we also provide another implementation using the MindSpore \cite{mindspore}.      
	For simplicity, we split each dataset into the train and test sets with $80\%$ and $20\%$ instances. To further improve HEM, one can construct a valid set for hyperparameters tuning. We train our model on the train set, and select the best model on the train set to evaluate on the test set. 
	Please refer to Appendix \ref{appendix_imple_hyper} for implementation details, hyperparameters, and hardware specification.
	
	
	\subsubsection{Main Baselines}
	Our main baselines include five widely-used human-designed cut selection rules and a state-of-the-art (SOTA) learning-based method. Cut selection rules include NoCuts, Random, Normalized Violation (NV), Efficacy (Eff), and Default. NoCuts does not add any cuts. 
	Default denotes the default rules used in SCIP 8.0.0. For learning-based methods, we implement a SOTA learning-based method \cite{tang_icml20}, namely score-based policy (SBP), as our main learning baseline. 
	All baselines, except Default and NoCuts, select a fixed ratio of candidate cuts with high scores. 
	Please see Appendix \ref{appendix_imple_baselines_one_round} for implementation details of these baselines.
	
	\subsubsection{Evaluation metrics}
	We use two widely used evaluation metrics, i.e., the average solving time (Time, lower is better), and the average primal-dual gap integral (PD integral, lower is better). Additionally, we provide more results in terms of another two metrics, i.e., the average number of nodes and the average primal-dual gap, in Appendix \ref{appendix_results_more_metrics}. Furthermore, to evaluate different cut selection methods compared to pure branch-and-bound without cutting plane separation, we propose an \textit{Improvement} metric. Specifically, we define the metric by 
	$  \text{Im}_{M}(\cdot) = \frac{M(\text{NoCuts}) - M(\cdot)}{M(\text{NoCuts})},
	$
	where $M(\text{NoCuts})$ represents the performance of NoCuts, and $M(\cdot)$ represents a mapping from a method to its performance. The improvement metric represents the improvement of a given method compared to NoCuts. \textbf{We mainly focus on the Time metric on the easy datasets}, as the solver can solve all instances to optimality within the given time limit. However, HEM and the baselines cannot solve all instances to optimality within the time limit on the medium and hard datasets. As a result, the average solving time of those unsolved instances is the same, which makes it difficult to distinguish the performance of different cut selection methods using the Time metric. Therefore, \textbf{we mainly focus on the PD integral metric on the medium and hard datasets}. The PD integral is also a well-recognized metric for evaluating the solver performance \cite{nips21_ml4co_competition, cao2022ml4co}.
	
	\subsection{Comparative evaluation of HEM}\label{exp:main_evaluation}
	The results in Table \ref{evaluation_all} suggest the following.
	(1) \textbf{Easy datasets.} HEM significantly outperforms all the baselines on the easy datasets, especially on Maximum Independent Set and Multiple Knapsack. SBP achieves much better performance than all the rule-based baselines, demonstrating that our implemented SBP is a strong baseline. Compared to SBP, HEM improves the Time by up to $16.4\%$ on the three datasets,  demonstrating the superiority of our method over the SOTA learning-based method. 
	(2) \textbf{Medium datasets.} On MIK and CORLAT, HEM still outperforms all the baselines. Especially on CORLAT, HEM achieves at least $33.48\%$ improvement in terms of the PD integral compared to the baselines. 
	(3) \textbf{Hard datasets.}
	HEM significantly outperforms the baselines in terms of the PD integral on several problems in the hard datasets. HEM achieves outstanding performance on two challenging datasets from MIPLIB 2017 and real-world problems (Load Balancing and Anonymous), 
	demonstrating the powerful ability to enhance MILP solvers with HEM in large-scale real-world applications. Moreover, 
	SBP performs extremely poorly on several medium and hard datasets, which implies that it can be difficult to learn good cut selection policies on challenging MILP problems.

	\begin{table*}[t]
		\caption{Comparison with two more learning baselines, i.e., AdaptiveCutsel \cite{adaptive_cut_selection} and Lookahead \cite{l2c_lookahead}, on Maximum Independent Set, CORLAT, and MIPLIB mixed supportcase. The best
			performance is marked in bold.}
		\label{exp:evaluation_more_learning_baselines}
		\centering
		\resizebox{0.96\textwidth}{!}{
			\begin{tabular}{@{}cccccccccc@{}}
				\toprule
				\toprule
				& \multicolumn{3}{c}{Easy: Maximum Independent Set ($n=500,\,\,m=1953$)} & \multicolumn{3}{c}{Medium: CORLAT ($n=466,\,\,m=486$)} & \multicolumn{3}{c}{Hard: MIPLIB mixed supportcase ($n=19766,\,\,m=19910$)} \\ \midrule
				\multirow{2}{*}{Method} & \multirow{2}{*}{Time(s) $\downarrow$} & \multirow{2}{*}{\begin{tabular}[c]{@{}c@{}}Improvement $\uparrow$\\      (Time, \%)\end{tabular}} & \multirow{2}{*}{PD integral $\downarrow$} & \multirow{2}{*}{Time(s) $\downarrow$} & \multirow{2}{*}{PD integral $\downarrow$} & \multirow{2}{*}{\begin{tabular}[c]{@{}c@{}}Improvement $\uparrow$\\      (PD integral, \%)\end{tabular}} & \multirow{2}{*}{Time(s) $\downarrow$} & \multirow{2}{*}{PD integral $\downarrow$} & \multirow{2}{*}{\begin{tabular}[c]{@{}c@{}}Improvement $\uparrow$\\      (PD integral, \%)\end{tabular}} \\
				&  &  &  &  &  &  &  &  &  \\ \cmidrule(r){1-4} \cmidrule(lr){5-7} \cmidrule(l){8-10}
				NoCuts & 8.78 (6.66) & NA & 71.31 (51.74) & 103.31 (128.14) & 2818.41 (5908.31) & NA & 170.00 (131.60) & 9927.96 (11334.07) & NA \\
				Default & 3.88 (5.04) & 55.80 & 29.44 (35.27) & 75.21 (120.30) & 2412.09 (5892.88) & 14.42 & 164.61 (135.82) & 9672.34 (10668.24) & 2.57 \\ \midrule
				AdaptiveCutsel & 2.74 (3.92) & 68.79 & 21.40 (27.95) & 74.61 (120.74) & 1619.50 (3862.10) & 42.54 & 161.03 (136.40) & 8769.63 (10008.97) & 11.67 \\
				Lookahead & 2.27 (5.00) & 74.15 & 20.89 (37.38) & 74.38 (125.47) & 1493.00 (3684.44) & 47.03 & \textbf{159.61 (130.97)} & 9293.82 (10428.78) & 6.39 \\
				SBP & 2.43 (5.55) & 72.30 & 21.99 (40.86) & 70.42 (122.17) & 2023.87 (5085.96) & 28.19 & 165.61 (135.25) & 7408.65 (7903.47) & 25.37 \\
				HEM (Ours) & \textbf{1.76 (3.69)} & \textbf{80.00} & \textbf{16.01 (26.21)} & \textbf{58.31 (110.51)} & \textbf{1079.99 (2653.14)} & \textbf{61.68} & 162.96 (138.21) & \textbf{6874.80 (6729.97)} & \textbf{30.75} \\ \bottomrule
			\end{tabular}
		}
	\end{table*}

	\begin{table*}[t]
		\caption{The average performance of HEM++ and the baselines in the multiple rounds setting on Maximum Independent Set, CORLAT, and MIPLIB mixed supportcase. The best performance is marked
			in bold.}
		\label{exp:evaluation_multiple_rounds}
		\centering
		\resizebox{0.96\textwidth}{!}{
			\begin{tabular}{@{}cccccccccc@{}}
				\toprule\toprule
				& \multicolumn{3}{c}{Easy: Maximum Independent Set ($n=500,\,\,m=1953$)} & \multicolumn{3}{c}{Medium: CORLAT ($n=466,\,\,m=486$)} & \multicolumn{3}{c}{Hard: MIPLIB mixed supportcase ($n=19766,\,\,m=19910$)} \\ \midrule
				\multirow{2}{*}{Method} & \multirow{2}{*}{Time (s) $\downarrow$} & \multirow{2}{*}{\begin{tabular}[c]{@{}c@{}}Improvement $\uparrow$\\      (Time, \%)\end{tabular}} & \multirow{2}{*}{PD integral $\downarrow$} & \multirow{2}{*}{Time (s) $\downarrow$} & \multirow{2}{*}{PD integral $\downarrow$} & \multirow{2}{*}{\begin{tabular}[c]{@{}c@{}}Improvement $\uparrow$\\      (PD integral, \%)\end{tabular}} & \multirow{2}{*}{Time (s) $\downarrow$} & \multirow{2}{*}{PD integral $\downarrow$} & \multirow{2}{*}{\begin{tabular}[c]{@{}c@{}}Improvement $\uparrow$\\      (PD integral, \%)\end{tabular}} \\
				&  &  &  &  &  &  &  &  &  \\ \cmidrule(r){1-4} \cmidrule(lr){5-7} \cmidrule(l){8-10}
				NoCuts & 8.78 (6.66) & NA & 71.32 (51.74) & 103.31 (128.14) & 2818.41 (5908.31) & NA & 170.00 (131.60) & 9927.96 (11334.07) & NA \\
				Default & 0.28 (0.09) & 96.78 & 6.43 (1.17) & 18.19 (63.09) & 1293.01 (4988.62) & 54.12 & 149.40 (135.06) & 8273.98 (8854.58) & 16.66 \\
				Random & 1.01 (2.38) & 88.53 & 10.91 (12.28) & 21.45 (63.13) & 1537.53 (4990.91) & 45.45 & 153.20 (136.38) & 8407.73 (9667.58) & 15.31 \\
				NV & 1.24 (1.62) & 85.91 & 12.38 (8.63) & 77.21 (119.69) & 3083.72 (7544.76) & -9.41 & 156.46 (130.11) & 8934.07 (10299.78) & 10.01 \\
				Eff & 0.27 (0.09) & 96.97 & 6.39 (1.11) & 23.12 (70.35) & 2305.60 (7032.43) & 18.20 & \textbf{128.88 (126.15)} & 7368.99 (8642.19) & 25.77 \\ \midrule
				SBP & 0.26 (0.14) & 97.04 & 6.57 (1.70) & 8.93 (35.76) & 692.31 (2678.17) & 75.44 & 147.92 (131.63) & 7360.16 (7782.50) & 25.86 \\
				HEM++ (Ours) & \textbf{0.21 (0.09)} & \textbf{97.63} & \textbf{6.03 (1.12)} & \textbf{4.96 (10.08)} & \textbf{460.68 (980.31)} & \textbf{83.65} & 141.71 (133.49) & \textbf{6545.59 (7241.17)} & \textbf{34.07} \\ \bottomrule
			\end{tabular}
		}
	\end{table*}
	
	\begin{table*}[t]
		\centering
		\caption{Evaluation of HEM++ on three extremely hard datasets under the one round setting. The best performance is marked
			in bold. The results show that HEM++ outperforms HEM in terms of the primal-dual gap integral.}
		\label{eva:HEM++}
		\resizebox{0.96\textwidth}{!}{
			\begin{tabular}{@{}cccccccccc@{}}
				\toprule
				\toprule
				& \multicolumn{3}{c}{Hard: Anonymous ($n=37881,\,\,m=49603$)} & \multicolumn{3}{c}{Hard: MIPLIB mixed neos ($n=6958,\,\,m=5660$)} & \multicolumn{3}{c}{Hard: MIPLIB mixed supportcase ($n=19766,\,\,m=19910$)} \\ \midrule
				\multirow{2}{*}{Method} & \multirow{2}{*}{Time (s) $\downarrow$} & \multirow{2}{*}{PD integral $\downarrow$} & \multirow{2}{*}{\begin{tabular}[c]{@{}c@{}}Improvement $\uparrow$ \\      (\%, PD integral)\end{tabular}} & \multirow{2}{*}{Time (s) $\downarrow$} & \multirow{2}{*}{PD integral $\downarrow$} & \multirow{2}{*}{\begin{tabular}[c]{@{}c@{}}Improvement $\uparrow$\\      (\%, PD integral)\end{tabular}} & \multirow{2}{*}{Time (s) $\downarrow$} & \multirow{2}{*}{PD integral $\downarrow$} & \multirow{2}{*}{\begin{tabular}[c]{@{}c@{}}Improvement $\uparrow$\\      (\%, PD integral)\end{tabular}} \\
				&  &  &  &  &  &  &  &  &  \\ \cmidrule(r){1-4} \cmidrule(lr){5-7} \cmidrule(l){8-10}
				NoCuts & 246.22 (94.90) & 18297.30 (9769.42) & NA & 253.65 (80.29) & 14652.29 (12523.37) & NA & 170.00 (131.60) & 9927.96 (11334.07) & NA \\
				Default & 244.02 (97.72) & 17407.01 (9736.19) & 4.90 & 256.58 (76.05) & 14444.05 (12347.09) & 1.42 & 164.61 (135.82) & 9672.34 (10668.24) & 2.57 \\ \midrule
				SBP & 245.71 (92.46) & 18188.63 (9651.85) & 0.59 & 256.48 (78.59) & 13531.00 (12898.22) & 7.65 & 165.61 (135.25) & 7408.65 (7903.47) & 25.37 \\
				HEM (Ours) & \textbf{241.68 (97.23)} & 16077.15 (9108.21) & 12.10 & \textbf{248.66 (89.46)} & 8678.76 (12337.00) & 40.77 & 162.96 (138.21) & 6874.80 (6729.97) & 30.75 \\
				HEM++ (Ours) & 243.23 (96.62) & \textbf{15444.70 (9260.60)} & \textbf{15.59} & 252.07 (83.06) & \textbf{8557.46 (12395.07)} & \textbf{41.60} & \textbf{153.83 (131.55)} & \textbf{6586.04 (6313.25)} & \textbf{33.66} \\ \bottomrule
			\end{tabular}
		}
	\end{table*}
	
	\subsubsection{Comparison with more learning baselines} 
	To fully evaluate the superiority of HEM, we compare HEM with two more learning baselines, i.e., AdaptiveCutsel \cite{adaptive_cut_selection} and Lookahead \cite{l2c_lookahead}. As shown in Table \ref{exp:evaluation_more_learning_baselines}, the results demonstrate that HEM significantly outperforms the two learning-based methods by a large margin in terms of the Time (up to 11.21\% improvement) and PD integral (up to 24.36\% improvement). Both SBP and Lookahead significantly outperform the default heuristic, showing that learning cut selection is important for improving solver performance. Moreover, SBP performs on par with Lookahead on easy and medium datasets (Maximum Independent Set and CORLAT), while SBP outperforms Lookahead on hard datasets (MIPLIB mixed supportcase). This demonstrates that SBP is a strong learning baseline, especially on large-scale hard datasets. Therefore, we use SBP as our main learning baseline in this paper.  
	
	\subsection{Comparative evaluation of HEM++}\label{exp:hem++}
	Compared with HEM, HEM++ has two main advantages. First, the formulation of HEM++ is more widely applicable. To demonstrate its superiority, we compare HEM++ with the baselines under the multiple rounds setting in Section \ref{exp:hem++_multiple}. Second, the model and training method of HEM++ are more powerful. To demonstrate the effectiveness of its model and training method, we compare HEM++ with HEM under the one round setting in Section \ref{exp:hem++_one}. 
	
	\subsubsection{Evaluation under the multiple rounds setting}\label{exp:hem++_multiple}
	In this part, we compare HEM++ with the baselines rather than HEM under the multiple rounds setting, as it is challenging for HEM to compute multi-step policy gradient. Note that we reimplement all the baselines to adapt to the multiple rounds setting. Please refer to Appendix \ref{appendix_imple_baselines_multi_rounds} for implementation details of the baselines under the multiple rounds setting. Specifically, we set the cut separation rounds as $T=10$, and compare HEM++ with the baselines on Maximum Independent Set (Easy), CORLAT (Medium), and MIPLIB mixed supportcase (Hard). The results in Table \ref{exp:evaluation_multiple_rounds} show that HEM significantly outperforms the baselines in terms of the Time and/or PD integral. This demonstrates that HEM++ is well applicable to the multiple rounds setting, suggesting the superiority of HEM++ over HEM.
	
	
	\subsubsection{Evaluation under the one round setting}\label{exp:hem++_one}
	To evaluate whether our proposed HEM++ further improves the performance of HEM, we compare HEM++ with HEM under the one round setting. Specifically, we compare HEM++ with HEM on three large-scale datasets, i.e., Anonymous, MIPLIB mixed neos, and MIPLIB mixed supportcase. Please refer to Appendix \ref{appendix_results_hem++} for results on more datasets. 
	The results in Tabel \ref{eva:HEM++} show that HEM++ outperforms HEM in terms of the Time and/or PD integral, demonstrating the effectiveness of HEM++. Specifically, the results suggest that the Set2Seq model in HEM++ well improves the representation of input cuts, and the HPPO in HEM++ further improves the sample-efficiency. 
	
	
	
	\subsection{Ablation study}\label{exp:ablation_study}
	In this subsection, we present ablation studies on Maximum Independent Set (MIS), CORLAT, and MIPLIB mixed neos, which are representative datasets from the easy, medium, and hard datasets, respectively. We provide more results on the other datasets in Appendix \ref{appendix_results_ablation_study}. 
	
	\begin{table*}[t]
		\centering
		\caption{Comparison between HEM and HEM without the higher-level model. The best performance is marked in bold. }
		\label{ablation_each_component}
		\resizebox{0.96\textwidth}{!}{
			\begin{tabular}{@{}cccccccccc@{}}
				\toprule
				\toprule
				& \multicolumn{3}{c}{Easy: Max Independent Set ($n=500,\,\,m=1953$)} & \multicolumn{3}{c}{Medium: Corlat ($n=466,\,\,m=486$)} & \multicolumn{3}{c}{Hard: MIPLIB mixed neos ($n=6958,\,\,m=5660$)} \\ \midrule
				\multirow{2}{*}{Method} & \multirow{2}{*}{Time(s) $\downarrow$} & \multirow{2}{*}{\begin{tabular}[c]{@{}c@{}}Improvement $\uparrow$\\      (Time, \%)\end{tabular}} & \multirow{2}{*}{PD integral $\downarrow$} & \multirow{2}{*}{Time(s) $\downarrow$} & \multirow{2}{*}{PD integral $\downarrow$} & \multirow{2}{*}{\begin{tabular}[c]{@{}c@{}}Improvement $\uparrow$\\      (PD integral, \%)\end{tabular}} & \multirow{2}{*}{Time(s) $\downarrow$} & \multirow{2}{*}{PD integral $\downarrow$} & \multirow{2}{*}{\begin{tabular}[c]{@{}c@{}}Improvement $\uparrow$\\      (PD integral, \%)\end{tabular}} \\
				&  &  &  &  &  &  &  &  &  \\ \cmidrule(r){1-4} \cmidrule(lr){5-7} \cmidrule(l){8-10}
				NoCuts & 8.78 (6.66) & NA & 71.31 (51.74) & 103.30 (128.14) & 2818.40 (5908.31) & NA & 253.65 (80.29) & 14652.29 (12523.37) & NA \\
				Default & 3.88 (5.04) & 55.81 & 29.44 (35.27) & 75.20 (120.30) & 2412.09 (5892.88) & 14.42 & 256.58 (76.05) & 14444.05 (12347.09) & 1.42 \\
				SBP & 2.43 (5.55) & 72.32 & 21.99 (40.86) & 70.41 (122.17) & 2023.87 (5085.96) & 28.19 & 256.48 (78.59) & 13531.00 (12898.22) & 7.65 \\ \midrule
				HEM w/o H & 1.88 (4.20) & 78.59 & 16.70 (28.15) & 63.14 (115.26) & 1939.08 (5484.83) & 31.20 & 249.21 (88.09) & 13614.29 (12914.76) & 7.08 \\
				HEM (Ours) & \textbf{1.76 (3.69)} & \textbf{79.95} & \textbf{16.01 (26.21)} & \textbf{58.31 (110.51)} & \textbf{1079.99 (2653.14)} & \textbf{61.68} & \textbf{248.66 (89.46)} & \textbf{8678.76 (12337.00)} & \textbf{40.77} \\ \bottomrule
			\end{tabular}
		}
	\end{table*}
	
	\begin{table*}[t]
		\centering
		\caption{Comparison between HEM, HEM-ratio, and HEM-ratio-order. The best performance is marked in bold. }
		\label{ablation_each_factor}
		\resizebox{0.96\textwidth}{!}{
			\begin{tabular}{@{}cccccccccc@{}}
				\toprule
				\toprule
				& \multicolumn{3}{c}{Easy: Maximum Independent Set ($n=500,\,\,m=1953$)} & \multicolumn{3}{c}{Medium: Corlat ($n=466,\,\,m=486$)} & \multicolumn{3}{c}{Hard: MIPLIB mixed neos ($n=6958,\,\,m=5660$)} \\ \midrule
				\multirow{2}{*}{Method} & \multirow{2}{*}{Time(s) $\downarrow$} & \multirow{2}{*}{\begin{tabular}[c]{@{}c@{}}Improvement $\uparrow$\\      (Time, \%)\end{tabular}} & \multirow{2}{*}{PD integral $\downarrow$} & \multirow{2}{*}{Time(s) $\downarrow$} & \multirow{2}{*}{PD integral $\downarrow$} & \multirow{2}{*}{\begin{tabular}[c]{@{}c@{}}Improvement $\uparrow$\\      (PD integral, \%)\end{tabular}} & \multirow{2}{*}{Time(s) $\downarrow$} & \multirow{2}{*}{PD integral $\downarrow$} & \multirow{2}{*}{\begin{tabular}[c]{@{}c@{}}Improvement $\uparrow$\\      (PD integral, \%)\end{tabular}} \\
				&  &  &  &  &  &  &  &  &  \\ \cmidrule(r){1-4} \cmidrule(lr){5-7} \cmidrule(l){8-10}
				NoCuts & 8.78 (6.66) & NA & 71.31 (51.74) & 103.30 (128.14) & 2818.40 (5908.31) & NA & 253.65 (80.29) & 14652.29 (12523.37) & NA \\
				Default & 3.88 (5.04) & 55.81 & 29.44 (35.27) & 75.20 (120.30) & 2412.09 (5892.88) & 14.42 & 256.58 (76.05) & 14444.05 (12347.09) & 1.42 \\
				SBP & 2.43 (5.55) & 72.32 & 21.99 (40.86) & 70.41 (122.17) & 2023.87 (5085.96) & 28.19 & 256.48 (78.59) & 13531.00 (12898.22) & 7.65 \\\midrule
				HEM-ratio-order & 2.30 (5.18) & 73.80 & 21.19 (38.52) & 70.94 (122.93) & 1416.66 (3380.10) & 49.74 & 245.99 (93.67) & 14026.75 (12683.60) & 4.27 \\
				HEM-ratio & 2.26 (5.06) & 74.26 & 20.82 (37.81) & 67.27 (117.01) & 1251.60 (2869.87) & 55.59 & \textbf{244.87 (95.56)} & 13659.93 (12900.59) & 6.77 \\
				HEM (Ours) & \textbf{1.76 (3.69)} & \textbf{79.95} & \textbf{16.01 (26.21)} & \textbf{58.31 (110.51)} & \textbf{1079.99 (2653.14)} & \textbf{61.68} & 248.66 (89.46) & \textbf{8678.76 (12337.00)} & \textbf{40.77} \\ \bottomrule
			\end{tabular}
		}
	\end{table*}
	
	\begin{table*}[t]
		\caption{Comparison between HEM++, HEM++ without HPPO, and HEM++ without Set2Seq.}
		\label{exp:hem++_ablation_study}
		\resizebox{\textwidth}{!}{
			\begin{tabular}{@{}cccccccccc@{}}
				\toprule
				\toprule
				& \multicolumn{3}{c}{Hard: Anonymous ($n=37881,\,\,m=49603$)} & \multicolumn{3}{c}{Hard: MIPLIB mixed neos ($n=6958,\,\,m=5660$)} & \multicolumn{3}{c}{Hard: MIPLIB mixed supportcase ($n=19766,\,\,m=19910$)} \\ \midrule
				\multirow{2}{*}{Method} & \multirow{2}{*}{Time (s) $\downarrow$} & \multirow{2}{*}{PD integral $\downarrow$} & \multirow{2}{*}{\begin{tabular}[c]{@{}c@{}}Improvement $\uparrow$\\      (\%, PD integral)\end{tabular}} & \multirow{2}{*}{Time (s) $\downarrow$} & \multirow{2}{*}{PD integral $\downarrow$} & \multirow{2}{*}{\begin{tabular}[c]{@{}c@{}}Improvement $\uparrow$\\      (\%, PD integral)\end{tabular}} & \multirow{2}{*}{Time (s) $\downarrow$} & \multirow{2}{*}{PD integral $\downarrow$} & \multirow{2}{*}{\begin{tabular}[c]{@{}c@{}}Improvement $\uparrow$\\      (\%, PD integral)\end{tabular}} \\
				&  &  &  &  &  &  &  &  &  \\ \midrule
				NoCuts & 246.22 (94.90) & 18297.30 (9769.42) & NA & 253.65 (80.29) & 14652.29 (12523.37) & NA & 170.00 (131.60) & 9927.96 (11334.07) & NA \\
				Default & 244.02 (97.72) & 17407.01 (9736.19) & 4.90 & 256.58 (76.05) & 14444.05 (12347.09) & 1.42 & 164.61 (135.82) & 9672.34 (10668.24) & 2.57 \\ 
				SBP & 245.71 (92.46) & 18188.63 (9651.85) & 0.59 & 256.48 (78.59) & 13531.00 (12898.22) & 7.65 & 165.61 (135.25) & 7408.65 (7903.47) & 25.37 \\ \midrule
				HEM & 241.68 (97.23) & 16077.15 (9108.21) & 12.10 & \textbf{248.66 (89.46)} & 8678.76 (12337.00) & 40.77 & 162.96 (138.21) & 6874.80 (6729.97) & 30.75 \\
				HEM++ w/o HPPO & \textbf{239.50 (104.24)} & 15848.32 (9870.58) & 13.38 & 254.38 (79.07) & 12574.79 (12548.1) & 14.18 & 154.66 (135.77) & 8440.19 (9655.38) & 14.99 \\
				HEM++ w/o Set2Seq & 250.19 (87.43) & 17144.41 (9084.42) & 6.30 & 258.80 (71.41) & 8627.41 (12353.81) & 41.12 & 169.43 (132.38) & 6737.99 (6518.49) & 32.13 \\
				HEM++ & 243.23 (96.62) & \textbf{15444.70 (9260.60)} & \textbf{15.59} & 252.07 (83.06) & \textbf{8557.46 (12395.07)} & \textbf{41.60} & \textbf{153.83 (131.55)} & \textbf{6586.04 (6313.25)} & \textbf{33.66} \\ \bottomrule
			\end{tabular}
		}
	\end{table*}
	
	\begin{table}[t]
		\caption{Evaluate the generalization ability of HEM on Set Covering and Maximum Independent Set.}
		\label{generalization_setcovering_mis}
		\resizebox{0.48\textwidth}{!}{
			\begin{tabular}{@{}ccccccc@{}}
				\toprule
				\toprule
				& \multicolumn{3}{c}{Set Covering ($n=1000,\,m=1000,\, 2\times$)} & \multicolumn{3}{c}{Set Covering ($n=1000,\,m=2000,\, 4\times$)} \\ \midrule
				\multirow{2}{*}{Method} & \multirow{2}{*}{Time(s) $\downarrow$} & \multirow{2}{*}{\begin{tabular}[c]{@{}c@{}}Improvement $\uparrow$\\      (time, \%)\end{tabular}} & \multirow{2}{*}{PD integral $\downarrow$} & \multirow{2}{*}{Time(s) $\downarrow$} & \multirow{2}{*}{\begin{tabular}[c]{@{}c@{}}Improvement $\uparrow$ \\      (time, \%)\end{tabular}} & \multirow{2}{*}{PD integral $\downarrow$} \\
				&  &  &  &  &  &  \\ \cmidrule(r){1-4} \cmidrule(l){5-7}
				NoCuts & 82.69 (78.27) & NA & 609.43 (524.92) & 284.44 (48.70) & NA & 3215.34 (1019.47) \\
				Default & 61.01 (78.12) & 26.22 & 494.63 (545.76) & 149.69 ( 141.92) & 47.37 & 1776.22 (1651.15) \\
				Random & 64.44 (73.98) & 22.07 & 520.84 (489.52) & 208.12 (131.52) & 26.53 & 2528.36 (1678.66) \\
				NV & 92.05 (80.11) & -11.32 & 725.53 (541.68) & 286.10 (45.47) & -0.58 & 3422.46 (1024.19) \\
				Eff & 92.32 (79.33) & -11.64 & 733.72 (538.60) & 286.20 (45.04) & -0.62 & 3437.06 (1043.44) \\\midrule
				SBP & 3.52 (1.36) & 95.74 & 92.89 (25.83) & 7.62 (6.46) & 97.32 & 256.79 (145.92) \\
				HEM (Ours) & \textbf{3.33 (0.47)} & \textbf{95.97} & \textbf{89.24 (14.26)} & \textbf{7.40 (5.03)} & \textbf{97.40} & \textbf{250.83 (131.43)} \\ \bottomrule
			\end{tabular}
		}
		\newline
		\vspace{2mm}
		\newline
		\resizebox{0.48\textwidth}{!}{
			\begin{tabular}{@{}ccccccc@{}}
				\toprule
				\toprule
				& \multicolumn{3}{c}{Maximum Independent Set ($n=1000,\,m=3946,\,4\times$)} & \multicolumn{3}{c}{Maximum Independent Set ($n=5940,\,m=1500,\,9\times$)} \\ \midrule
				\multirow{2}{*}{Method} & \multirow{2}{*}{Time(s) $\downarrow$} & \multirow{2}{*}{\begin{tabular}[c]{@{}c@{}}Improvement $\uparrow$\\      (time, \%)\end{tabular}} & \multirow{2}{*}{PD integral $\downarrow$} & \multirow{2}{*}{Time(s) $\downarrow$} & \multirow{2}{*}{\begin{tabular}[c]{@{}c@{}}Improvement $\uparrow$\\      (time, \%)\end{tabular}} & \multirow{2}{*}{PD integral $\downarrow$} \\
				&  &  &  &  &  &  \\\cmidrule(r){1-4} \cmidrule(l){5-7}
				NoCuts & 170.06 (100.61) & NA & 874.45 (522.29) & 300.00 (0) & NA & 2019.93 (353.27) \\
				Default & 42.40 (76.00) & 48.72 & 198.61 (331.20) & 111.18 (144.13) & 60.91 & 616.46 (798.94) \\
				Random & 118.25 (109.05) & -43.00 & 574.33 (516.11) & 245.13 (115.80) & 13.82 & 1562.20 (793.09) \\
				NV & 160.30 (101.41) & -93.86 & 784.98 (493.24) & 299.97 (0.49) & -5.46 & 1922.52 (349.67) \\
				Eff & 158.75 (100.40) & -91.98 & 779.63 (493.05) & 299.45 (3.77) & -5.28 & 1921.61 (361.26) \\\midrule
				SBP & 50.55 (89.14) & 38.87 & 253.81 (426.94) & 108.42 (143.68) & 61.88 & 680.41 (903.88) \\
				HEM (Ours) & \textbf{35.34 (67.91)} & \textbf{57.26} & \textbf{160.56 (282.03)} & \textbf{108.02 (143.02)} & \textbf{62.02} & \textbf{570.48 (760.65)} \\ \bottomrule
			\end{tabular}
		}
	\end{table}
	
	\subsubsection{Contribution of each component in HEM}
	We perform ablation studies to understand the contribution of each component in HEM. 
	We report the performance of HEM and HEM without the higher-level model (HEM w/o H) in Table \ref{ablation_each_component}. HEM w/o H is essentially a pointer network. Note that it can still implicitly predicts the number of cuts that should be selected by predicting an end token as used in language tasks \cite{seq2seq}. Please see Appendix \ref{appendix_imple_rlk} for details.  \textbf{First},
	the results in Table \ref{ablation_each_component} show that HEM w/o H outperforms all the baselines on MIS and CORLAT, demonstrating the advantages of the lower-level model. Although HEM w/o H outperforms Default on MIPLIB mixed neos, HEM w/o H performs on par with SBP. A possible reason is that it is difficult for HEM w/o H to explore the action space efficiently, and thus HEM w/o H tends to be trapped to local optimum. 
	\textbf{Second}, the results in Table \ref{ablation_each_component} show that HEM significantly outperforms HEM w/o H and the baselines on the three datasets. The results demonstrate that the higher-level model is important for efficient exploration in complex tasks, thus significantly improving the solving efficiency.
	

	\subsubsection{The importance of tackling \textbf{P1-P3} in cut selection}
	We perform ablation studies to understand the 
	importance of tackling \textbf{(P1)-(P3)} in cut selection.
	(1) \textbf{HEM.} HEM tackles \textbf{(P1)-(P3)} in cut selection simultaneously. (2) \textbf{HEM-ratio.} In order not to learn how many cuts should be selected, we remove the higher-level model of HEM and \textit{force the lower-level model to select a fixed ratio of cuts}. We denote it by HEM-ratio. Note that HEM-ratio is different from HEM w/o H (see Appendix \ref{appendix_imple_rlk}). HEM-ratio tackles \textbf{(P1)} and \textbf{(P3)} in cut selection. (3) \textbf{HEM-ratio-order.} To further mute the effect of the order of selected cuts, we reorder the selected cuts given by HEM-ratio with the original index of the generated cuts, which we denote by HEM-ratio-order. HEM-ratio-order mainly tackles \textbf{(P1)} in cut selection.
	The results in Table \ref{ablation_each_factor} suggest the following. HEM-ratio-order significantly outperforms Default and NoCuts, demonstrating that tackling \textbf{(P1)} by data-driven methods is crucial. HEM significantly outperforms HEM-ratio in terms of the PD integral, demonstrating the significance of tackling \textbf{(P2)}. HEM-ratio outperforms HEM-ratio-order in terms of the Time and the PD integral, which demonstrates the importance of tackling \textbf{(P3)}. Moreover, HEM-ratio and HEM-ratio-order perform better than SBP on MIS and CORLAT, demonstrating the advantages of using the sequence model to learn cut selection over SBP. HEM-ratio and HEM-ratio-order perform on par with SBP on MIPLIB mixed neos. We provide possible reasons in Appendix \ref{appendix_analysis_hem_ratio}.

	\subsubsection{Contribution of each component in HEM++}
	To further analyze the contribution of the Set2Seq model and HPPO training method in HEM++, we report the performance of HEM++, HEM++ without HPPO (HEM++ w/o HPPO), and HEM++ without Set2Seq (HEM++ w/o Set2Seq) under the one round setting. Note that HEM++ w/o HPPO  is equivalent to HEM with the Set2Seq model and HEM++ w/o Set2Seq  is equivalent to HEM with the HPPO training method. The results in Table \ref{exp:hem++_ablation_study} show that HEM++ significantly outperforms HEM++ w/o Set2Seq (up to 9.29\% improvement), demonstrating the contribution of the Set2Seq formulation over the Seq2Seq model. An interesting observation is that HEM++ w/o HPPO (i.e., HEM with Set2Seq) struggles to consistently outperform HEM on three hard datasets. A possible reason is that it is more difficult to train a multi-head attention encoder than a long short-term memory (LSTM), and thus it is important to train the model with a sample-efficient training method. Moreover, Table \ref{exp:hem++_ablation_study} shows that HEM++ significantly outperforms HEM++ w/o HPPO as well, demonstrating the importance of the sample-efficient training method. Moreover, we provide a detailed discussion on two major advantages of the Set2Seq model over the Seq2Seq model as follows. 
	\textbf{First}, the Set2Seq model can improve the representation of input candidate cuts by learning their \textit{order-independent} embeddings. Note that given the same instance and candidate cuts in different order, the corresponding optimal ordered subsets should be unchanged. Thus, it is preferable for cut selection policies to be agnostic to the order of input cuts. \textbf{Second}, we use an additional multi-head attention encoder, which is a core component of the Transformer \cite{transformer}, to instantiate the Set2Seq model. Compared to the LSTM encoder, it has been shown that the multi-head attention encoder offers significantly enhanced power and capability \cite{transformer}.

	\subsection{Generalization}\label{exp:generalization}
	We evaluate the ability of HEM to generalize across different sizes of MILPs. Let $n\times m$ denote the size of MILP instances. Following \cite{nips19_gcnn, rl_branch}, we test the generalization ability of HEM on Set Covering and Maximum Independent Set (MIS), as we can artificially generate instances with arbitrary sizes for synthetic MILP problems. Specifically, we test HEM on significantly larger instances than those seen during training on Set Covering and MIS.  
	The results in Table \ref{generalization_setcovering_mis} show that HEM significantly outperforms the baselines in terms of the Time and PD integral, demonstrating the superiority of HEM in terms of the generalization ability.
	Interestingly, SBP well generalizes to large instances as well, which demonstrates that SBP is a strong baseline.
	
	
	
	\begin{figure}[t]
		\centering
		\includegraphics[width=0.49\columnwidth]{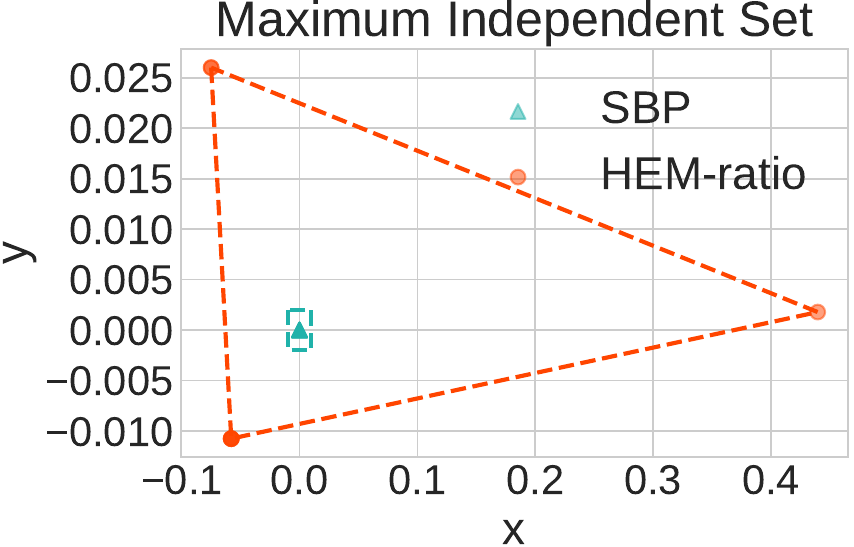}
		\includegraphics[width=0.49\columnwidth]{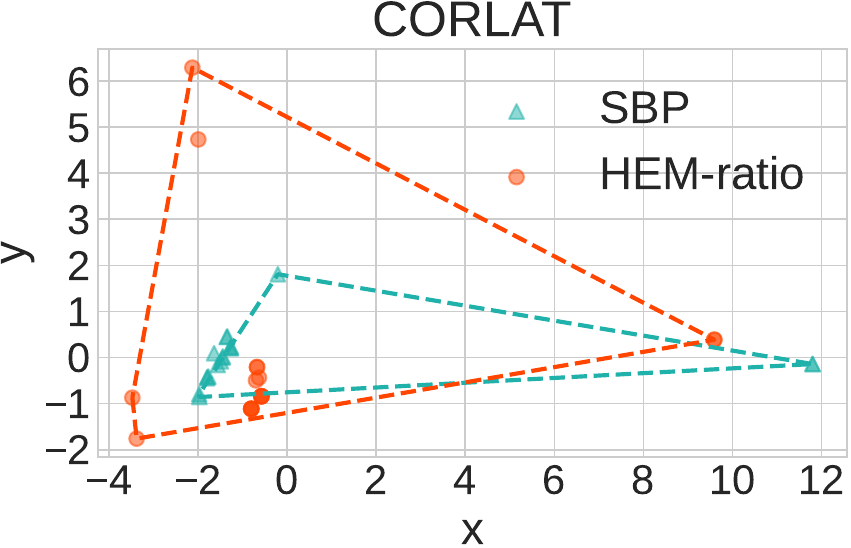}
		\caption{We perform principal component analysis on the cuts selected by HEM-ratio and SBP. Colored points illustrate the reduced cut features. The area covered by the dashed lines represents the diversity of selected cuts. The results show that HEM-ratio selects much more diverse cuts than SBP.}
		\label{fig:visualization}
	\end{figure}
	
	\begin{figure*}[t]
		\centering
		\includegraphics[width=0.27\textwidth,height=0.2\textwidth]{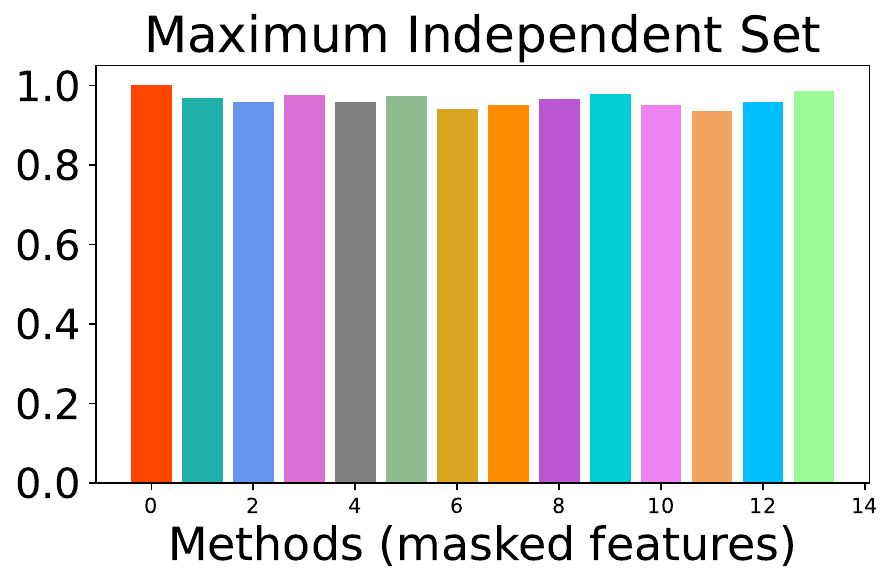}
		\hspace{-1.5mm}
		\includegraphics[width=0.27\textwidth,height=0.2\textwidth]{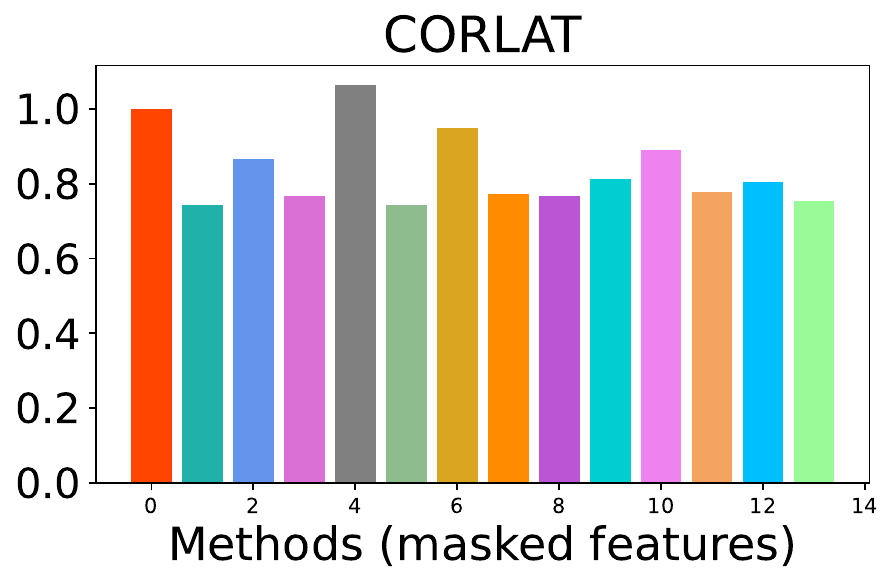}
		\hspace{-1.5mm}
		\includegraphics[width=0.45\textwidth,height=0.2\textwidth]{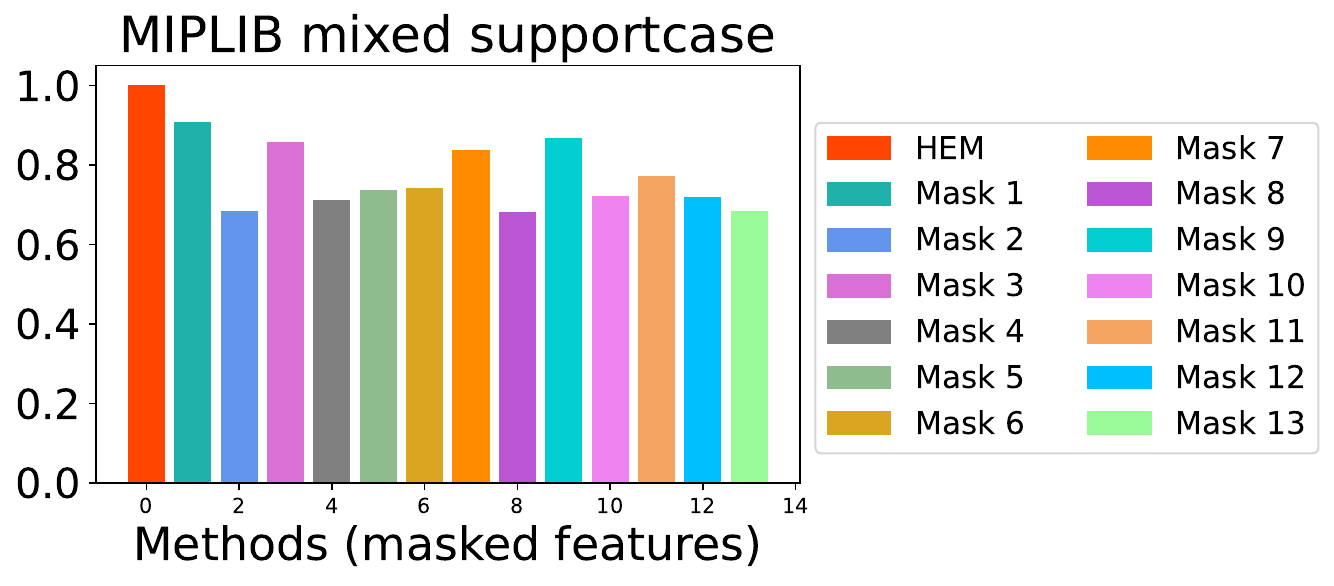}
		\caption{Evaluate the importance of the designed cut features. For the thirteen-dimensional features of a cut, we respectively set the $i$th dimension to zero, namely Mask $i$ ($i=1,2,\dots,13$). The results show the normalized performance of Mask $i$, defined by $\frac{M(\text{HEM})}{M(\text{Mask}\,\, i)}$. Here $M(\text{HEM})$ and $M(\text{Mask}\,\, i)$ denote the performance of HEM and Mask $i$, respectively.} 
		\label{fig:mask_visualization}
	\end{figure*}
	
	\begin{table*}[t]
		\centering
		\caption{Evaluation on Huawei production planning and order matching problems.}
		\label{evaluation_huawei_pp_om}
		\resizebox{0.96\textwidth}{!}{
			\begin{tabular}{@{}ccccccccc@{}}
				\toprule
				\toprule
				& \multicolumn{4}{c}{Production planning ($n=3582.25,\,\,m=5040.42$)} & \multicolumn{4}{c}{Order matching ($n=67839.68,\,\,m=31364.84$)} \\ \midrule
				\multirow{2}{*}{Method} & \multirow{2}{*}{Time (s) $\downarrow$} & \multirow{2}{*}{\begin{tabular}[c]{@{}c@{}}Improvement $\uparrow$ \\      (Time, \%)\end{tabular}} & \multirow{2}{*}{PD integral $\downarrow$} & \multirow{2}{*}{\begin{tabular}[c]{@{}c@{}}Improvement $\uparrow$\\      (PD integral, \%)\end{tabular}} & \multirow{2}{*}{Time (s) $\downarrow$} & \multirow{2}{*}{\begin{tabular}[c]{@{}c@{}}Improvement $\uparrow$ \\      (Time, \%)\end{tabular}} & \multirow{2}{*}{PD integral $\downarrow$} & \multirow{2}{*}{\begin{tabular}[c]{@{}c@{}}Improvement$\uparrow$\\      (PD integral, \%)\end{tabular}} \\
				&  &  &  &  &  &  &  &  \\ \cmidrule(r){1-5} \cmidrule(l){6-9}
				NoCuts & 278.79 (231.02) & NA & 17866.01 (21309.85) & NA & 248.42 (287.29) & NA & 403.41 (345.51) & NA \\
				Default & 296.12 (246.25) & -6.22 & 17703.39 (21330.40) & 0.91 & 129.34 (224.24) & 47.93 & 395.80 (341.23) & 1.89 \\
				Random & 280.18 (237.09) & -0.50 & 18120.21 (21660.01) & -1.42 & 95.76 (202.23) & 61.45 & 406.73 (348.30) & -0.82 \\
				NV & 259.48 (227.81) & 6.93 & 17295.18 (21860.07) & 3.20 & 245.61 (282.04) & 1.13 & 406.69 (347.72) & -0.81 \\
				Eff & 263.60 (229.24) & 5.45 & 16636.52 (21322.89) & 6.88 & 243.30 (276.65) & 2.06 & 417.71 (360.13) & -3.54 \\ \midrule
				SBP & 276.61 (235.84) & 0.78 & 16952.85 (21386.07) & 5.11 & 44.14 (148.60) & 82.23 & 379.23 (326.86) & 5.99 \\
				HEM (Ours) & \textbf{241.77 (229.97)} & \textbf{13.28} & \textbf{15751.08 (20683.53)} & \textbf{11.84} & \textbf{43.88 (148.66)} & \textbf{82.34} & \textbf{368.51 (309.02)} & \textbf{8.65} \\ \bottomrule
			\end{tabular}
		}
	\end{table*}
	
	\subsection{Visualization and explainability analysis}\label{exp:visu_and_explain}
	\subsubsection{Diversity of selected cuts}\label{sec:visu_diversity}
	We visualize the diversity of selected cuts, an important metric for evaluating whether the selected cuts complement each other nicely \cite{theoretical_cuts}. We visualize the cuts selected by HEM-ratio and SBP on a randomly sampled instance from Maximum Independent Set and CORLAT, respectively. We evaluate HEM-ratio rather than HEM, as HEM-ratio selects the same number of cuts as SBP. Furthermore, we perform principal component analysis on the selected cuts to reduce the cut features to two-dimensional space. Colored points illustrate the reduced cut features. To visualize the diversity of selected cuts, we use dashed lines to connect the points with the smallest and largest x,y coordinates. That is, \textit{the area covered by the dashed lines represents the diversity}.
	Fig. \ref{fig:visualization} shows that SBP tends to select many similar cuts that are possibly redundant, especially on Maximum Independent Set. In contrast, HEM-ratio selects much more diverse cuts that can well complement each other. Please refer to Appendix \ref{appendix_results_visualization} for results on more datasets.

	\subsubsection{Analysis on structure of selected cuts}\label{sec:visu_structure}
	
	To provide further insight into learned policies by HEM, we analyze categories of the selected cuts by HEM on problems with knapsack constraints (specific structured problems). The results (see Appendix \ref{appendix_results_visualization_structure}) show that our learned model selects 92.03\%  and 99.5\% cover inequalities (cover cuts) on Multiple Knapsack and MIK, respectively. It has been shown that a prominent category of cuts for knapsack problems is cover cuts \cite{lift_cover,flow_cover}. Therefore, the results suggest that our learned policies can well capture the underlying
	structure of specific structured problems, such as knapsack problems. 
	
	\subsubsection{Feature analysis}
	To guide heuristics design for the cut selection, we conduct the following experiments to analyze the importance of features on the cut selection. Specifically, we evaluate the importance of each dimension feature by masking it. For the thirteen-dimensional features of a cut, we respectively set the $i$-th dimension to zero, namely Mask $i$ ($i=1,2,\dots,13$). Please refer to Appendix \ref{appendix_cut_features} for details of the thirteen-dimensional features.  
	Fig \ref{fig:mask_visualization} shows the \textit{normalized performance} of Mask $i$ on Maximum Independent Set (MIS), CORLAT, and MIPLIB mixed supportcase. We define the \textit{normalized performance} of Mask $i$ by $\frac{M(\text{HEM})}{M(\text{Mask}\,\, i)}$, where $M(\text{HEM})$ and $M(\text{Mask}\,\, i)$ denote the performance of HEM and Mask $i$, respectively. That is, the lower the normalized performance, the higher the corresponding Time and/or PD integral, indicating that the corresponding feature has more significant impact on cut selection. 
	The results in Fig \ref{fig:mask_visualization} suggest the following conclusions. 
	First, most of the features have a significant impact on the performance, showing the effectiveness of our designed thirteen-dimensional features.
	Second, the 1st to 5th dimension features are widely used for heuristic design, while the results show that the 6th to 13th dimension features have a more significant impact on the cut selection. Therefore, incorporating the 6th to 13th dimension features into the heuristic design is important to improve the existing heuristics.
	Third, masking the 4th dimension feature improves the performance on CORLAT, suggesting that dropping the 4th dimension feature is beneficial when designing heuristics for the CORLAT dataset.

	\begin{table*}[t]
		\caption{Comparison between Default+ and Default. The best performance is marked in bold.}
		\label{default_hem_results}
		\centering
		\resizebox{0.96\textwidth}{!}{
			\begin{tabular}{@{}cccccccccc@{}}
				\toprule
				\toprule
				& \multicolumn{3}{c}{Anonymous ($n=37881,\,\,m=49603$)} & \multicolumn{3}{c}{Production planning ($n=3582.25,\,\,m=5040.42$)} & \multicolumn{3}{c}{Order matching ($n=67839.68,\,\,m=31364.84$)} \\ \midrule
				\multirow{2}{*}{Method} & \multirow{2}{*}{Time (s) $\downarrow$} & \multirow{2}{*}{PD integral $\downarrow$} & \multirow{2}{*}{\begin{tabular}[c]{@{}c@{}}Improvement $\uparrow$ \\      (\%, PD integral)\end{tabular}} & \multirow{2}{*}{Time (s) $\downarrow$} & \multirow{2}{*}{PD integral $\downarrow$} & \multirow{2}{*}{\begin{tabular}[c]{@{}c@{}}Improvement $\uparrow$\\      (\%, PD integral)\end{tabular}} & \multirow{2}{*}{Time (s) $\downarrow$} & \multirow{2}{*}{PD integral $\downarrow$} & \multirow{2}{*}{\begin{tabular}[c]{@{}c@{}}Improvement $\uparrow$ \\      (\%, PD integral)\end{tabular}} \\
				&  &  &  &  &  &  &  &  &  \\ \cmidrule(r){1-4} \cmidrule(l){5-7} \cmidrule(l){8-10}
				NoCuts & 246.22 (94.90) & 18297.31 (9769.42) & NA & \textbf{278.79 (231.02)} & 17866.01 (21309.86) & NA & 248.42 (287.29) & 403.41 (345.51) & NA \\
				Default & 244.02 (97.72) & 17252.12 (9736.19) & 5.71 & 296.12 (246.25) & 17703.39 (21330.4) & 0.91 & 129.34 (224.24) & 395.80 (341.23) & 1.89 \\ 
				Default+ & \textbf{242.59 (100.09)} & \textbf{16439.48 (9586.23)} & \textbf{10.15} & 280.40 (241.69) & \textbf{17098.54 (21347.24)} & \textbf{4.30} & \textbf{129.04 (224.23)} & \textbf{380.27 (333.05)} & \textbf{5.74} \\ \bottomrule
			\end{tabular}
		}
	\end{table*}
	
	\subsection{Deployment in real-world challenging problems}\label{exp:deployment}
	
	\subsubsection{Deploying learned policies}
	
	To further evaluate the effectiveness of our proposed HEM, we deploy HEM to large-scale real-world production planning and order matching problems at Huawei, which is one of the largest global commercial technology enterprises.  
	Please refer to Appendix \ref{appendix_datasets_sc} for details of the problems.
	The results in Table \ref{evaluation_huawei_pp_om} show that HEM significantly outperforms all the baselines in terms of the Time and PD integral (up to 82.34\% improvement). 
	The results demonstrate the strong ability to enhance modern MILP solvers with HEM in real-world applications. Interestingly, Default performs poorer than NoCuts in production planning problems, which implies that an improper cut selection policy could significantly degrade the performance of MILP solvers. 
	
	\subsubsection{Extracting order rules from our learned policies}\label{exp:order_rules}
	
	To extract order rules from HEM as proposed in Section \ref{method:order_heuristics}, we obtain the cutting plane categories with the highest number of times that is located at the first, second, and third position, respectively. We evaluate Default+ on three hard datasets as shown in Table \ref{default_hem_results}. Please refer to Appendix \ref{appendix_order_rules_results} for the statistical results and more evaluation results.  
	The results show that Default+ outperforms Default in terms of the Time and PD integral (up to 4.44$\%$ improvement), which demonstrates the effectiveness of our extracted order rules. Therefore, Default+ is effective for simple and efficient deployment of our method into modern MILP solvers.

	
	

	\section{Conclusion}
	In this paper, we observe from extensive empirical results that the order of selected cuts has a significant impact on the efficiency of solving MILPs. We propose a novel \textbf{h}ierarchical s\textbf{e}quence/s\textbf{e}t \textbf{m}odel (HEM) to learn cut selection policies via reinforcement learning. Specifically, HEM is a two-level model: (1) a higher-level model to learn the number of cuts that should be selected, (2) and a lower-level model---that formulates the cut selection task as a sequence/set to sequence learning problem---to learn policies selecting an ordered subset with the cardinality determined by the higher-level model. Experiments show that HEM significantly improves the efficiency of solving MILPs compared to competitive baselines on both synthetic and large-scale real-world MILPs. We believe that our proposed approach brings new insights into learning cut selection. 
	
	
	%

	\ifCLASSOPTIONcompsoc
	\section*{Acknowledgments}
	\else
	\section*{Acknowledgment}
	\fi
	
	The authors would like to thank the associate editor and all the anonymous reviewers for their insightful comments. This work was supported in part by National Key R\&D Program of China under contract 2022ZD0119801, National Nature Science Foundations of China grants U19B2026, U19B2044, 61836011, 62021001, and 61836006.
	
	\ifCLASSOPTIONcaptionsoff
	\newpage
	\fi

	
	
	\bibliographystyle{IEEEtran}
	\bibliography{IEEEtran}

	\begin{IEEEbiography}[{\includegraphics[width=1in,height=1.25in,clip,keepaspectratio]{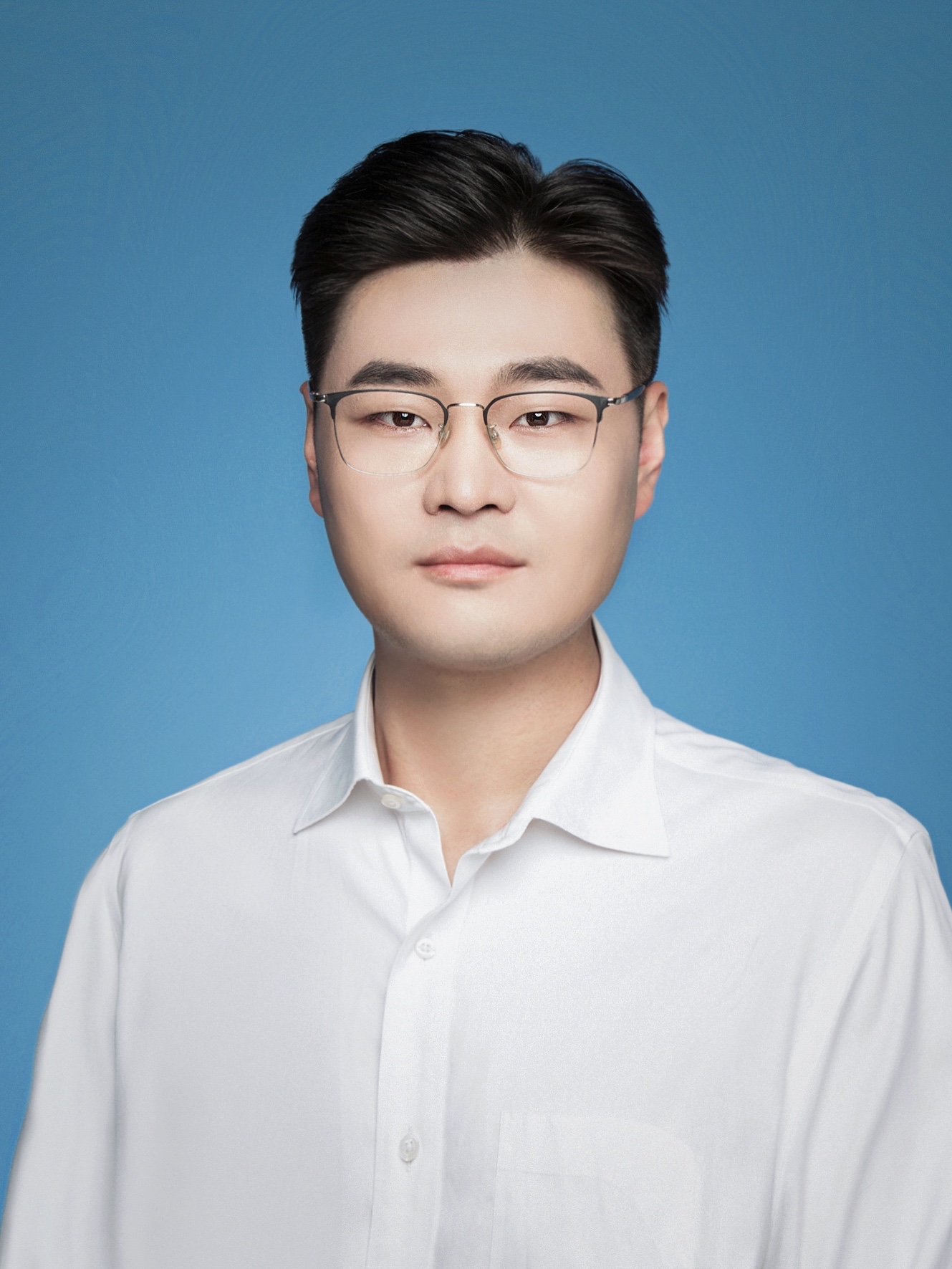}}]{Jie Wang}
		received the B.Sc. degree in electronic information science and technology from University of Science and Technology of China, Hefei, China, in 2005, and the Ph.D. degree in computational science from the Florida State University, Tallahassee, FL, in 2011. He is currently a professor in the Department of Electronic Engineering and Information Science at University of Science and Technology of China, Hefei, China. His research interests include reinforcement learning, knowledge graph, large-scale optimization, deep learning, etc.  He is a senior member of IEEE.
	\end{IEEEbiography}
	\begin{IEEEbiography}[{\includegraphics[width=1in,height=1.25in,clip,keepaspectratio]{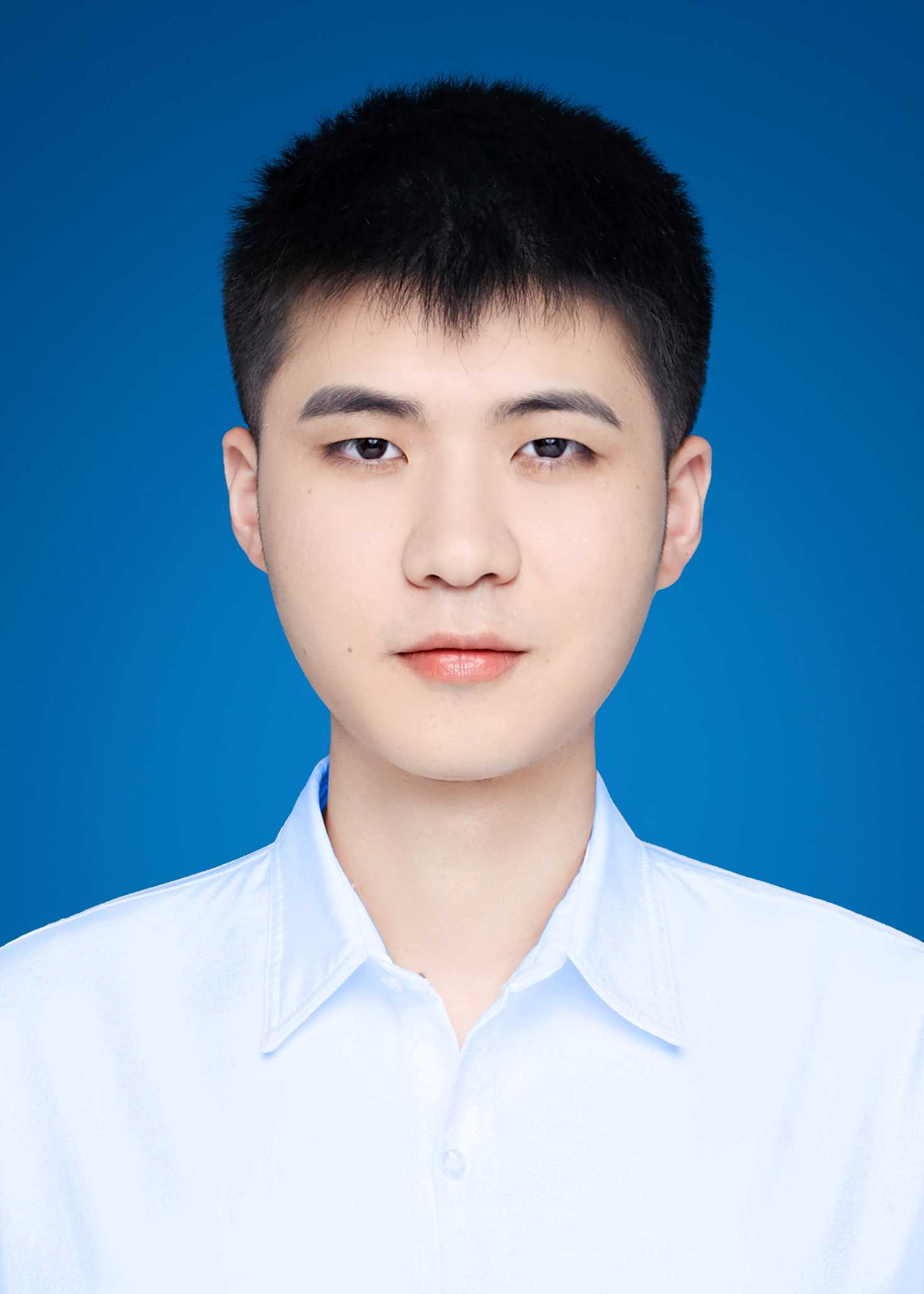}}]{Zhihai Wang} 
		received the B.Sc. degree in Electrical and Electronic Engineering from Huazhong University of Science and Technology, Wuhan, China, in 2020. He is currently a Ph.D. candidate in the School of Data Science at University of Science and Technology of China, Hefei, China. His research interests include reinforcement learning and learning to optimize.
	\end{IEEEbiography}
	\begin{IEEEbiography}[{\includegraphics[width=1in,height=1.25in,clip,keepaspectratio]{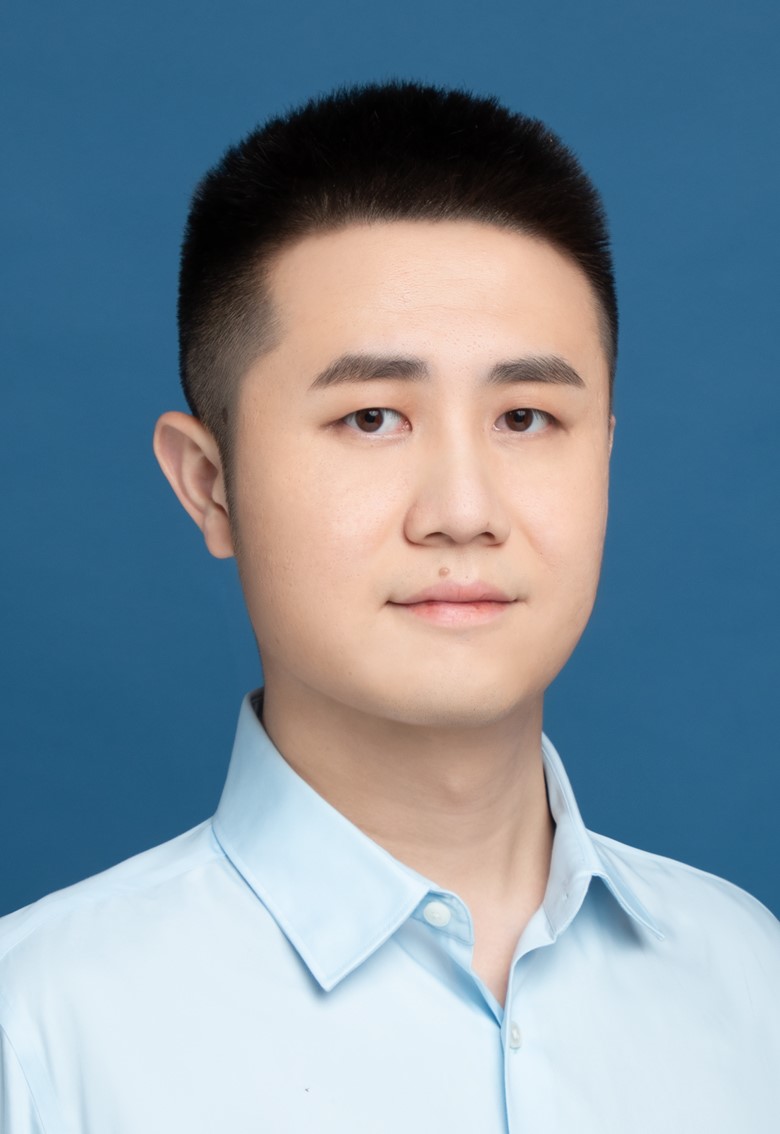}}]{Xijun Li} 
		is a senior researcher of HUAWEI Noah’s Ark Lab. Before that, he has received M.Sc. degree from Shanghai Jiao Tong University, P.R. China, in 2018. He is working towards his Ph.D. degree in the University of Science and Technology of China (HUAWEI-USTC Joint Ph.D. Program) under the supervision of Prof. Jie Wang. He has published 10+ papers on top peer-reviewed conferences and journals (such as ICLR, KDD, ICDE, SIGMOD, DAC, CIKM, ICDCS, TCYB, etc.) and applied/published 12 patents (first author in 6 patents) with Noah's Ark Lab. And he has won the championship of student learderboard in Dual Track of NeurIPS’21 ML4CO Competition. His recent research interests focus on Mathematical Programming Solver, Learning to Optimization (L2O) and Machine Learning for Computer System (ML4CS).
	\end{IEEEbiography}
	\begin{IEEEbiography}[{\includegraphics[width=1in,height=1.25in,clip,keepaspectratio]{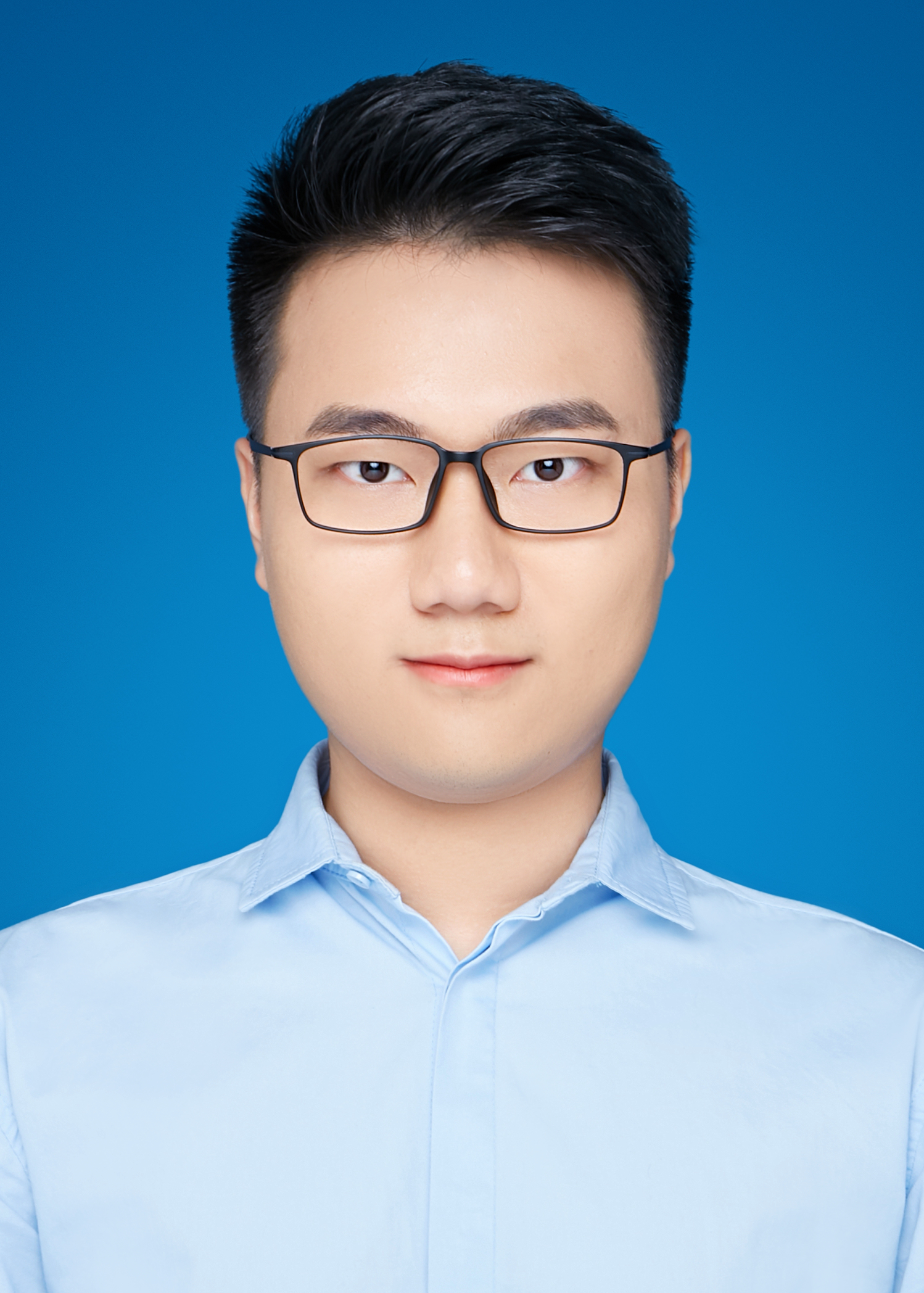}}]{Yufei Kuang} 
		received the B.Sc. degree in Statistics from Nanjing University, Nanjing, China, in 2020. He is currently a Ph.D. candidate in Department of Electronic Engineering and Information Science at University of Science and Technology of China, Hefei, China. His research interests include reinforcement learning and learning to optimize.
	\end{IEEEbiography}
	\begin{IEEEbiography}[{\includegraphics[width=1in,height=1.25in,clip,keepaspectratio]{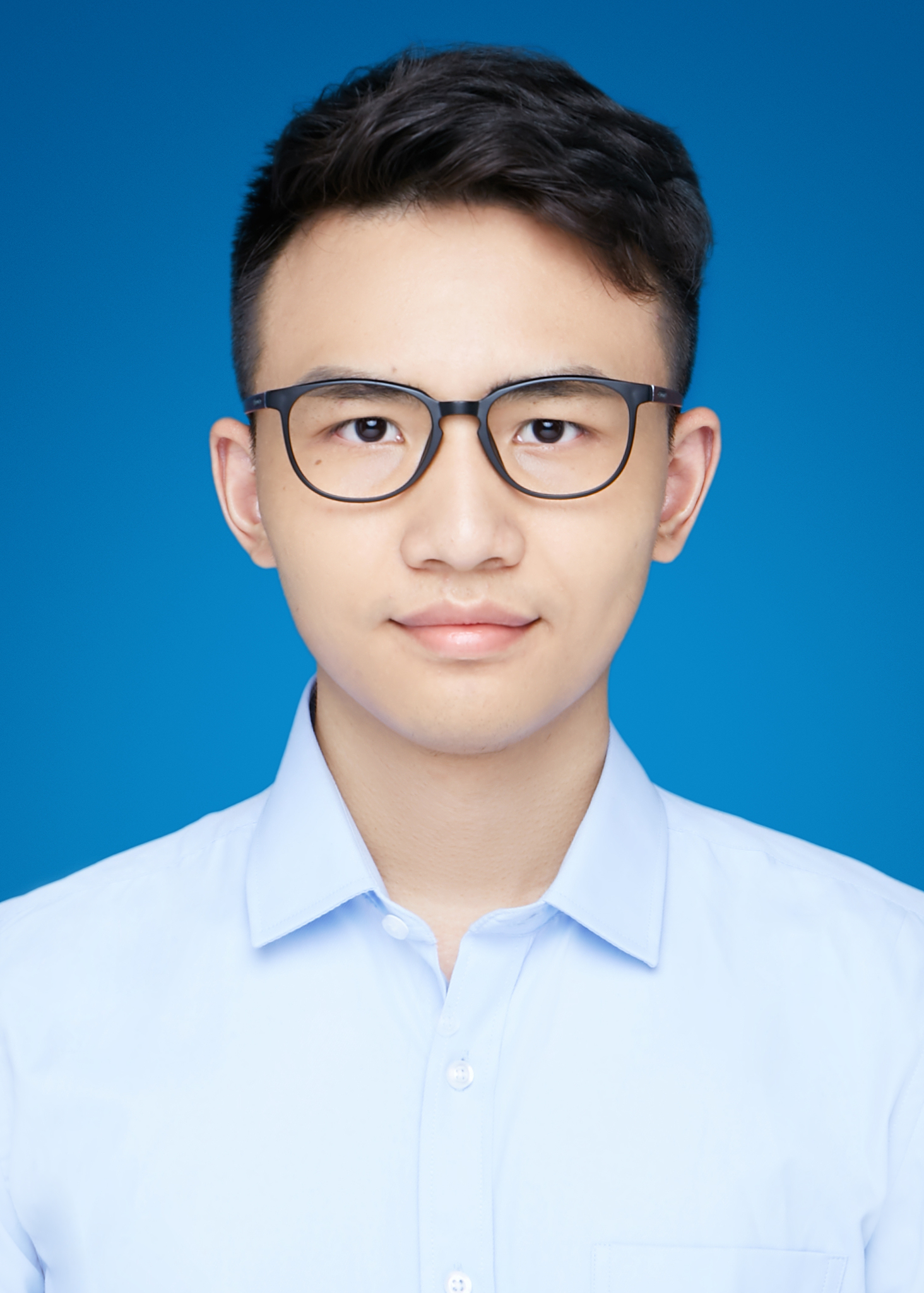}}]{Zhihao Shi}
		received the B.Sc. degree in Department of Electronic Engineering and Information Science from University of Science and Technology of China, Hefei, China, in 2020. a Ph.D. candidate in the Department of Electronic Engineering and Information Science at University of Science and Technology of China, Hefei, China. His research interests include graph representation learning and natural language processing.
	\end{IEEEbiography}
	\begin{IEEEbiography}[{\includegraphics[width=1in,height=1.25in,clip,keepaspectratio]{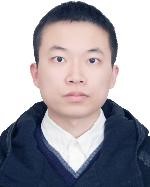}}]{Fangzhou Zhu} 
		is currently a principal engineer in Noah's Ark Lab, Huawei. He received the M.S degree from Soochow University. His current research interests include AI aided linear/mixed integer programming solver and its applications in supply chain management, especially for large scale industrial production planning problem. He is the key member of several supply chain optimization projects in Huawei.
	\end{IEEEbiography}
	\begin{IEEEbiography}[{\includegraphics[width=1in,height=1.25in,clip,keepaspectratio]{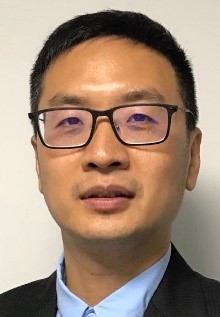}}]{Mingxuan Yuan} 
		is currently a principal researcher of Huawei Noah’s Ark Lab, Hong Kong. His research interests include data-driven optimization algorithms, data-driven SAT/MIP solving algorithms and data-driven EDA algorithm.
	\end{IEEEbiography}
	\begin{IEEEbiography}[{\includegraphics[width=1in,height=1.25in,clip,keepaspectratio]{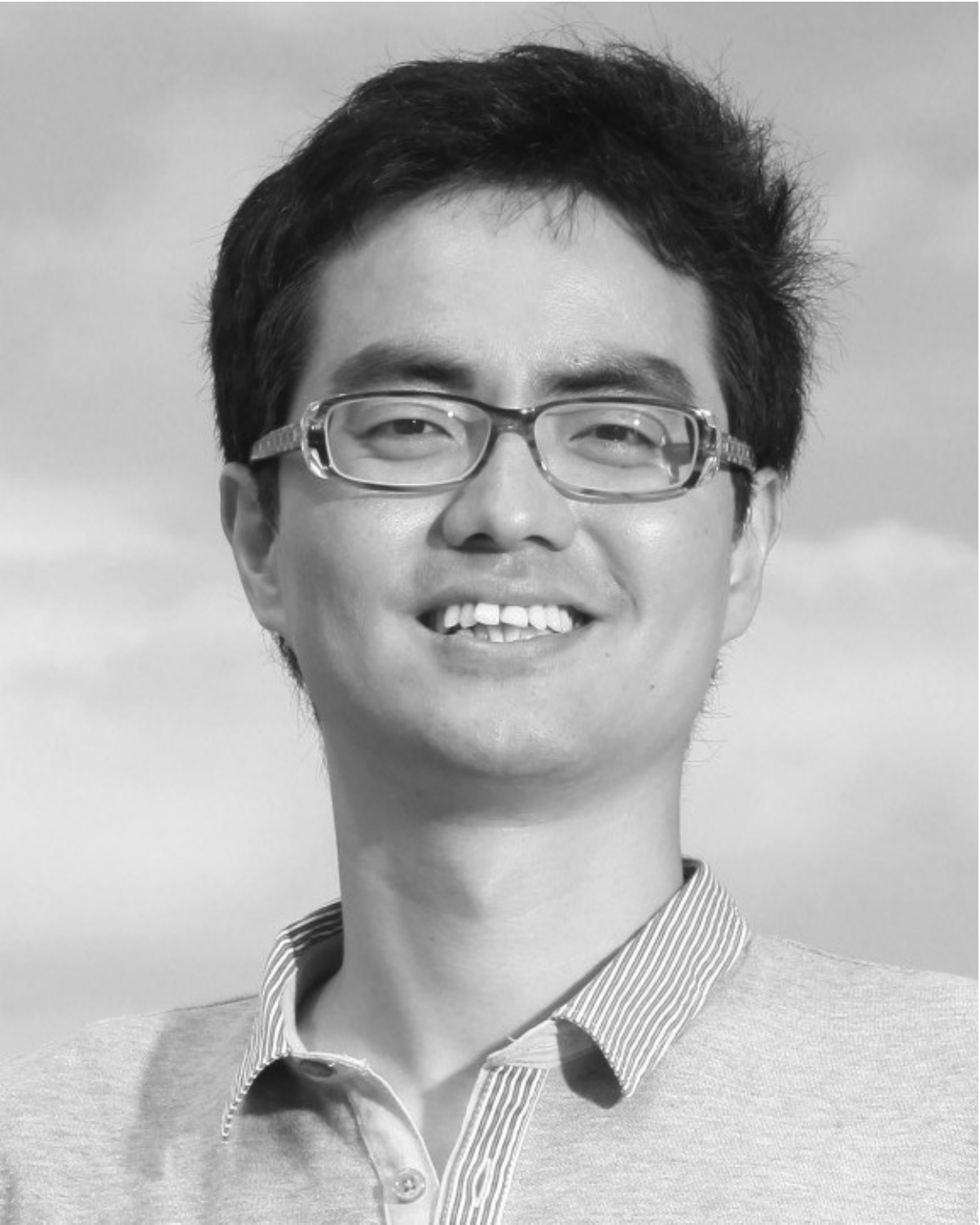}}]{Jia Zeng} 
		received the B.Eng. degree from the Wuhan University of Technology, Wuhan, China, in 2002,
		and the Ph.D. degree from the City University of Hong Kong, Hong Kong, in 2007.
		He is currently the AI chief scientist for enterprise intelligence (e.g., supply chain management) at Huawei Noah's Ark Lab, Hong Kong.
		He is also the director of the AI Strategic Planning and Cooperation Department.
		His research interests are foundation models, general problem solvers and embodied AI.
		He is a member of the CCF, the ACM and a senior member of the IEEE.
	\end{IEEEbiography}
	\begin{IEEEbiography}[{\includegraphics[width=1in,height=1.25in,clip,keepaspectratio]{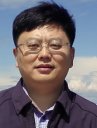}}]{Yongdong Zhang} 
		received the Ph.D. degree in electronic engineering from Tianjin University, Tianjin, China, in 2002. He is currently a Professor at the University of Science and Technology of China. He has authored more than 100 refereed journal articles and conference papers. His current research interests include multimedia content analysis and understanding, multimedia content security, video encoding, and streaming media technology. He serves as an Editorial Board Member of Multimedia Systems journal and Neurocomputing. He was the recipient of the Best Paper Award in PCM 2013, ICIMCS 2013, and ICME 2010, and the Best Paper Candidate in ICME 2011. He is a Senior Member of IEEE.
	\end{IEEEbiography}
	\begin{IEEEbiography}[{\includegraphics[width=1in,height=1.25in,clip,keepaspectratio]{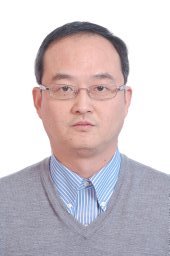}}]{Feng Wu} 
		received the B.S. degree in electrical engineering from Xidian University in 1992, and the M.S. and Ph.D. degrees in computer science from the Harbin Institute of Technology in 1996 and 1999, respectively. He is currently a Professor with the University of Science and Technology of China, where he is also the Dean of the School of Information Science and Technology. Before that, he was a Principal Researcher and the Research Manager with Microsoft Research Asia. His research interests include image and video compression, media communication, and media analysis and synthesis. He has authored or coauthored over 200 high quality articles (including several dozens of IEEE Transaction papers) and top conference papers on MOBICOM, SIGIR, CVPR, and ACM MM. He has 77 granted U.S. patents. His 15 techniques have been adopted into international video coding standards. As a coauthor, he received the Best Paper Award at 2009 IEEE Transactions on Circuits and Systems for Video Technology, PCM 2008, and SPIE VCIP 2007. He also received the Best Associate Editor Award from IEEE Circuits and Systems Society in 2012. He also serves as the TPC Chair for MMSP 2011, VCIP 2010, and PCM 2009, and the Special Sessions Chair for ICME 2010 and ISCAS 2013. He serves as an Associate Editor for IEEE Transactions on Circuits and Systems for Video Technology, IEEE Transactions ON Multimedia, and several other international journals.
	\end{IEEEbiography}


	
	
	
	

	\clearpage
	\newpage
	
	\appendices

	\section{Proof}\label{appendix_proof}
	\subsection{Proof of Proposition \ref{proof_hpg}}\label{appendix_proof_hpg_one_round}
	\begin{proof}
		Please refer to Appendix A in our previous version in \url{https://arxiv.org/abs/2302.00244}.
	\end{proof}

	\subsection{Proof of Proposition \ref{proof_hpg_multi_rounds}}\label{appendix_proof_hpg_multi_rounds}
	\begin{proof}
		The optimization objective takes the form of
		\begin{align*}
			J(\theta) = \mathbb{E}_{s_0,k_0,a_{0}^k\dots}[\sum_{t=0}^{\infty} \gamma^t r(s_t,k_t,a_{t}^k)],
		\end{align*}
		where $s_0\sim \mu$, $(k_t,a_t^k)\sim \pi_{\theta}(\cdot|s_t)$, $s_{t+1}=f(s_t,a_t)$, $\pi_{\theta}(k_t,a_t^k|s_t) = \pi^h_{\theta_1}(k_t|s_t) \cdot \pi^l_{\theta_2}(a_t^k|s_t,k_t)$, and $\gamma\in (0,1]$ is the discount factor. Then we derive a hierarchical policy gradient under the multiple rounds setting as follows. 
		
		We first rewrite the objective as follows
		\begin{align*}
			J([\theta_1,\theta_2]) & = \mathbb{E}_{s_0\sim \mu} [V^{\pi_{\theta}}(s_0)],
		\end{align*}
		where $V^{\pi_{\theta}}(s_0)$ denotes the state-value function. 
		
		We then compute the policy gradient for $\theta_1$:
		\begin{align*}
			& \nabla_{\theta_1} J([\theta_1,\theta_2])  \\
			& = \mathbb{E}_{s_0\sim\mu}[\nabla_{\theta_1} \int_{(k_0,a_0^{k_0})} \pi_{\theta}(k_0,a_0^{k_0}|s_0) Q^{\pi_{\theta}}(s_0,k_0,a_0^{k_0})d(k_0,a_0^{k_0})] \\
			& = \mathbb{E}_{s_0\sim\mu}[\int_{(k_0,a_0^{k_0})}\nabla_{\theta_1}  \pi_{\theta}(k_0,a_0^{k_0}|s_0) Q^{\pi_{\theta}}(s_0,k_0,a_0^{k_0})d(k_0,a_0^{k_0})] \\
			& + \mathbb{E}_{s_0\sim\mu}[\int_{(k_0,a_0^{k_0})} \pi_{\theta}(k_0,a_0^{k_0}|s_0) \nabla_{\theta_1}Q^{\pi_{\theta}}(s_0,k_0,a_0^{k_0})d(k_0,a_0^{k_0}) ] \\
			& = \mathbb{E}_{s_0\sim\mu, (k_0,a_0^{k_0})\sim\pi_{\theta}}[\nabla_{\theta_1} \log(\pi_{\theta_1}^h(k_0|s_0))Q^{\pi_{\theta}}(s_0,k_0,a_0^{k_0})] \\
			& + \gamma \mathbb{E}_{s_0\sim\mu, (k_0,a_0^{k_0})\sim\pi_{\theta},s_1= f}[\nabla_{\theta_1} V^{\pi_{\theta}}(s_1)] \\
			& = \mathbb{E}_{s_0\sim\mu, (k_0,a_0^{k_0})\sim\pi_{\theta}}[\nabla_{\theta_1} \log(\pi_{\theta_1}^h(k_0|s_0))Q^{\pi_{\theta}}(s_0,k_0,a_0^{k_0}) \\
			& + \gamma \mathbb{E}_{s_1=f, (k_1,a_1^{k_1})\sim\pi_{\theta}}[
			\nabla_{\theta_1} \log(\pi_{\theta_1}^h(k_1|s_1))Q^{\pi_{\theta}}(s_1,k_1,a_1^{k_1})
			]] \\
			& + \gamma^2 \mathbb{E}_{s_0\sim\mu, (k_0,a_0^{k_0})\sim\pi_{\theta},s_1=f,(k_1,a_1^{k_1})\sim\pi_{\theta},s_2= f}[\nabla_{\theta_1} V^{\pi_{\theta}}(s_2)] \\
			& = \cdots\cdots \\
			& = \mathbb{E}_{s\sim\rho^{\pi_{\theta}},(k,a^k)\sim \pi_{\theta}}[Q^{\pi_{\theta}}(s,k,a^k)  \nabla_{\theta_1} \log(\pi_{\theta_1}^h(k|s))],
		\end{align*}
		where $\rho^{\pi_{\theta}}(s) = \sum_{t=0}^{\infty} \gamma^t P(s_t = s|\pi_{\theta})$ denotes the expected state visitation frequencies under the policy $\pi_{\theta}$. 
		
		Similarly, we compute the policy gradient for $\theta_2$:
		\begin{align*}
			& \nabla_{\theta_2} J([\theta_1,\theta_2])  \\
			& = \mathbb{E}_{s_0\sim\mu}[\nabla_{\theta_2} \int_{(k_0,a_0^{k_0})} \pi_{\theta}(k_0,a_0^{k_0}|s_0) Q^{\pi_{\theta}}(s_0,k_0,a_0^{k_0})d(k_0,a_0^{k_0})] \\
			& = \mathbb{E}_{s_0\sim\mu}[\int_{(k_0,a_0^{k_0})}\nabla_{\theta_2}  \pi_{\theta}(k_0,a_0^{k_0}|s_0) Q^{\pi_{\theta}}(s_0,k_0,a_0^{k_0})d(k_0,a_0^{k_0})] \\
			& + \mathbb{E}_{s_0\sim\mu}[\int_{(k_0,a_0^{k_0})} \pi_{\theta}(k_0,a_0^{k_0}|s_0) \nabla_{\theta_2}Q^{\pi_{\theta}}(s_0,k_0,a_0^{k_0})d(k_0,a_0^{k_0}) ] \\
			& = \mathbb{E}_{s_0\sim\mu, (k_0,a_0^{k_0})\sim\pi_{\theta}}[\nabla_{\theta_2} \log(\pi_{\theta_2}^l(a_0^{k_0}|s_0,k_0))Q^{\pi_{\theta}}(s_0,k_0,a_0^{k_0})] \\
			& + \gamma \mathbb{E}_{s_0\sim\mu, (k_0,a_0^{k_0})\sim\pi_{\theta},s_1= f}[\nabla_{\theta_2} V^{\pi_{\theta}}(s_1)] \\
			& = \cdots\cdots \\
			& = \mathbb{E}_{s\sim\rho^{\pi_{\theta}},(k,a^k)\sim \pi_{\theta}}[Q^{\pi_{\theta}}(s,k,a^k)  \nabla_{\theta_2} \log(\pi_{\theta_2}^l(a^k|s,k))],
		\end{align*}
		where $\rho^{\pi_{\theta}}(s) = \sum_{t=0}^{\infty} \gamma^t P(s_t = s|\pi_{\theta})$ denotes the expected state visitation frequencies under the policy $\pi_{\theta}$. 
		
		We complete the proof.

	\end{proof}
	
	\subsection{Proof of Proposition \ref{pro:cuts}}\label{proof:cuts}
	In this part, we provide detailed proof of Proposition \ref{pro:cuts}. 
	
	\begin{proof}
		Let $\mathbf{y} \in \mathbb{Z} \times (\left[ 0,d \right] \cap \mathbb{Z})^n$ satisfying $\mathbf{y} \leq_{L} \mathbf{z}$. The proposition is obvious if $\mathbf{y} = \mathbf{z}$ or for the index value $i = 0 $ (i.e., $y_0 \leq \lfloor z_0 \rfloor$).
		
		So we assume now $\mathbf{y} \neq \mathbf{z}$ and $k \geq 0$ with $k = \text{argmin}_{j} \{ y_j < z_j: j=0,\dots,n \}$. Notice that the definition of $k$ and the assumption $\mathbf{y}\leq_L \mathbf{z}$ imply $y_q = z_q$ for all $q\in \{0,\dots, k-1\}$. Thus the inequality (\ref{eq:cuts}) is trivially satisfied for all $i\in \{0,\dots, k\}$. 
		
		Now we consider the left hand-side of the inequality (\ref{eq:cuts}) for the case $i>k$. For the vector $\mathbf{y}$, we have 
		\begin{align*}
			& y_i + \sum_{j=0}^{i-1} a_{i-j} (y_j-\lceil z_j \rceil) \\
			= & y_i + \sum_{j=k}^{i-1} a_{i-j} (y_j-\lceil z_j \rceil) \\
			\leq & y_i - a_{i-k} + \sum_{j=k+1}^{i-1} a_{i-j} (y_j - \lceil z_j \rceil) \,\, (\text{as}\,\, y_k-\lceil z_k \rceil \leq -1) \\
			\leq & y_i - a_{i-k} + d \sum_{j=k+1}^{i-1} a_{i-j} \,\,(\text{as}\,\,  y_j - \lceil z_j \rceil \leq d) \\
			= & y_i - a_{i-k} + d \sum_{j=1}^{i-k-1} a_{j} \\
			= & y_i - d \,\, (\text{as} \,\, a_{i-k} = d(1+\sum_{j=1}^{i-k-1}a_j)) \\ 
			\leq & 0 \leq \lfloor z_i \rfloor.
		\end{align*}
		Now we complete the proof.
	\end{proof}

	\subsection{Proof of Theorem \ref{theorem:convergence}}\label{proof:cuts_convergence}
	In this part, we provide detailed proof of Theorem \ref{theorem:convergence}, which follows \cite{neto2012simple}. To prove Theorem \ref{theorem:convergence}, we first introduce the following Lemmas. 
	
	\begin{definition}
		Given a vector $\mathbf{y} = (y_0,\dots,y_n) \in \mathbb{R} \times \left[ 0,d \right]^n$, $d \in \mathbb{Z}_{+}$, we define $\alpha(\mathbf{y})$ by the integral vector $\alpha(\mathbf{y})_i=\lfloor y_i \rfloor$ if $i \leq k$ and $\alpha(\mathbf{y})_i = d$ otherwise, for all $i\in \{0,\dots,n\}$, with $k=\text{argmin}_j\{ y_j \not\in \mathbb{Z}: j=0,\dots,n \}$. 
	\end{definition}
	
	\begin{lemma}\label{lemma:boundness}
		(Boundedness) Let $\mathbf{z}^{\star} = (z_0,z_1,\dots,z_n)$ denote an optimal solution of the LP relaxation of the lexicographical optimization problem (\ref{eq:obj_lex}) (i.e., the problem (\ref{eq:obj_lex_lp})), and $\mathbf{x}^{\star} = (x_0,x_1,\dots,x_n)$ denote an optimal integer solution of the lexicographical optimization problem (\ref{eq:obj_lex}). Then the following inequalities hold 
		\begin{align*}
			(d\sum_{i=1}^{n} \min (0,c_i),0,\dots,0) \leq_{L} \mathbf{x}^{\star} \leq_{L} \alpha(\mathbf{z}^{\star})
		\end{align*}
	\end{lemma}
	\begin{proof}
		If $\mathbf{z}^{\star}$ is integral, the case is trivial. Thus, we focus on the case when $\mathbf{z}^{\star}$ is not integral. The left inequality stems from the bounds on the variables $x_1, \dots, x_n$. The right inequality is due to the fact $\mathbf{z}^{\star}$ solves a relaxation of the problem solved by $\mathbf{x}^{\star}$, thus implying $\mathbf{x}^{\star}\leq_{L}\mathbf{z}^{\star}$, and since $\mathbf{x}^{\star}$ is integral we have $\mathbf{x}^{\star} \leq_{L} \alpha(\mathbf{z}^{\star})$. We complete the proof.
	\end{proof}
	
	\begin{lemma}\label{lemma:monotonically_decrease}
		(Monotonically Decreasing) Let $\mathbf{z}_t^{\star}$ denote an optimal solution of the linear programming relaxation problem that is solved at iteration $t$ in our focused cutting plane algorithm. Then, if $\mathbf{z}_t^{\star}$ and $\mathbf{z}_{t+1}^{\star}$ are not integral vectors, the following inequality holds: $\alpha(\mathbf{z}_{t+1}^{\star}) <_{L} \alpha(\mathbf{z}_t^{\star})$
	\end{lemma}
	\begin{proof}
		Since the vectors $\mathbf{z}_t^{\star}$ and $\mathbf{z}_{t+1}^{\star}$ are not integral we have $\alpha(\mathbf{z}_t^{\star}) <_{L} \mathbf{z}_t^{\star}$ and $\alpha(\mathbf{z}_{t+1}^{\star}) <_{L} \mathbf{z}_{t+1}^{\star}$. 
		
		Next we aim to prove $\mathbf{z}_{t+1}^{\star} <_{L} \alpha(\mathbf{z}_t^{\star})$. To prove this, we first prove that the following inclusion holds
		\begin{align*}
			\mathbb{A} = \{\mathbf{x} \in \mathbb{Q}: \mathbf{x} \,\,\text{satisfies}\,\, (\ref{eq:cuts})_k\} \subseteq \mathbb{B} = \{ \mathbf{x} \in \mathbb{Q}: \mathbf{x} \leq_L \alpha(\mathbf{z}_t^{\star}) \},
		\end{align*}
		where $\mathbb{Q} \subseteq \mathbb{R}\times \left[ 0,d \right]^n$, $\mathbb{Q} \subseteq \{ \mathbf{x} \in \mathbb{R}^{n+1}: \mathbf{x} \leq_L \mathbf{z}_t^{\star} \}$, and $k = \text{argmin}_{i}\{z_i \not\in \mathbb{Z}: i=0,\dots,n \} $ ($\mathbf{z}_t^{\star} = (z_0,\dots,z_n)$). 
		
		To prove that, we show that for any vector $\mathbf{y} \in  \mathbb{Q} \setminus  \mathbb{B}$, then $\mathbf{y} \not\in \mathbb{A}$. Note that $\mathbf{y} \in \mathbb{Q}$ but $\alpha(\mathbf{z}_t^{\star}) <_{L} \mathbf{y}$.
		As $\mathbb{Q} \subseteq \{ \mathbf{x} \in \mathbb{R}^{n+1}: \mathbf{x} \leq_L \mathbf{z}_t^{\star} \} $ and $\mathbf{y} \in \mathbb{Q}$, 
		we have $y_i=z_i$ for all $i\in \{ 0,1,\dots,k-1 \}$. So the inequality $\alpha(\mathbf{z}_t^{\star}) <_L \mathbf{y}$ implies $y_k > \lfloor z_k \rfloor$. Then $\mathbf{y}$ violates the inequality (\ref{eq:cuts})$_k$ as $y_i=z_i$ for all $i\in \{0,\dots,k-1\}$. That is, $\mathbf{y} \not\in \mathbb{A}$. Thus, we prove $\mathbb{A} \subseteq \mathbb{B}$.
		
		Now we are ready to prove $\mathbf{z}_{t+1}^{\star} <_{L} \alpha(\mathbf{z}_t^{\star})$. As $\mathbf{z}_{t+1}^{\star}$ satisfies the inequality (\ref{eq:cuts})$_{k}$, then $\mathbf{z}_{t+1}^{\star} \leq_{L} \alpha(\mathbf{z}_t^{\star})$. Due to the fact that $\mathbf{z}_{t+1}^{\star}$ is not integral, then $\mathbf{z}_{t+1}^{\star} <_L \alpha(\mathbf{z}_t^{\star})$. We complete the proof.
	\end{proof}
	
	Now we are ready to prove Theorem \ref{theorem:convergence}. We provide detailed proof as follows.
	\begin{proof}
		Let $\mathbf{x}^{\star}$ denote an optimal integer solution of the lexmax problem (\ref{eq:obj_lex}). From Lemma \ref{lemma:boundness}, $\mathbf{x}^{\star}$ is bounded with $(d\sum_{i=1}^{n} \min (0,c_i),0,\dots,0) \leq_{L} \mathbf{x}^{\star} \leq_{L} \alpha(\mathbf{z}_t^{\star})$, where $\mathbf{z}_t^{\star}$ denotes an optimal solution of the LP relaxation problem solved at iteration $t$. Since the number of integer vectors that are lexicographically between $(d\sum_{i=1}^{n} \min (0,c_i),0,\dots,0)$ and $\alpha(\mathbf{z}_t^{\star})$ is finite, and from Lemma \ref{lemma:monotonically_decrease} we have $\alpha(\mathbf{z}_{t+1}^{\star}) <_{L} \alpha(\mathbf{z}_{t}^{\star})$ whenever $\mathbf{z}_{t+1}^{\star}$ is not integral. Suppose $k = \text{argmin}_{i}\{ \alpha(\mathbf{z}_{t+1}^{\star})_i < \alpha(\mathbf{z}_{t}^{\star})_i: i=0,\dots,n\}$, then $\alpha(\mathbf{z}_{t}^{\star})_k - \alpha(\mathbf{z}_{t+1}^{\star})_k \geq 1$.
		Thus, we get that an integer optimal solution is found in a finite number of iterations. From Lemma \ref{lemma:boundness}, the maximum number of iterations is $\mathcal{O}(d^{n+1}C)$ with $C = 1+\sum_{i=1}^{n}| c_i |$. Now we complete the proof.
	\end{proof}

	\section{Details of the Datasets Used in this Paper}\label{appendix_dataset}
	\subsection{Datasets used in Section \ref{sec:order} in the main text}\label{datasets_motivating}
	Please refer to Appendix D in our previous version in \url{https://arxiv.org/abs/2302.00244}.
	
	\subsection{Huawei's large-scale real-world production planning and order matching problems}\label{appendix_datasets_sc}
	
	The production planning problem aims to find the optimal production planning for thousands of factories according to the daily order demand. The constraints include the production capacity for each production line in each factory, transportation limit, the order rate, etc. The optimization objective is to minimize the production cost and lead time simultaneously. We split the dataset into training and test sets with $80\%$ and $20\%$ of the instances. The average size of the production planning problems is approximately equal to $3500\times 5000=1.75\times10^8$, which are large-scale real-world problems.
	
	\subsection{The datasets used in the main evaluation}\label{appendix_main_datasets}
	
	Please refer to Appendix D in our previous version in \url{https://arxiv.org/abs/2302.00244}.
	
	\subsubsection{Detailed description of the aforementioned datasets}
	In this part, we provide detailed description of the aforementioned datasets. 
	Note that all datasets we use except MIPLIB 2017 are application-specific, i.e., they contain instances from only a single application. We summarize the statistical description of the used datasets in this paper in Table \ref{datasets}. Let $n,m$ denote the average number of variables and constraints in the MILPs. Let $m\times n$ denote the size of the MILPs. We emphasize that the largest size of our used datasets is up to two orders of magnitude larger than that used in previous learning-based cut selection methods \cite{tang_icml20, l2c_lookahead}, which demonstrates the superiority of our proposed HEM. Moreover, we test the inference time of our proposed HEM given the average number of candidate cuts. The results in Table \ref{datasets} show that the computational overhead of the HEM is very low.

	\section{More Results}\label{appendix_results}

	\subsection{More results of comparative evaluation of HEM}\label{appendix_results_more_metrics}
	
	Please refer to Appendix G in our previous version in \url{https://arxiv.org/abs/2302.00244}.
	
	\subsection{More results of comparative evaluation of HEM++}\label{appendix_results_hem++}
	\subsubsection{Evaluation under the multiple rounds setting}
	We compare HEM++ with the baselines under the multiple rounds setting on MIPLIB mixed neos as shown in Table \ref{appendix:more_results_hem++_multi_rounds}. The results show that HEM++ significantly outperforms the baselines, improving the efficiency by up to 29.47\%. This demonstrates that HEM++ is widely applicable to the multiple rounds setting.  
	
	\subsubsection{Evaluation under the one round setting}
	We compare HEM++ with HEM under the one round setting on Maximum Independent Set, MIK, and CORLAT. The results in Table \ref{appendix:more_results_hem++_one_round} show that HEM++ outperforms HEM on MIK (3.1\% improvement in terms of the PD integral), and performs on par with HEM on Maximum Independent Set and CORLAT. This demonstrates the superiority of HEM++ over HEM. 
	
	\begin{table}[t]
		\caption{Evaluation of HEM++ under the multiple rounds setting on MIPLIB mixed neos.}
		\label{appendix:more_results_hem++_multi_rounds}
		\centering
		\resizebox{0.48\textwidth}{!}{
			\begin{tabular}{@{}cccc@{}}
				\toprule
				\toprule
				& \multicolumn{3}{c}{MIPLIB mixed neos} \\ \midrule
				\multirow{2}{*}{Method} & \multirow{2}{*}{Time (s)$\downarrow$} & \multirow{2}{*}{PD integral$\downarrow$} & \multirow{2}{*}{\begin{tabular}[c]{@{}c@{}}Improvement $\uparrow$\\      (PD integral, \%)\end{tabular}} \\
				&  &  &  \\ \midrule
				NoCuts & \textbf{253.65 (80.30)} & 14652.29 (12523.37) & NA \\
				Default & 258.31 (75.24) & 12222.21 (12544.25) & 16.58 \\
				Random & 257.11 (74.61) & 12606.61 (11884.03) & 13.96 \\
				NV & 254.63 (78.67) & 13287.41 (11721.17) & 9.32 \\
				Eff & 263.78 (64.82) & 11345.05 (12506.45) & 22.57 \\ \midrule
				SBP & 258.90 (72.79) & 11983.70 (12724.59) & 18.21 \\
				HEM++ (Ours) & 255.77 (77.37) & \textbf{10334.00 (12537.6)} & \textbf{29.47} \\ \bottomrule
			\end{tabular}
		}
	\end{table}
	
	\begin{table*}[t]
		\caption{Evaluation of HEM++ under the one round setting.}
		\label{appendix:more_results_hem++_one_round}
		\centering
		\resizebox{0.96\textwidth}{!}{
			\begin{tabular}{@{}cccccccccc@{}}
				\toprule
				\toprule
				& \multicolumn{3}{c}{Maximum Independent Set} & \multicolumn{3}{c}{MIK} & \multicolumn{3}{c}{CORLAT} \\ \midrule
				\multirow{2}{*}{Method} & \multirow{2}{*}{Time (s) $\downarrow$} & \multirow{2}{*}{\begin{tabular}[c]{@{}c@{}}Improvement $\uparrow$\\      (\%, Time)\end{tabular}} & \multirow{2}{*}{PD integral$\downarrow$} & \multirow{2}{*}{Time (s)$\downarrow$} & \multirow{2}{*}{PD integral$\downarrow$} & \multirow{2}{*}{\begin{tabular}[c]{@{}c@{}}Improvement $\uparrow$\\      (\%, PD integral)\end{tabular}} & \multirow{2}{*}{Time (s)$\downarrow$} & \multirow{2}{*}{PD integral$\downarrow$} & \multirow{2}{*}{\begin{tabular}[c]{@{}c@{}}Improvement $\uparrow$\\      (\%, PD integral)\end{tabular}} \\
				&  &  &  &  &  &  &  &  &  \\ \cmidrule(r){1-4} \cmidrule(lr){5-7} \cmidrule(l){8-10}
				NoCuts & 8.78 (6.66) & NA & 71.32 (51.74) & 300.01 (0) & 2355.87 (996.08) & NA & 103.31 (128.14) & 2818.41 (5908.31) & NA \\
				Default & 3.89 (5.04) & 55.69 & 29.44 (35.28) & 179.63 (122.36) & 844.41 (924.30) & 64.16 & 75.21 (120.30) & 2412.09 (5892.88) & 14.40 \\ \midrule
				HEM (Ours) & \textbf{1.76 (3.69)} & \textbf{80.00} & \textbf{16.01 (26.21)} & \textbf{176.12 (125.18)} & 785.05 (790.39) & 66.68 & \textbf{58.31 (110.51)} & \textbf{1079.99 (2653.14)} & \textbf{61.68} \\
				HEM++ (Ours) & 1.90 (4.22) & 78.36 & 17.45 (30.83) & 204.90 (111.52) & \textbf{712.00 (777.09)} & \textbf{69.78} & 67.59 (119.51) & 1091.53 (2214.94) & 61.27 \\ \bottomrule
			\end{tabular}
		}
	\end{table*}

	\section{Algorithm Implementation and Experimental Settings}\label{appendix_implementation}
	
	\subsection{Implementation details and hyperparameters }\label{appendix_imple_hyper}
	
	Please refer to Appendix F in our previous version in \url{https://arxiv.org/abs/2302.00244}.
	
	\subsection{Implementation details of HEM++}
	\subsubsection{Details of the self-attention mechanism}\label{appendix_mha}
	Suppose there are $N$ candidate cuts, we denote each cut feature vector by $\textbf{c}_i$, $ i=1,2,3,\dots,N$. We compute the query $\textbf{q}_i\in \mathbb{R}^{d_k}$, key $\textbf{k}_i\in \mathbb{R}^{d_k}$, and value $\textbf{v}_i\in \mathbb{R}^{d_v}$ for each cut by projecting the cut feature: 
	\begin{align*}
		\textbf{q}_i = \textbf{W}^Q\textbf{c}_i, \,\, \textbf{k}_i = \textbf{W}^K \textbf{c}_i, \,\,
		\textbf{v}_i = \textbf{W}^V \textbf{c}_i,
	\end{align*}
	where $d_k$, $d_v$, and $d_c$ denote dimensions, $\textbf{W}^Q \in \mathbb{R}^{d_k\times d_c} $, $\textbf{W}^K \in \mathbb{R}^{d_k\times d_c} $, $\textbf{W}^V \in \mathbb{R}^{d_v\times d_c} $ are learnable parameters. Given the queries, keys, and values, we compute the embeddings by Scaled Dot-Product Attention \cite{transformer}. Specifically, for the query $\textbf{q}_i$ of cut $i$ and the key $\textbf{k}_i$ of cut $j$, we compute the attention weight by
	\begin{align*}
		u_{ij} =    \frac{\textbf{q}_i^{\top}\textbf{k}_j}{\sqrt{d_k}}.   
	\end{align*}
	We denote $\textbf{u}_i = \left[u_{i1},\dots u_{iN}\right]$. Finally, we compute the attention by
	\begin{align*}
		\textbf{c}_i^{\prime} = \sum_{j} \text{Softmax}(\textbf{u}_i)_j \textbf{v}_j,
	\end{align*}
	where $\text{Softmax}(\textbf{u}_i)_j$ denotes the $j$-th element of the softmax vector. 
	
	\subsubsection{Details of the hierarchical proximal policy optimization method}\label{appendix_hppo}
	\noindent\textbf{Estimation of the advantage function} We learn a state-value function $V_{\alpha}(s)$ as a baseline for reducing the variance of the advantage-function estimator. Following \cite{ppo}, we estimate the advantage-function by
	\begin{align*}
		\hat{A}(s,a) = G(s,a) - V_{\alpha}(s),
	\end{align*}
	where $G(s,a)$ denotes the monte-carlo return \cite{rl_sutton}. 
	
	\noindent\textbf{Hyperparameters} We apply the same optimizer and learning rates as that of HEM. For each epoch, we collect 32 samples, and we set the total epochs as 100. For each epoch, HEM performs one gradient update, while HEM++ performs ten gradient updates. Therefore, HEM++ is more sample-efficient than HEM. We set $\epsilon$ as 0.2. 
	
	\subsection{More details of HEM}
	
	Please refer to Appendix F in our previous version in \url{https://arxiv.org/abs/2302.00244}.
	
	\subsection{Implementation details of the baselines}
	\subsubsection{The one round setting}\label{appendix_imple_baselines_one_round}
	
	Please refer to Appendix F in our previous version in \url{https://arxiv.org/abs/2302.00244}.
	
	\subsubsection{The multiple rounds setting}\label{appendix_imple_baselines_multi_rounds}
	For the human-designed baselines, they select a fixed ratio of cuts with high scores at each round until the cut separation procedure terminates. For SBP,  it learns a scoring function to select a fixed ratio of cuts with high scores at each round until the cut separation procedure terminates, and achieves a return feedback for training the scoring function via evolutionary strategies.

	\section{Additional Results}
	
	\subsection{More results of ablation study}\label{appendix_results_ablation_study}
	
	In this section, we provide more results of ablation studies in the main text. 
	
	\subsubsection{In-depth analysis of HEM-ratio and SBP}\label{appendix_analysis_hem_ratio}
	We provide possible reasons for HEM-ratio performing poorer than SBP on several challenging MILP problem benchmarks. 
	Fundamentally, HEM-ratio formulates the cut selection task as a sequence modeling problem, which has two main advantages over SBP. That is, the sequence model can not only capture the underlying order information, but also capture the interaction among cuts. However, training a sequence model is more difficult than training a scoring function, as the sequence model aims to learn a much more complex task. Specifically, the scoring function aims to learn to score each cut, while the sequence model aims to model the joint probability of the selected cuts. The latter is a more challenging learning task. Moreover, we follow the reinforcement learning paradigm instead of supervised learning to train the model, making the training process more unstable. Therefore, the sequence model may suffer from inefficient exploration and be trapped to a local optimum. As a result, HEM-ratio can perform poorer than SBP, especially on challenging MILP problem benchmarks. 
	
	\subsubsection{Contribution of each component}
	Please refer to Appendix G in our previous version in \url{https://arxiv.org/abs/2302.00244}.
	
	\subsubsection{The importance of tackling \textbf{P1-P3} in cut selection}
	Please refer to Appendix G in our previous version in \url{https://arxiv.org/abs/2302.00244}.
	
	\subsection{More visualization results}\label{appendix_results_visualization}
	Please refer to Appendix G in our previous version in \url{https://arxiv.org/abs/2302.00244}.
	
	\subsection{Analysis on structure of selected cuts  }\label{appendix_results_visualization_structure}
	We provide the results in Table \ref{tab:hem_knapsack_cut_structure}.   
	Moreover, we analyze categories of the selected cuts by HEM++ on problems with knapsack constraints as well. The results in Table \ref{tab:hem_knapsack_cut_structure} show that policies learned by HEM++ select 90.95\%  and 99.6\% cover inequalities (cover cuts) on Multiple Knapsack and MIK, respectively. Thus, the results suggest that HEM++ can also well capture the underlying structure of specific structured problems, such as knapsack problems.

	\begin{table*}[t]
		\centering
		\caption{HEM and HEM++ select more than 90\% cover cuts, which is known as an important category of cuts for solving problems with knapsack constraints \cite{lift_cover, flow_cover}.}
		\label{tab:hem_knapsack_cut_structure}
		\resizebox{0.98\textwidth}{!}{
			\begin{tabular}{@{}cccccc@{}}
				\toprule
				\toprule
				Method & Knapsack problems & Percent of selected cover cuts   (\%) & Method & Knapsack problems & Percent of selected cover cuts   (\%) \\ \midrule
				HEM & Multiple Knapsack & 92.03 & HEM++ & Multiple Knapsack & 90.95 \\
				& MIK & 99.5 &  & MIK & 99.6 \\ \bottomrule
			\end{tabular}
		}
	\end{table*}

	\subsection{Evaluation with a time limit of three hours}\label{appendix_results_long_test_time}
	Please refer to Appendix G in our previous version in \url{https://arxiv.org/abs/2302.00244}.
	
	\subsection{Training curves}
	Please refer to Appendix G in our previous version in \url{https://arxiv.org/abs/2302.00244}.
	
	\subsection{Generalize to non-root nodes}
	In this part, we evaluate whether HEM can well generalize to non-root nodes. The results in Tables \ref{evaluation_one_round_non_root} and \ref{evaluation_multiple_rounds_non_root} demonstrate that our learned models outperform the baselines for all nodes (both root and non-root nodes) under the one round setting and multiple rounds setting. Specifically, under the one round setting with non-root cuts, our model improves the Time and Primal-dual gap integral by up to 91.29\% and 29.61\%, respectively. Under the multiple rounds setting with non-root cuts, our model improves the Time and Primal-dual gap integral by up to 97.72\% and 33.02\%, respectively.

	In a branch-and-bound search tree, modern solvers solve a linear programming (LP) relaxation of a Mixed-Integer Linear Program (MILP) at each node. Compared to adding cuts at root nodes, \textbf{adding cuts at both root and non-root nodes would further strengthen the LP realaxations at non-root nodes}, which thus may significantly improve the efficiency of solving MILPs. However, adding cuts at non-root nodes could lead to large models, which can increase the computational burden and present numerical instabilities \cite{implementing_cutting,l2c_lookahead}. Therefore, whether to add cuts at non-root nodes relies on heuristics as well in real-world applications. 
	
	\begin{table*}[t]
		\caption{The average performance of HEM and the baselines for non-root nodes under the one round setting. We obtain the results by deploying the models---which are learned at root nodes---to all nodes (root and non-root nodes). The best performance is marked
			in bold.}
		\label{evaluation_one_round_non_root}
		\centering
		\resizebox{0.96\textwidth}{!}{
			\begin{tabular}{@{}ccccccc@{}}
				\toprule
				\toprule
				One round with non-root nodes & \multicolumn{3}{c}{Easy: Maximum Independent Set} & \multicolumn{3}{c}{Hard: MIPLIB mixed supportcase} \\ \midrule
				\multirow{2}{*}{Method} & \multirow{2}{*}{Time (s)} & \multirow{2}{*}{\begin{tabular}[c]{@{}c@{}}Improvement \\      (Time, \%)\end{tabular}} & \multirow{2}{*}{PD integral} & \multirow{2}{*}{Time (s)} & \multirow{2}{*}{PD integral} & \multirow{2}{*}{\begin{tabular}[c]{@{}c@{}}Improvement \\      (PD integral, \%)\end{tabular}} \\
				&  &  &  &  &  &  \\ \cmidrule(r){1-4} \cmidrule(l){5-7} 
				NoCuts & 8.78 & NA & 71.32 & 170.00 & 9927.96 & NA \\
				Default & 1.47 & 83.25 & 13.95 & 128.89 & 9406.43 & 5.25 \\
				Random & 2.07 & 76.47 & 20.78 & 124.65 & 9116.01 & 8.18 \\
				NV & 3.21 & 63.42 & 29.81 & 134.58 & 8034.05 & 19.08 \\
				Eff & 2.08 & 76.27 & 21.06 & \textbf{124.20} & 9035.90 & 8.99 \\ \midrule
				HEM (Ours) & \textbf{0.76} & \textbf{91.29} & \textbf{9.75} & 135.80 & \textbf{6987.87} & \textbf{29.61} \\ \bottomrule
			\end{tabular}
		}
	\end{table*}
	
	\begin{table*}[t]
		\caption{The average performance of HEM++ and the baselines for non-root nodes under the multiple rounds setting. We obtain the results by deploying the models---which are learned at root nodes---to all nodes (root and non-root nodes).}
		\label{evaluation_multiple_rounds_non_root}
		\centering
		\resizebox{0.96\textwidth}{!}{
			\begin{tabular}{@{}ccccccc@{}}
				\toprule
				\toprule
				Multiple rounds with non-root nodes & \multicolumn{3}{c}{Easy: Maximum Independent Set} & \multicolumn{3}{c}{Hard: MIPLIB mixed supportcase} \\ \midrule
				\multirow{2}{*}{Method} & \multirow{2}{*}{Time (s)} & \multirow{2}{*}{\begin{tabular}[c]{@{}c@{}}Improvement \\      (Time, \%)\end{tabular}} & \multirow{2}{*}{PD integral} & \multirow{2}{*}{Time (s)} & \multirow{2}{*}{PD integral} & \multirow{2}{*}{\begin{tabular}[c]{@{}c@{}}Improvement \\      (PD integral, \%)\end{tabular}} \\
				&  &  &  &  &  &  \\ \cmidrule(r){1-4} \cmidrule(l){5-7} 
				NoCuts & 8.78 & NA & 71.32 & 170.00 & 9927.96 & NA \\
				Default & 0.28 & 96.78 & 6.41 & 123.13 & 8139.53 & 18.01 \\
				Random & 1.04 & 88.11 & 11.00 & 125.03 & 9416.23 & 5.15 \\
				NV & 1.14 & 87.05 & 12.05 & 107.78 & 7551.19 & 23.94 \\
				Eff & 0.27 & 96.97 & 6.39 & \textbf{113.25} & 7246.07 & 27.01 \\ \midrule
				HEM++ (Ours) & \textbf{0.20} & \textbf{97.72} & \textbf{5.86} & 118.12 & \textbf{6649.86} & \textbf{33.02} \\ \bottomrule
			\end{tabular}
		}
	\end{table*}
	
	\begin{table}[t]
		\caption{Comparison between Default+ and Default on CORLAT.}
		\label{appendix_default+}
		\centering
		\resizebox{0.48\textwidth}{!}{
			\begin{tabular}{@{}cccc@{}}
				\toprule
				\toprule
				& \multicolumn{3}{c}{CORLAT} \\ \midrule
				Method & Time (s) $\downarrow$ & PD integral $\downarrow$& Improvement $\uparrow$(\%, PD integral) \\ \midrule
				NoCuts & 103.31 (128.14) & 2818.41 (5908.31) & NA \\
				Default & 75.21 (120.30) & 2412.09 (5892.88) & 14.42 \\
				Default+ & \textbf{73.67 (119.14)} & \textbf{2315.50 (5765.29)} & \textbf{17.84} \\ \bottomrule
			\end{tabular}
		}
	\end{table}
	
	\subsection{Generalize to other solvers}
	
	Our proposed methodology can well generalize to other solvers as shown in Table \ref{generalization_other_solvers}. The results demonstrate that HEM significantly outperforms the default cut selection method in the CBC solver \cite{cbc} in terms of the primal-dual gap (up to 18.67\% improvement). 
	
	We do not use commercial solvers, such as Gurobi \cite{gurobi_solver} and \cite{cplex}, as the backend solver, since they do not provide interfaces for users to customize cut selection methods. 
	
	As the CBC cannot generate any cut on the dataset Maximum Independent Set, we conduct the experiments on the dataset Load balancing. 
	
	We use the primal-dual gap metric rather than the primal-dual gap integral due to the reasons as follows. (1) The primal-dual gap is a well-recognized metric for evaluating the solvers as well. (2) Unlike the SCIP, the CBC does not provide interfaces for users to acquire the primal-dual gap integral. Due to limited time, we do not implement the interface. 
	
	\begin{table*}[t]
		\caption{The performance of HEM and the default strategy used in the CBC solver \cite{cbc}.}
		\label{generalization_other_solvers}
		\centering
		\resizebox{0.96\textwidth}{!}{
			\begin{tabular}{ccccc}
				\toprule
				\toprule
				& \multicolumn{2}{c}{MIPLIB mixed supportcase} & \multicolumn{2}{c}{Load balancing} \\
				Method & Primal dual gap (PD gap) & Improvement (\%,PD gap) & Primal dual gap (PD gap) & Improvement (\%,PD gap) \\ \cmidrule(r){1-3} \cmidrule(l){4-5}
				CBC Default & 227.93 & NA & 0.98 & NA \\
				HEM (ours) & \textbf{185.38} & \textbf{18.67} & \textbf{0.91} & \textbf{7.14} \\ \bottomrule
			\end{tabular}
		}
	\end{table*}
	
	\subsection{A detailed computational analysis of our model and the baselines' model}
	
	We provide a detailed computational analysis of our proposed model and the baselines' model in Table \ref{computational_analysis}. We summarize the conclusions in the following. (1) The training time of HEM and SBP is comparable, as most of their training time is spent on interacting with solvers to collect training samples. (2) The model parameters of HEM (212749) and SBP with LSTM (172289) are comparable. (3) The inference time of HEM is longer than that of SBP and SBP with LSTM. Nevertheless, the inference time of HEM (0.34s on average) is very low compared to the solving time (162s on average), especially on hard datasets. (4) The training of HEM is stable (please see Figure \ref{fig:training_curves}).  
	
	\subsection{Measuring the primal and dual integrals}
	
	We have conducted experiments to measure the Primal Integral (PI) and Dual Integral (DI) as shown in Table \ref{primalintegral_dualintegral}. The results show that the performance improvement of HEM is from both the primal and dual sides.
	
	Specifically, we use the optimal objective values as the reference values to measure the PI/DI. However, it is time-consuming to obtain optimal solutions for all instances. We conduct the experiments on three easy datasets due to limited time. Interestingly, the results demonstrate that proper cut selection policies can improve both the PI and DI. Moreover, the results show that HEM achieves more improvement from the primal side than the dual side on Set Cover and Maximum Independent Set, while HEM achieves more improvement from the dual side on Multiple Knapsack.
	
	\begin{table*}[t]
		\caption{A detailed computational analysis of our model and the baselines' model.}
		\label{computational_analysis}
		\centering
		\resizebox{0.96\textwidth}{!}{
			\begin{tabular}{@{}cccccccc@{}}
				\toprule
				\toprule
				& \multicolumn{7}{c}{Maximum Independent Set} \\ \midrule
				& \multicolumn{2}{c}{Model characteristics} & \multicolumn{2}{c}{Training} &  & \multicolumn{2}{c}{Testing} \\ 
				Model & Model paramters & GPU Memory (MB) & Training time (h) & Training samples & Avg Cuts & Inference time (s) & Performance/Time (s) \\ \midrule
				SBP & 18433 & 2.07 & 3.03 & \multirow{3}{*}{3200} & \multirow{3}{*}{57} & 0.0003 & 2.43 \\
				SBP+LSTM encoder & 172289 & 2.66 & 2.83 &  &  & 0.031 & 2.87 \\
				HEM (Ours) & 212749 & 2.81 & 2.54 &  &  & 0.11 & 1.76 \\ \bottomrule 
			\end{tabular}
		}
		\newline
		\resizebox{0.96\textwidth}{!}{
			\begin{tabular}{@{}cccccccc@{}}
				\toprule
				\multicolumn{1}{l}{} & \multicolumn{7}{c}{MIPLIB mixed supportcase} \\ \midrule
				& \multicolumn{2}{c}{Model characteristics} & \multicolumn{2}{c}{Training} &  & \multicolumn{2}{c}{Testing} \\
				Model & Model paramters & GPU Memory & Training time (h) & Training samples & Avg Cuts & Inference time & Performance/PD integral \\ \midrule
				SBP & 18433 & 2.07 & 13.81 & \multirow{3}{*}{3200} & \multirow{3}{*}{173} & 0.0004 & 7408.65 \\
				SBP+LSTM encoder & 172289 & 2.66 & 13.24 &  &  & 0.033 & 7726.54 \\
				HEM (Ours) & 212749 & 2.81 & 13.89 &  &  & 0.34 & 6874.8 \\ \bottomrule 
			\end{tabular}
		}
	\end{table*}
	
	\begin{table*}[t]
		\caption{Evaluate the performance of HEM and Default in terms of the primal and dual integrals.}
		\label{primalintegral_dualintegral}
		\centering
		\resizebox{0.96\textwidth}{!}{
			\begin{tabular}{@{}ccccc@{}}
				\toprule
				\toprule
				& \multicolumn{4}{c}{Set Cover} \\ \midrule
				Method & PrimalIntegral (PI) & Improvement (\%, PI) & DualIntegral (DI) & Improvement (\%, DI) \\ \midrule
				NoCuts & 52.34 & NA & 59.85 & NA \\
				Default & 45.02 & 13.99 & 49.95 & 16.54 \\
				HEM (ours) & \textbf{28.84} & \textbf{44.90} & \textbf{35.95} & \textbf{39.93} \\ \bottomrule
			\end{tabular}
		}
		\resizebox{0.96\textwidth}{!}{
			\begin{tabular}{@{}ccccc@{}}
				\toprule
				& \multicolumn{4}{c}{Maximum Independent Set} \\ \midrule
				Method & PrimalIntegral (PI) & Improvement (\%, PI) & DualIntegral (DI) & Improvement (\%, DI) \\ \midrule
				NoCuts & 66.83 & NA & 16.24 & NA \\
				Default & 32.19 & 51.83 & 11.92 & 26.60 \\
				HEM (ours) & \textbf{18.33} & \textbf{72.57} & \textbf{7.97} & \textbf{50.92} \\ \bottomrule
			\end{tabular}
		}
		\resizebox{0.96\textwidth}{!}{
			\begin{tabular}{@{}ccccc@{}}
				\toprule
				& \multicolumn{4}{c}{Multiple Knapsack} \\ \midrule
				Method & PrimalIntegral (PI) & Improvement (\%, PI) & DualIntegral (DI) & Improvement (\%, DI) \\ \midrule
				NoCuts & 39.39 & NA & 41.62 & NA \\
				Default & 25.40 & 35.52 & 25.70 & 38.25 \\
				HEM (ours) & \textbf{19.24} & \textbf{51.16} & \textbf{18.90} & \textbf{54.59} \\ \bottomrule
			\end{tabular}
		}
	\end{table*}
	
	\subsection{More results of evaluation of Default+}\label{appendix_order_rules_results}
	We provide the statistical results on three hard datasets in Table \ref{hem_separator_category_percent}. 
	Furthermore, we compare Defalut+ with Default on CORLAT as shown in Table \ref{appendix_default+}. The results show that Default+ outperforms Default (3.42\% improvement in terms of the PD integral). 
	
	\subsection{Industrial applicability of HEM}
	
	To thoroughly demonstrate that our proposed HEM is applicable to very large-scale datasets, we conduct the following experiments. Note that we use the SCIP \cite{scip8}, the state-of-the-art open-source solver, as our backend solver. \textbf{First}, we evaluate HEM on the entire MIPLIB dataset to demonstrate its applicability to very large-scale datasets. \textbf{Second}, we have deployed our HEM to a commercial solver, and further demonstrate its industrial applicability on security-constrained unit commitment (SCUC) problems from the electrical grid. \textbf{Third}, we analyze the model inference efficiency of HEM on large-scale datasets compared with simple scoring models to demonstrate that HEM is efficient on large-scale datasets as well.
	
	\textbf{Applicability of HEM to the Entire MIPLIB Dataset} Given the entire MIPLIB dataset, we first follow Table \ref{tab:Criteria} to filter instances on which cut selection policies do not have a significant impact on the solving efficiency, which follows \cite{adaptive_cut_selection}. As a result, we evaluate HEM on a large MIPLIB dataset with 645 instances in total. 
	
	\textbf{First}, for validation and model selection, we randomly choose sixteen instances from the entire MIPLIB dataset as a validation set. The results in Table \ref{hem_miplib_whole_testing} show that HEM outperforms the baselines in terms of the average solving time and primal-dual integral on the validation set, demonstrating the applicability of HEM to large-scale datasets. \textbf{Second}, we select the best model on the validation set, and compare the model with the Default cut selection strategy on all instances from the entire MIPLIB dataset. The results in Table \ref{hem_miplib_whole_all} (Left) show that HEM outperforms Default in terms of the average solving time and performs on par with Default in terms of the average primal-dual integral, further demonstrating the applicability of HEM to very large-scale datasets. \textbf{Finally}, we also report detailed performance of HEM and Default on each instance from the entire MIPLIB dataset in Tables \ref{hem_miplib_whole_all_each_instance1}-\ref{hem_miplib_whole_all_each_instance11}. The results show that HEM struggles to consistently outperform Default on all instances from MIPLIB, \textbf{as these instances are rather heterogeneous, posing significant challenges for machine learning methods \cite{adaptive_cut_selection}}. Nevertheless, the results show that HEM can significantly outperform Default on many very large-scale instances, such as the instances neos-780889 (size $\approx1.35\times 10^{10}$, 23.71\% improvement, see Table \ref{hem_miplib_whole_all_each_instance11}) and map14860-20 (size $\approx 5.41\times 10^{10}$, 18.25\% improvement, see Table \ref{hem_miplib_whole_all_each_instance9}). The results further demonstrate the applicability of HEM to very large-scale datasets. 
	
	\textbf{Industrial Applicability to a Commercial Solver} In addition to the open source solver SCIP \cite{scip8}, we have deployed our HEM into a commercial solver (i.e., OptVerse \cite{optverse} from Huawei) to further demonstrate the industrial applicability of HEM. (We do not use alternative commercial solvers, such as Gurobi \cite{gurobi} and Cplex \cite{cplex}, as they do not provide interfaces for users to customize cut selection methods based on their generated cuts. In contrast, the developers of OptVerse provide the interface for us to customize cut selection methods based on the generated cuts by OptVerse, which applies to the cut selection setting we focus on when using SCIP.)
	To demonstrate the strong ability of HEM to enhance commercial solvers, we compare HEM-OptVerse with Default-OptVerse on real SCUC problems from the electrical grid as shown in Table \ref{hem_miplib_whole_all} (Right). The results show that HEM-OptVerse significantly outperforms Default-OptVerse, improving the solving time by 8.99\%. 
	
	\textbf{Efficiency Analysis of HEM} We provided a detailed computational analysis of HEM compared with simple scoring models in Appendix F.10. As shown in Table \ref{computational_analysis}, although the inference time of HEM (0.34 seconds on average) considerably exceeds that of scoring models (0.0004 seconds on average), the inference time for HEM remains rather low compared to the total solving time (162 seconds on average). Especially on large-scale datasets, the total solving time may significantly increase, whereas the model inference time of HEM exhibits only marginal changes. The major reason is that the model inference time of HEM is contingent on the number of input candidate cuts, while the number of candidate cuts remains consistent across varying dataset scales, typically ranging from 20 to 5,000. Therefore, the HEM model consistently maintains its inference efficiency when handling very large-scale datasets.

	\subsection{More Motivating Results}
	We present more motivating results on the order of selected cuts in Table \ref{hem_order}. The results show that adding the same selected cuts in different order significantly impacts the solver performance \textit{by up to orders of magnitude}. 
	
	\begin{table}[t]
		\caption{More motivating results on the order of selected cuts. The results show that adding the
			same selected cuts in different order significantly impacts the solver
			performance by up to orders of magnitude.}
		\label{hem_order}
		\centering
		\resizebox{0.48\textwidth}{!}{
			\begin{tabular}{@{}cccc@{}}
				\toprule
				\toprule
				& Nodes &  &  \\ \midrule
				Instance name & Max & Min & Ratio \\ \midrule
				lectsched-5-obj & 34993.00 & 2566.00 & 13.60 \\
				berlin & 15534.00 & 1.00 & 15500.00 \\ \bottomrule
			\end{tabular}
		}
	\end{table}
	
	\begin{table*}[t]
		\caption{The cutting plane categories of selected cuts by our learned policies.}
		\label{hem_separator_category_percent}
		\centering
		\resizebox{0.96\textwidth}{!}{
			\begin{tabular}{@{}ccccccc@{}}
				\toprule
				\toprule
				& \multicolumn{2}{c}{Anonymous ($n=37881,\,\,m=49603$)} & \multicolumn{2}{c}{Production planning ($n=3582.25,\,\,m=5040.42$)} & \multicolumn{2}{c}{Order matching ($n=67839.68,\,\,m=31364.84$)} \\ \midrule
				Ranking & Category & Percent (\%) & Category & Percent (\%) & Category & Percent (\%) \\ \cmidrule(r){1-3} \cmidrule(lr){4-5} \cmidrule(lr){6-7}
				The first & Implbd & 50 & Implbd & 60 & Intobj & 90 \\
				The second & Knapsack Cover & 45 & Gomory & 50 & Implbd & 30 \\
				The third & Gomory & 35 & Clique & 50 & Knapsack cover & 50 \\ \bottomrule
			\end{tabular}
		}
	\end{table*}
	
	\begin{table*}[t]
		\caption{Evaluation of HEM on the validation set of the entire MIPLIB dataset.}
		\label{hem_miplib_whole_testing}
		\centering
		\resizebox{0.9\textwidth}{!}{
			\begin{tabular}{@{}ccccc@{}}
				\toprule
				\toprule
				& \multicolumn{3}{c}{MIPLIB Entire Validation Set} &  \\ \midrule
				\multirow{2}{*}{Method} & \multirow{2}{*}{Time (s)} & \multirow{2}{*}{\begin{tabular}[c]{@{}c@{}}Improvement \\      (\%, Time)\end{tabular}} & \multirow{2}{*}{PD integral} & \multirow{2}{*}{\begin{tabular}[c]{@{}c@{}}Improvement \\      (\%, PD integral)\end{tabular}} \\
				&  &  &  &  \\ \midrule
				NoCuts & 265.31 (86.24) & NA & 11566.19 (12644.29) & NA \\
				Default & 253.79 (97.92) & 4.34 & 10809.30 (12217.92) & 6.54 \\
				Random & 263.64 (85.60) & 0.63 & 11174.70 (12462) & 3.38 \\
				NV & 263.10 (89.31) & 0.83 & 11260.19 (12482.14) & 2.65 \\
				Eff & 263.03 (88.06) & 0.86 & 11143.30 (12479.65) & 3.66 \\ \midrule
				HEM (Ours) & \textbf{249.16 (105.36)} & \textbf{6.09} & \textbf{10172.41 (12251.67)} & \textbf{12.05} \\ \bottomrule
			\end{tabular}
		}
	\end{table*}
	
	\begin{table*}[t]
		\caption{(\textbf{Left}) Evaluation of HEM on all instances from the entire MIPLIB dataset. (\textbf{Right}) Evaluation of HEM in the OptVerse commercial solver on the SCUC problems.}
		\label{hem_miplib_whole_all}
		\centering
		\resizebox{0.98\textwidth}{!}{
} \\

	\begin{table*}[t]
		\caption{Evaluation of HEM on all instances from the MIPLIB dataset. In particular, HEM improves the solving time by two orders of magnitude compared to Default on the berlin benchmark.}
		\label{hem_miplib_whole_all_each_instance1}
		\centering
		\resizebox{\textwidth}{!}{
			\begin{tabular}{@{}cccccccc@{}}
				\toprule
				\toprule
				\multirow{2}{*}{Instance   name} & \multirow{2}{*}{Method} & \multirow{2}{*}{Time (s)} & \multirow{2}{*}{PD integral} & \multirow{2}{*}{Instance name} & \multirow{2}{*}{Method} & \multirow{2}{*}{Time (s)} & \multirow{2}{*}{PD integral} \\
				&  &  &  &  &  &  &  \\ \midrule
				supportcase7 & HEM (Ours) & 171.03 & 10293.18 & istanbul-no-cutoff & HEM (Ours) & 268.96 & 10945.90 \\
				supportcase7 & Default & 161.56 & 5079.67 & istanbul-no-cutoff & Default & 233.91 & 8885.47 \\
				neos-4413714-turia & HEM (Ours) & 303.26 & 30325.68 & dws012-03 & HEM (Ours) & 300.03 & 30002.59 \\
				neos-4413714-turia & Default & 300.11 & 30011.48 & dws012-03 & Default & 300.01 & 30001.06 \\
				loopha13 & HEM (Ours) & 300.00 & 30000.04 & neos-826650 & HEM (Ours) & 300.00 & 30000.05 \\
				loopha13 & Default & 300.01 & 30000.64 & neos-826650 & Default & 300.00 & 30000.08 \\
\bottomrule
			\end{tabular}
		}
	\end{table*}
	
	\begin{table*}
		\caption{Evaluation of HEM on all instances from the MIPLIB dataset.}
		\label{hem_miplib_whole_all_each_instance2}
		\centering
		\resizebox{\textwidth}{!}{
			\begin{tabular}{@{}cccccccc@{}}
				\toprule
				\toprule
				\multirow{2}{*}{Instance   name} & \multirow{2}{*}{Method} & \multirow{2}{*}{Time (s)} & \multirow{2}{*}{PD integral} & \multirow{2}{*}{Instance name} & \multirow{2}{*}{Method} & \multirow{2}{*}{Time (s)} & \multirow{2}{*}{PD integral} \\
				&  &  &  &  &  &  &  \\ \midrule
				neos-3759587-noosa & HEM (Ours) & 31.99 & 3198.11 & nexp-150-20-8-5 & HEM (Ours) & 109.70 & 10445.19 \\
				neos-3759587-noosa & Default & 24.19 & 2418.13 & nexp-150-20-8-5 & Default & 300.00 & 24821.33 \\
				neos-4343293-stony & HEM (Ours) & 300.00 & 18256.30 & ns1849932 & HEM (Ours) & 300.09 & 30008.95 \\
				neos-4343293-stony & Default & 300.00 & 18456.28 & ns1849932 & Default & 303.24 & 30323.93 \\
				radiationm40-10-02 & HEM (Ours) & 300.00 & 12987.75 & beasleyC1 & HEM (Ours) & 10.34 & 167.80 \\
				radiationm40-10-02 & Default & 300.01 & 13722.23 & beasleyC1 & Default & 9.36 & 150.95 \\
\bottomrule
			\end{tabular}
		}
	\end{table*}
	
	\begin{table*}
		\caption{Evaluation of HEM on all instances from the MIPLIB dataset.}
		\label{hem_miplib_whole_all_each_instance3}
		\centering
		\resizebox{0.9\textwidth}{!}{
			\begin{tabular}{@{}cccccccc@{}}
				\toprule
				\multirow{2}{*}{Instance   name} & \multirow{2}{*}{Method} & \multirow{2}{*}{Time (s)} & \multirow{2}{*}{PD integral} & \multirow{2}{*}{Instance name} & \multirow{2}{*}{Method} & \multirow{2}{*}{Time (s)} & \multirow{2}{*}{PD integral} \\
				&  &  &  &  &  &  &  \\ \midrule
				neos-3656078-kumeu & HEM (Ours) & 300.67 & 30067.06 & g200x740 & HEM (Ours) & 300.00 & 4194.39 \\
				neos-3656078-kumeu & Default & 300.12 & 30012.27 & g200x740 & Default & 300.00 & 2361.11 \\
				neos-954925 & HEM (Ours) & 302.97 & 8384.86 & neos-3426085-ticino & HEM (Ours) & 49.43 & 278.56 \\
				neos-954925 & Default & 302.01 & 7483.41 & neos-3426085-ticino & Default & 300.03 & 2467.60 \\
				ran13x13 & HEM (Ours) & 70.75 & 425.67 & pigeon-08 & HEM (Ours) & 4.90 & 68.65 \\
				ran13x13 & Default & 70.71 & 427.02 & pigeon-08 & Default & 2.32 & 34.96 \\
\bottomrule
			\end{tabular}
		}
	\end{table*}
	
	\begin{table*}
		\caption{Evaluation of HEM on all instances from the MIPLIB dataset.}
		\label{hem_miplib_whole_all_each_instance4}
		\centering
		\resizebox{\textwidth}{!}{
			\begin{tabular}{@{}cccccccc@{}}
				\toprule
				\multirow{2}{*}{Instance   name} & \multirow{2}{*}{Method} & \multirow{2}{*}{Time (s)} & \multirow{2}{*}{PD integral} & \multirow{2}{*}{Instance name} & \multirow{2}{*}{Method} & \multirow{2}{*}{Time (s)} & \multirow{2}{*}{PD integral} \\
				&  &  &  &  &  &  &  \\ \midrule
				neos-933638 & HEM (Ours) & 300.00 & 24675.26 & cmflsp50-24-8-8 & HEM (Ours) & 300.00 & 30000.35 \\
				neos-933638 & Default & 300.00 & 25978.22 & cmflsp50-24-8-8 & Default & 300.00 & 30000.48 \\
				supportcase18 & HEM (Ours) & 300.01 & 1230.69 & piperout-03 & HEM (Ours) & 37.54 & 1390.49 \\
				supportcase18 & Default & 300.00 & 2327.46 & piperout-03 & Default & 20.24 & 627.14 \\
				f2000 & HEM (Ours) & 300.02 & 30001.84 & qnet1 & HEM (Ours) & 1.84 & 45.97 \\
				f2000 & Default & 300.02 & 30002.42 & qnet1 & Default & 2.00 & 47.09 \\
\bottomrule
			\end{tabular}
		}
	\end{table*}
	
	\begin{table*}
		\caption{Evaluation of HEM on all instances from the MIPLIB dataset.}
		\label{hem_miplib_whole_all_each_instance5}
		\centering
		\resizebox{\textwidth}{!}{
			\begin{tabular}{@{}cccccccc@{}}
				\toprule
				\multirow{2}{*}{Instance   name} & \multirow{2}{*}{Method} & \multirow{2}{*}{Time (s)} & \multirow{2}{*}{PD integral} & \multirow{2}{*}{Instance name} & \multirow{2}{*}{Method} & \multirow{2}{*}{Time (s)} & \multirow{2}{*}{PD integral} \\
				&  &  &  &  &  &  &  \\ \midrule
				manna81 & HEM (Ours) & 300.01 & 155.57 & adult-regularized & HEM (Ours) & 303.61 & 30360.54 \\
				manna81 & Default & 0.88 & 24.43 & adult-regularized & Default & 303.07 & 30307.08 \\
				n6-3 & HEM (Ours) & 300.02 & 7110.18 & neos-4355351-swalm & HEM (Ours) & 306.36 & 30636.41 \\
				n6-3 & Default & 300.00 & 6411.26 & neos-4355351-swalm & Default & 305.05 & 30504.98 \\
				markshare2 & HEM (Ours) & 300.02 & 30000.00 & neos-3216931-puriri & HEM (Ours) & 300.00 & 30000.12 \\
				markshare2 & Default & 300.02 & 30000.00 & neos-3216931-puriri & Default & 300.00 & 30000.09 \\
\bottomrule
			\end{tabular}
		}
	\end{table*}
	
	\begin{table*}
		\caption{Evaluation of HEM on all instances from the MIPLIB dataset.}
		\label{hem_miplib_whole_all_each_instance6}
		\centering
		\resizebox{\textwidth}{!}{
			\begin{tabular}{@{}cccccccc@{}}
				\toprule
				\multirow{2}{*}{Instance   name} & \multirow{2}{*}{Method} & \multirow{2}{*}{Time (s)} & \multirow{2}{*}{PD integral} & \multirow{2}{*}{Instance name} & \multirow{2}{*}{Method} & \multirow{2}{*}{Time (s)} & \multirow{2}{*}{PD integral} \\
				&  &  &  &  &  &  &  \\ \midrule
				unitcal-7 & HEM (Ours) & 300.19 & 30018.65 & brazil3 & HEM (Ours) & 300.00 & 30000.14 \\
				unitcal-7 & Default & 300.19 & 29627.79 & brazil3 & Default & 300.00 & 30000.17 \\
				gt2 & HEM (Ours) & 0.10 & 3.11 & drayage-100-12 & HEM (Ours) & 300.00 & 3857.06 \\
				gt2 & Default & 0.10 & 3.32 & drayage-100-12 & Default & 300.00 & 3395.58 \\
				co-100 & HEM (Ours) & 311.21 & 28621.40 & neos-3660371-kurow & HEM (Ours) & 300.00 & 17472.28 \\
				co-100 & Default & 303.17 & 27884.97 & neos-3660371-kurow & Default & 300.00 & 14003.05 \\
\bottomrule
			\end{tabular}
		}
	\end{table*}
	
	\begin{table*}
		\caption{Evaluation of HEM on all instances from the MIPLIB dataset.}
		\label{hem_miplib_whole_all_each_instance7}
		\centering
		\resizebox{\textwidth}{!}{
			\begin{tabular}{@{}cccccccc@{}}
				\toprule
				\multirow{2}{*}{Instance   name} & \multirow{2}{*}{Method} & \multirow{2}{*}{Time (s)} & \multirow{2}{*}{PD integral} & \multirow{2}{*}{Instance name} & \multirow{2}{*}{Method} & \multirow{2}{*}{Time (s)} & \multirow{2}{*}{PD integral} \\
				&  &  &  &  &  &  &  \\ \midrule
				vpphard & HEM (Ours) & 303.40 & 30339.75 & neos-1601936 & HEM (Ours) & 300.00 & 30000.17 \\
				vpphard & Default & 300.75 & 30074.70 & neos-1601936 & Default & 300.02 & 30001.80 \\
				glass4 & HEM (Ours) & 300.01 & 25187.20 & tanglegram6 & HEM (Ours) & 30.80 & 2360.64 \\
				glass4 & Default & 300.01 & 15874.85 & tanglegram6 & Default & 29.37 & 2348.04 \\
				neos-5049753-cuanza & HEM (Ours) & 300.18 & 30018.02 & fastxgemm-n2r6s0t2 & HEM (Ours) & 300.00 & 27880.95 \\
				neos-5049753-cuanza & Default & 300.29 & 30028.82 & fastxgemm-n2r6s0t2 & Default & 300.01 & 26975.04 \\
\bottomrule
			\end{tabular}
		}
	\end{table*}
	
	\begin{table*}
		\caption{Evaluation of HEM on all instances from the MIPLIB dataset.}
		\label{hem_miplib_whole_all_each_instance8}
		\centering
		\resizebox{0.8\textwidth}{!}{
			\begin{tabular}{@{}cccccccc@{}}
				\toprule
				\multirow{2}{*}{Instance   name} & \multirow{2}{*}{Method} & \multirow{2}{*}{Time (s)} & \multirow{2}{*}{PD integral} & \multirow{2}{*}{Instance name} & \multirow{2}{*}{Method} & \multirow{2}{*}{Time (s)} & \multirow{2}{*}{PD integral} \\
				&  &  &  &  &  &  &  \\ \midrule
				air04 & HEM (Ours) & 36.90 & 971.45 & neos-3611689-kaihu & HEM (Ours) & 6.25 & 29.51 \\
				air04 & Default & 123.59 & 1286.61 & neos-3611689-kaihu & Default & 9.02 & 47.69 \\
				supportcase6 & HEM (Ours) & 301.69 & 28828.67 & rlp1 & HEM (Ours) & 0.13 & 5.54 \\
				supportcase6 & Default & 301.67 & 29107.77 & rlp1 & Default & 0.10 & 3.22 \\
				gr4x6 & HEM (Ours) & 0.04 & 0.64 & map06 & HEM (Ours) & 300.01 & 17859.15 \\
				gr4x6 & Default & 0.01 & 0.29 & map06 & Default & 300.06 & 16778.84 \\
\bottomrule
			\end{tabular}
		}
	\end{table*}
	
	\begin{table*}
		\caption{Evaluation of HEM on all instances from the MIPLIB dataset.}
		\label{hem_miplib_whole_all_each_instance9}
		\centering
		\resizebox{0.8\textwidth}{!}{
			\begin{tabular}{@{}cccccccc@{}}
				\toprule
				\multirow{2}{*}{Instance   name} & \multirow{2}{*}{Method} & \multirow{2}{*}{Time (s)} & \multirow{2}{*}{PD integral} & \multirow{2}{*}{Instance name} & \multirow{2}{*}{Method} & \multirow{2}{*}{Time (s)} & \multirow{2}{*}{PD integral} \\
				&  &  &  &  &  &  &  \\ \midrule
				comp12-2idx & HEM (Ours) & 300.87 & 30087.37 & neos-1445743 & HEM (Ours) & 13.36 & 160.71 \\
				comp12-2idx & Default & 300.00 & 30000.21 & neos-1445743 & Default & 26.00 & 327.97 \\
				acc-tight4 & HEM (Ours) & 89.45 & 8944.87 & neos-5192052-neckar & HEM (Ours) & 0.03 & 0.58 \\
				acc-tight4 & Default & 138.36 & 13835.93 & neos-5192052-neckar & Default & 0.02 & 0.51 \\
				neos-1171448 & HEM (Ours) & 45.94 & 1684.42 & neos-1605061 & HEM (Ours) & 300.00 & 30000.10 \\
				neos-1171448 & Default & 123.30 & 1493.79 & neos-1605061 & Default & 300.04 & 30003.63 \\
\bottomrule
			\end{tabular}
		}
	\end{table*}
	
	\begin{table*}
		\caption{Evaluation of HEM on all instances from the MIPLIB dataset.}
		\label{hem_miplib_whole_all_each_instance10}
		\centering
		\resizebox{0.8\textwidth}{!}{
			\begin{tabular}{@{}cccccccc@{}}
				\toprule
				\multirow{2}{*}{Instance   name} & \multirow{2}{*}{Method} & \multirow{2}{*}{Time (s)} & \multirow{2}{*}{PD integral} & \multirow{2}{*}{Instance name} & \multirow{2}{*}{Method} & \multirow{2}{*}{Time (s)} & \multirow{2}{*}{PD integral} \\
				&  &  &  &  &  &  &  \\ \midrule
				neos5 & HEM (Ours) & 300.00 & 1637.07 & assign1-10-4 & HEM (Ours) & 300.00 & 4387.55 \\
				neos5 & Default & 300.00 & 1301.29 & assign1-10-4 & Default & 300.00 & 5346.93 \\
				sp150x300d & HEM (Ours) & 222.67 & 1791.81 & bppc4-08 & HEM (Ours) & 300.00 & 805.90 \\
				sp150x300d & Default & 300.01 & 4886.29 & bppc4-08 & Default & 300.00 & 815.13 \\
				fhnw-schedule-pairb200 & HEM (Ours) & 300.51 & 27161.85 & rocI-4-11 & HEM (Ours) & 90.04 & 4926.67 \\
				fhnw-schedule-pairb200 & Default & 300.26 & 27035.21 & rocI-4-11 & Default & 65.96 & 3593.95 \\
\bottomrule
			\end{tabular}
		}
	\end{table*}
	
	\begin{table*}
		\caption{Evaluation of HEM on all instances from the MIPLIB dataset.}
		\label{hem_miplib_whole_all_each_instance11}
		\centering
		\resizebox{0.8\textwidth}{!}{
			\begin{tabular}{@{}cccccccc@{}}
				\toprule
				\multirow{2}{*}{Instance   name} & \multirow{2}{*}{Method} & \multirow{2}{*}{Time (s)} & \multirow{2}{*}{PD integral} & \multirow{2}{*}{Instance name} & \multirow{2}{*}{Method} & \multirow{2}{*}{Time (s)} & \multirow{2}{*}{PD integral} \\
				&  &  &  &  &  &  &  \\ \midrule
				rentacar & HEM (Ours) & 2.24 & 164.94 & z26 & HEM (Ours) & 314.05 & 17073.75 \\
				rentacar & Default & 2.04 & 153.38 & z26 & Default & 308.47 & 15960.29 \\
				ns1830653 & HEM (Ours) & 231.69 & 9503.19 & wachplan & HEM (Ours) & 300.00 & 3473.46 \\
				ns1830653 & Default & 216.83 & 10504.02 & wachplan & Default & 300.00 & 3463.03 \\
				fhnw-binschedule2 & HEM (Ours) & 330.04 & 8033.38 & reblock420 & HEM (Ours) & 300.00 & 29265.55 \\
				fhnw-binschedule2 & Default & 301.36 & 9529.63 & reblock420 & Default & 300.10 & 26684.16 \\
\bottomrule
			\end{tabular}
		}
	\end{table*}

\end{document}